\documentclass[twoside]{article}

\usepackage[accepted]{aistats2022}

\usepackage[round]{natbib}

\usepackage{microtype}
\usepackage{subfigure}
\usepackage{booktabs} 

\usepackage{amsmath, amssymb, amsfonts, amsthm, dsfont}
\usepackage{algorithm, algpseudocode, algorithmicx}
\usepackage{mathtools,commath}
\usepackage{graphicx, multirow}
\usepackage{xcolor}
\usepackage{authblk}
\usepackage{hyperref}

\newtheorem{theorem}{Theorem}

\newtheorem{proposition}[theorem]{Proposition}
\newtheorem{definition}{Definition}

\newtheorem{remark}{Remark}

\allowdisplaybreaks

\title{Beta Shapley: a Unified and Noise-reduced Data Valuation Framework for Machine Learning}
\date{\today}
\author{Yongchan Kwon}
\author{James Zou}
\affil{Stanford University}

\begin{document}
\twocolumn[

\aistatstitle{Beta Shapley: a Unified and Noise-reduced Data Valuation Framework for Machine Learning}

\aistatsauthor{Yongchan Kwon \And  James Zou} \aistatsaddress{Department of Biomedical Data Science, Stanford University}
]

\begin{abstract}
Data Shapley has recently been proposed as a principled framework to quantify the contribution of individual datum in machine learning. It can effectively identify helpful or harmful data points for a learning algorithm. In this paper, we propose Beta Shapley, which is a substantial generalization of Data Shapley. Beta Shapley arises naturally by relaxing the efficiency axiom of the Shapley value, which is not critical for machine learning settings. Beta Shapley unifies several popular data valuation methods and includes data Shapley as a special case. Moreover, we prove that Beta Shapley has several desirable statistical properties and propose efficient algorithms to estimate it. We demonstrate that Beta Shapley outperforms state-of-the-art data valuation methods on several downstream ML tasks such as: 1) detecting mislabeled training data; 2) learning with subsamples; and 3) identifying points whose addition or removal have the largest positive or negative impact on the model. 
\end{abstract}

\section{Introduction}
\label{s:intro}
Getting appropriate training data is often the biggest and most expensive challenge of machine learning (ML). In many real world applications, a fraction of the data could be very noisy due to outliers or label errors. Furthermore, data are often costly to collect and understanding of what types of data are more useful for training a model can help to guide data curation. For all of these motivations, data valuation has emerged as an important area of ML research. The goal of data valuation is to quantify the contribution of each training datum to the model's performance. 

Recently, inspired by ideas from economics and game theory, data Shapley has been proposed to represent the notion of which datum helps or harms the predictive performance of a model \citep{ghorbani2019}. 
Data Shapley has several benefits compared to existing data valuation methods. It uniquely satisfies the natural properties of fairness in cooperative game theory. Also, it better captures the influence of individual datum, showing superior performance on multiple downstream ML tasks, including identifying mislabeled observations in classification problems or detecting outliers in regression problems \citep{ghorbani2019, jia2019b, jia2019}.

Data Shapley is defined as a function of marginal contributions that measure the average change in a trained model's performance when a particular point is removed from a set with a given cardinality. The marginal contribution is a basic ingredient in many data valuation approaches. For example, the commonly used leave-one-out (LOO) analysis is equivalent to estimating a point's marginal contribution when it is removed from the entire training set. The marginal contribution of a point can vary if the cardinality of a given set changes, and data Shapley takes a simple average of the marginal contributions on all the different cardinalities. In this way, data Shapley can avoid the dependency on a specific cardinality, but it is unclear whether this uniform weight is optimal for quantifying the impact of individual datum. As we will show both theoretically and through experiments, this is in fact sub-optimal. The uniform averaging arises from the efficiency axiom of Shapley values, which is not essential in  ML settings. The axiom requires the sum of data values to equal the total utility, but it might not be sensible nor verifiable in practice.

\paragraph{Our contributions} In this paper, we propose Beta Shapley, a unified data valuation framework that naturally arises by relaxing the efficiency axiom. Our theoretical analyses show that Beta Shapley is characterized by reduced noise compared to data Shapley and can be applied to find optimal importance weights for subsampling. 
We develop an efficient algorithm to estimate it based on Monte Carlo methods. We demonstrate that Beta Shapley outperforms state-of-the-art data valuation methods on several downstream ML tasks including noisy label detection, learning with subsamples, and point addition and removal experiments. 

\paragraph{Related works}
The Shapley value was introduced in a seminar paper as a method of fair division of rewards in cooperative games \citep{shapley1953}. It has been applied to various ML problems, for instance, variable selection \citep{cohen2005, zaeri2018feature}, feature importance \citep{lundberg2017, covert2020understanding, lundberg2020local, covert2020explaining, covert2021improving}, model interpretation \citep{chen2018, sundararajan2019many, ghorbani2020neuron, wang2021shapley}, model importance \citep{rozemberczki2021shapley}, and the collaborative learning problems \citep{sim2020collaborative}.
As for the data valuation problem, data Shapley was introduced by \citet{ghorbani2019} and \citet{jia2019}, and many extensions have been studied in the literature. For example, KNN Shapley was proposed to address the computational cost of data Shapley by using the $k$-nearest neighborhood model \citep{jia2019b}, and distributional Shapley value was studied to deal with the random nature of data Shapley \citep{ghorbani2020distributional, kwon2021efficient}.

The relaxation of the Shapley axioms has been one of the central topics in the field of economics \citep{kalai1987weighted, weber1988probabilistic}. When the symmetry axiom is removed, the quasivalue and the weighted value have been studied \citep{shapley1953additive, banzhaf1964weighted, gilboa1991quasi, monderer1992weighted}. When the efficiency axiom is removed, the semivalue has been studied \citep{dubey1977probabilistic, dubey1981, ridaoui2018axiomatisation}. We refer to \citet{monderer2002variations} for a complementary literature review of variations of Shapley value. Our work is based on the semivalue and characterizes its statistical properties in ML settings.

\section{Preliminaries}
\label{s:preliminaries}
We review the marginal contribution, a key component for analyzing the impact of one datum, and various data valuation methods based on it. We first define some notations. Let $Z$ be a random variable for data defined on a set $\mathcal{Z} \subseteq \mathbb{R}^d$ for some integer $d$ and denote its distribution by $P_{Z}$. In supervised learning, we can think of $Z = (X, Y)$ defined on a set $\mathcal{X} \times \mathcal{Y}$, where $X$ and $Y$ describe the input and its label, respectively. Throughout this paper, we denote a set of independent and identically distributed (i.i.d.) samples from $P_Z$ by $\mathcal{D}=\{z_1, \dots, z_n\}$. We denote a utility function by $U: \cup_{j=0} ^{\infty} \mathcal{Z} ^j \to \mathbb{R}$. Here, we use the conventions $\mathcal{Z} ^0 := \{ \emptyset \}$ and $U(\emptyset)$ is the performance based on the best constant predictor. The utility function represents the performance of a model trained on a set of data points. In regression, for instance, one choice for $U(S)$ is the negative mean squared error of a model (e.g. linear regression) trained on the subset $S \subseteq \mathcal{D} \subseteq \mathcal{X} \times \mathcal{Y}$. Similarly, in classification, $U(S)$ can be the classification accuracy of a model (e.g. logistic regression) trained on $S$. Note that the utility depends on which model is used. Throughout this paper, the dependence on the model is omitted for notational simplicity, but it does not affect our results. Lastly, for a set $S$, we denote its cardinality by $|S|$, and for $m \in \mathbb{N}$, we use $[m]$ to denote a set of integers $\{ 1, \dots, m\}$. 

Data valuation has been studied as a problem to evaluate the impact of individual datum, and many existing metrics measure how much the model output or model performance changes after removing one data point of interest. These concepts can be formalized by the marginal contribution defined below. 

\begin{definition}[Marginal contribution]
For a function $h$ and $j \in [n]$, we define the marginal contribution of $z^* \in \mathcal{D}$ with respect to $j-1$ samples as
\begin{align*}
    \Delta_j (z^*; h, \mathcal{D}) := \frac{1}{\binom{n-1}{j-1}} \sum_{ S \in \mathcal{D}_{j} ^{\backslash z^*} } h(S\cup \{z^*\})-h(S),
\end{align*}
where $\mathcal{D}_{j} ^{\backslash z^*} := \{ S \subseteq \mathcal{D} \backslash \{z^*\}: |S|=j-1 \}$.
\label{def:marginal_contrib}
\end{definition}
The marginal contribution $\Delta_j (z^*; h, \mathcal{D})$ considers all possible subsets with the same cardinality $S \in \mathcal{D}_{j} ^{\backslash z^*}$ and measures the average changes of $h$ when datum of interest $z^*$ is removed from $S \cup \{z^*\}$. Note that when $j=n$, the marginal contribution $\Delta_n (z^*; h, \mathcal{D})$ equals to $h(\mathcal{D}) - h(\mathcal{D} \backslash \{ z^*\})$, and it captures the effect of deleting $z^*$ from the entire training dataset $\mathcal{D}$.

Many existing data valuation methods can be explained by the marginal contribution $\Delta_j (z^*; h, \mathcal{D})$. Specifically, Cook’s distance \citep{cook1980characterizations, cook1982residuals} is proportional to the squared $\ell_2$-norm of the marginal contribution $\norm{\Delta_n (z^*; h, \mathcal{D})}_2 ^2$ when $h$ outputs predictions of the given dataset $\mathcal{D}$, and the LOO method uses $\Delta_n (z^*; h, \mathcal{D})$ as data values when $h$ is a utility function $U$. Also, the influence function can be regarded as an approximation of LOO \citep{koh2017understanding}.

Data Shapley is another example that can be expressed as a function of marginal contributions \citep{ghorbani2019,jia2019, jia2019b}. To be more specific, data Shapley of datum $z^* \in \mathcal{D}$ is defined as
\begin{align}
    \psi_{\mathrm{shap}}(z^*; U, \mathcal{D}) := \frac{1}{n} \sum_{j=1} ^n \Delta_j (z^*; U, \mathcal{D}).
    \label{eqn:def_data_shapley_value}
\end{align}
Unlike Cook's distance or LOO methods, data Shapley in \eqref{eqn:def_data_shapley_value} considers all cardinalities and takes a simple average of the marginal contributions. By assigning the constant weight on different marginal contributions, it avoids the dependency on a specific cardinality and can capture the effect of one data point for small cardinality.

Data Shapley provides a principled data valuation framework in that it uniquely satisfies the natural properties of a fair division of rewards in cooperative game theory. \citet{shapley1953} showed that the Shapley value is the unique function $\psi$ that satisfies the following four axioms.
\begin{itemize}
    \item Linearity: for functions $U_1$, $U_2$ and $\alpha_1, \alpha_2 \in \mathbb{R}$, $ \psi( z^*; \alpha_1 U_1 + \alpha_2 U_2, \mathcal{D} ) = \alpha_1 \psi( z^*;  U_1 , \mathcal{D} ) + \alpha_2 \psi ( z^*; U_2, \mathcal{D} )$. 
    \item Null player: if $U(S\cup \{z^*\})=U(S)+c$ for any $S \subseteq \mathcal{D}\backslash \{z^*\}$ and some $c \in \mathbb{R}$, then $\psi (z^*; U, \mathcal{D}) = c$.
    \item Symmetry: for every $U$ and every permutation $\pi$ on $\mathcal{D}$, $\psi(\pi^* U) = \pi^* \psi U$ where $\pi^*U$ is defined as $(\pi^*U)(S) := U(\pi(S))$ for every $S \subseteq \mathcal{D}$.
    \item Efficiency: for every $U$, $\sum_{z \in \mathcal{D}} \psi(z; U, \mathcal{D}) = U(\mathcal{D})$.
\end{itemize}

Although data Shapley provides a fundamental framework for data values, there are some critical issues. In particular, it is unclear whether the uniform weight in \eqref{eqn:def_data_shapley_value} is optimal to represent the influence of one datum. When the cardinality $|S|$ is large enough, the performance change $U(S\cup\{z^*\})-U(S)$ is near zero, and thus the marginal contribution $\Delta_{|S|} (z^*; U, \mathcal{D})$ becomes negligible. In particular, when $U$ is a negative log-likelihood function, it can be shown that $U(S\cup\{z^*\})-U(S)=O_p(|S|^{-2})$ under mild conditions. This can make it hard to tell which data points contribute more to predictive performance. In the following section, we rigorously analyze the marginal contribution and show that using the uniform weight in \eqref{eqn:def_data_shapley_value} can be detrimental to capturing the influence of individual data.

\section{Theoretical analysis of marginal contribution}
\label{s:analysis_marginal_contributions}
In this section, we study asymptotic properties of the marginal contribution. To this end, we define a set $\mathfrak{D}=\{z^*, Z_1, \dots, Z_{n-1}\}$ where $Z_i$\rq{}s be i.i.d. random variables from $P_Z$, \textit{i.e.}, all elements of $\mathfrak{D}$ are random except for $z^*$. Given that $\Delta_j (z^*; U, \mathfrak{D})$ has the form of U-statistics \citep{hoeffding1948class}, Theorem 12.3 in \citet{van2000asymptotic} implies that for a fixed cardinality $j$, as $n$ approaches to infinity, we have 
\begin{align}
     (j^2 \zeta_{1} /n) ^{-1}\mathrm{Var} ( \Delta_j (z^*; U, \mathfrak{D}) \to 1,
    \label{eqn:u_stat_for_small_cardinality}
\end{align}
where $\zeta_{1} = \mathrm{Var} \left( \mathbb{E}[ U(S \cup \{z^*\})-U(S) \mid Z_1 ] \right)$ and $S$ is a random subset such that $|S|=j-1$ and each element in $S$ is chosen from $\{Z_1, \dots, Z_{n-1}\} = \mathfrak{D} \backslash \{z^*\}$ uniformly at random. All expectation and variance computations are under the $P_Z$. The result \eqref{eqn:u_stat_for_small_cardinality} shows that the asymptotic variance of $\Delta_j (z^*; U, \mathfrak{D})$ scales $O(j^2 \zeta_{1}/n)$ for a fixed cardinality $j$. However, since the data Shapley is a simple average of marginal contributions across all cardinalities $j \in \{1, \dots, n\}$, an analysis of marginal contribution for large $j$ is important to examine the statistical properties of the data Shapley. 

In the following theorem, we provide an asymptotic variance when the cardinality $j$ is allowed to increase to infinity. To begin with, for $j\in [n]$ we set $\zeta_{j} := \mathrm{Var} \left( U( S \cup \{z^*\})-U(S) \right)$ where $S$ is a random subset such that $|S|=j-1$ and $S\subseteq \mathfrak{D}\backslash\{z^*\}$. 

\begin{theorem}[Asymptotic distribution of the marginal contribution]
Suppose the cardinality $j = o(n^{1/2})$ and assume that $\lim_{j \to \infty}\zeta_{j}/(j\zeta_{1})$ is bounded. Then, $(j^2 \zeta_{1}/n) ^{-1} \mathrm{Var}(\Delta_j (z^*; U, \mathfrak{D})) \to 1$ as $n$ increases. 
\label{thm:u_stat_for_large_cardinality}
\end{theorem}

\begin{figure}[t]
    \centering
    \includegraphics[width=0.475\columnwidth]{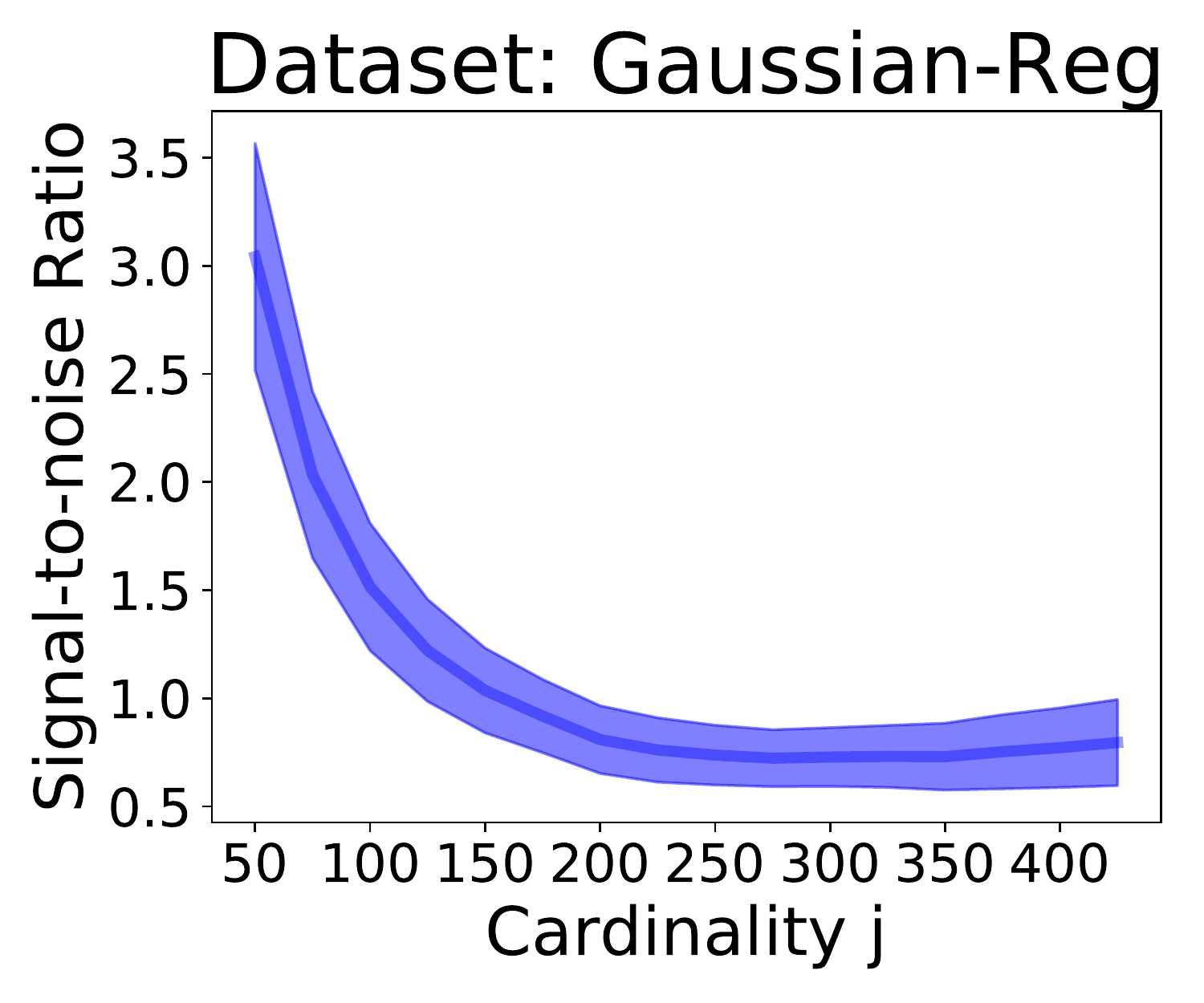}
    \includegraphics[width=0.475\columnwidth]{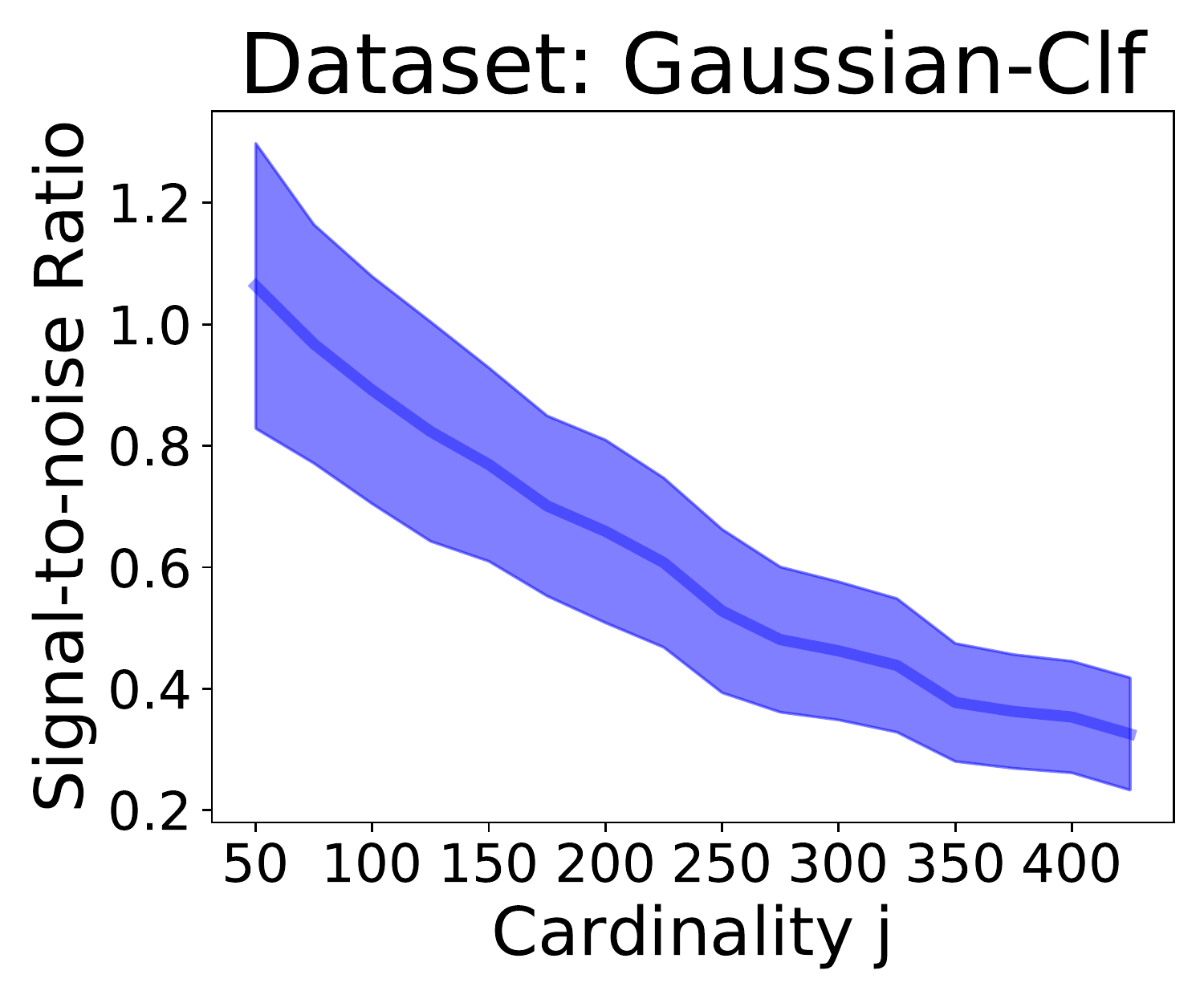}
    \caption{The signal-to-noise ratio of $\Delta_j (z^*; U, \mathfrak{D})$ as a function of the cardinality $j$ when $n=500$ in (left) regression and (right) classification settings. The data are generated from a generalized linear model. The signal-to-noise ratio generally decreases as the cardinality $j$ increases, showing that when $j$ is large, the signal of the marginal contribution at large cardinality is more likely to be perturbed by noise.} 
    \label{fig:snr}
\end{figure}

\begin{figure*}[t]
    \centering
    \includegraphics[width=0.245\textwidth]{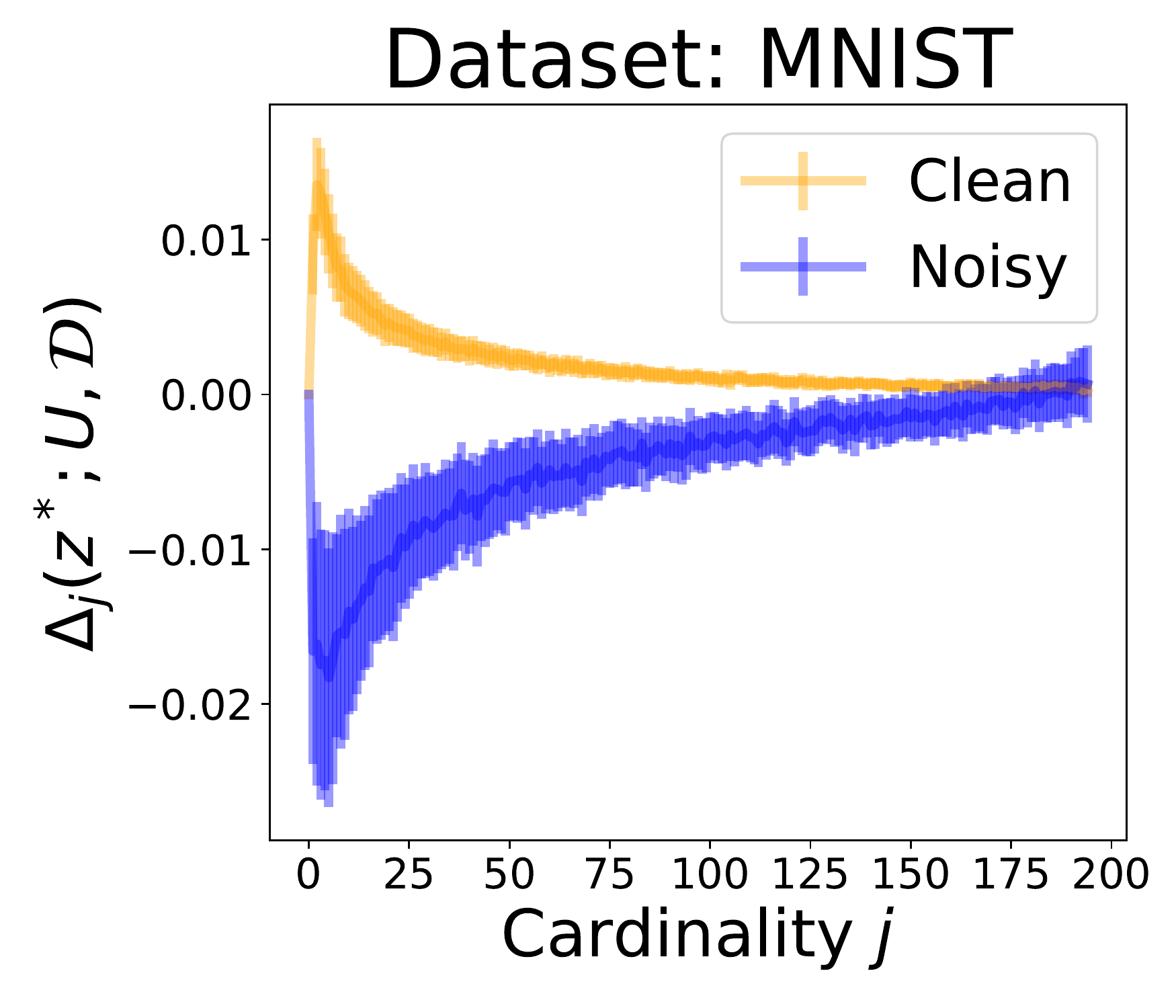}
    \includegraphics[width=0.245\textwidth]{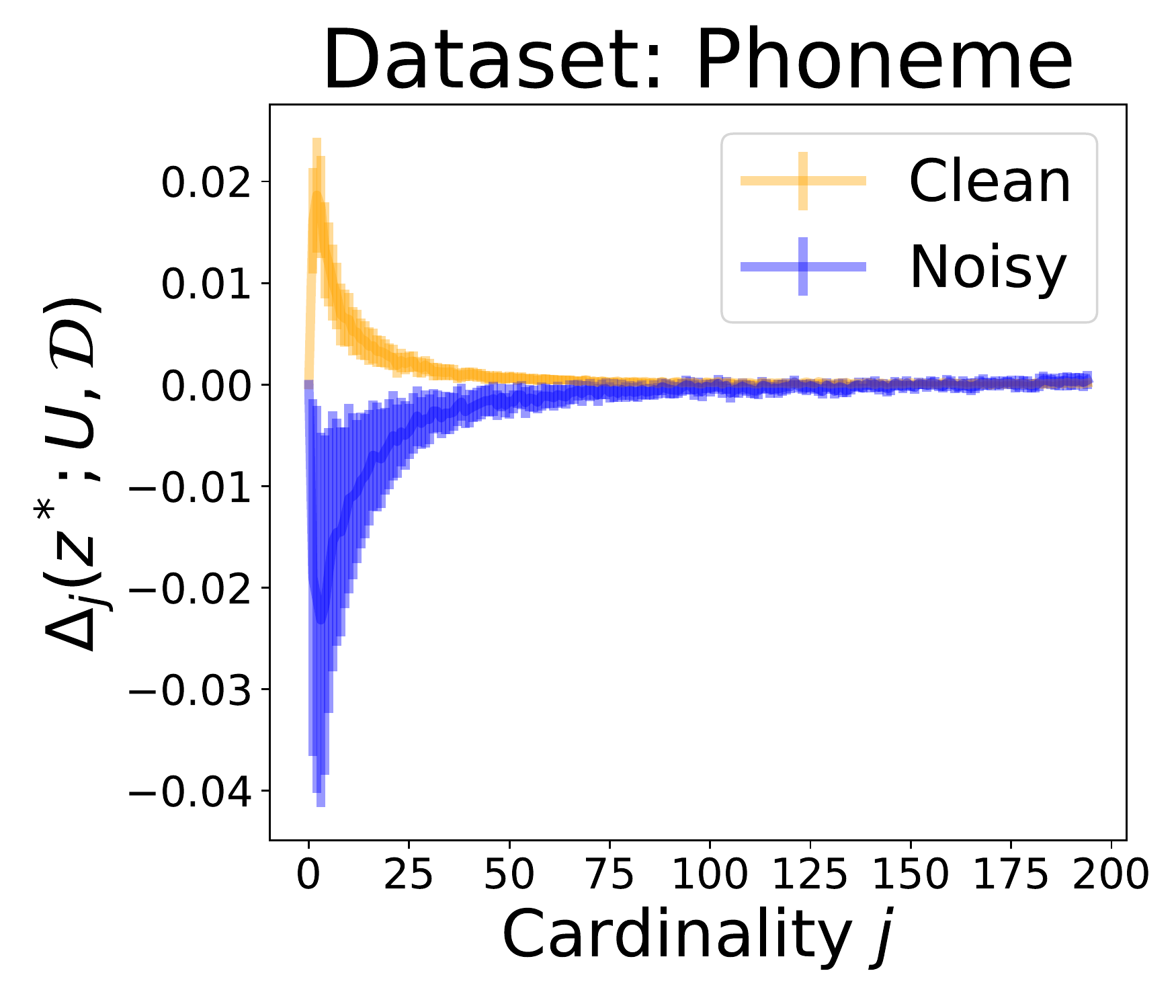}
    \includegraphics[width=0.245\textwidth]{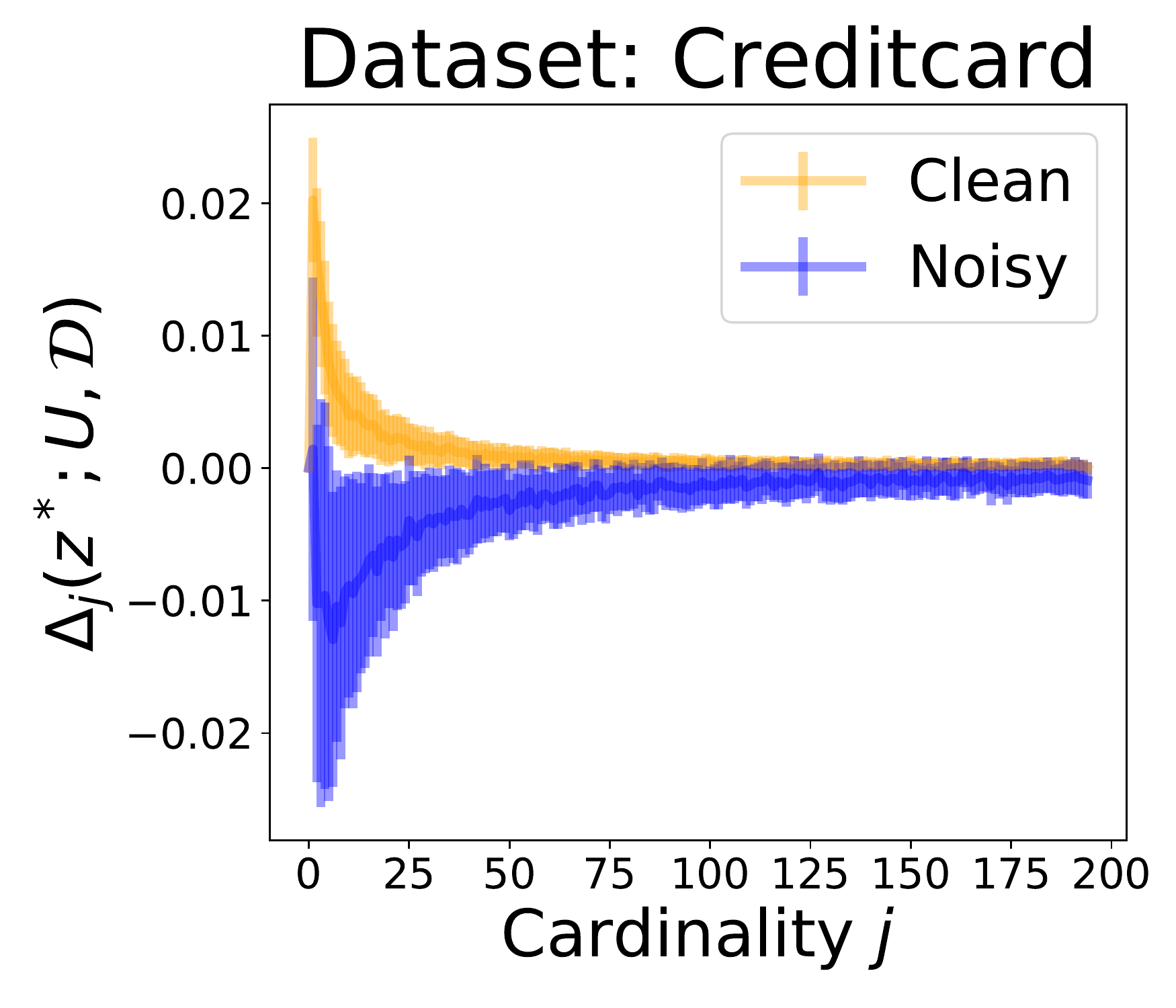}
    \includegraphics[width=0.245\textwidth]{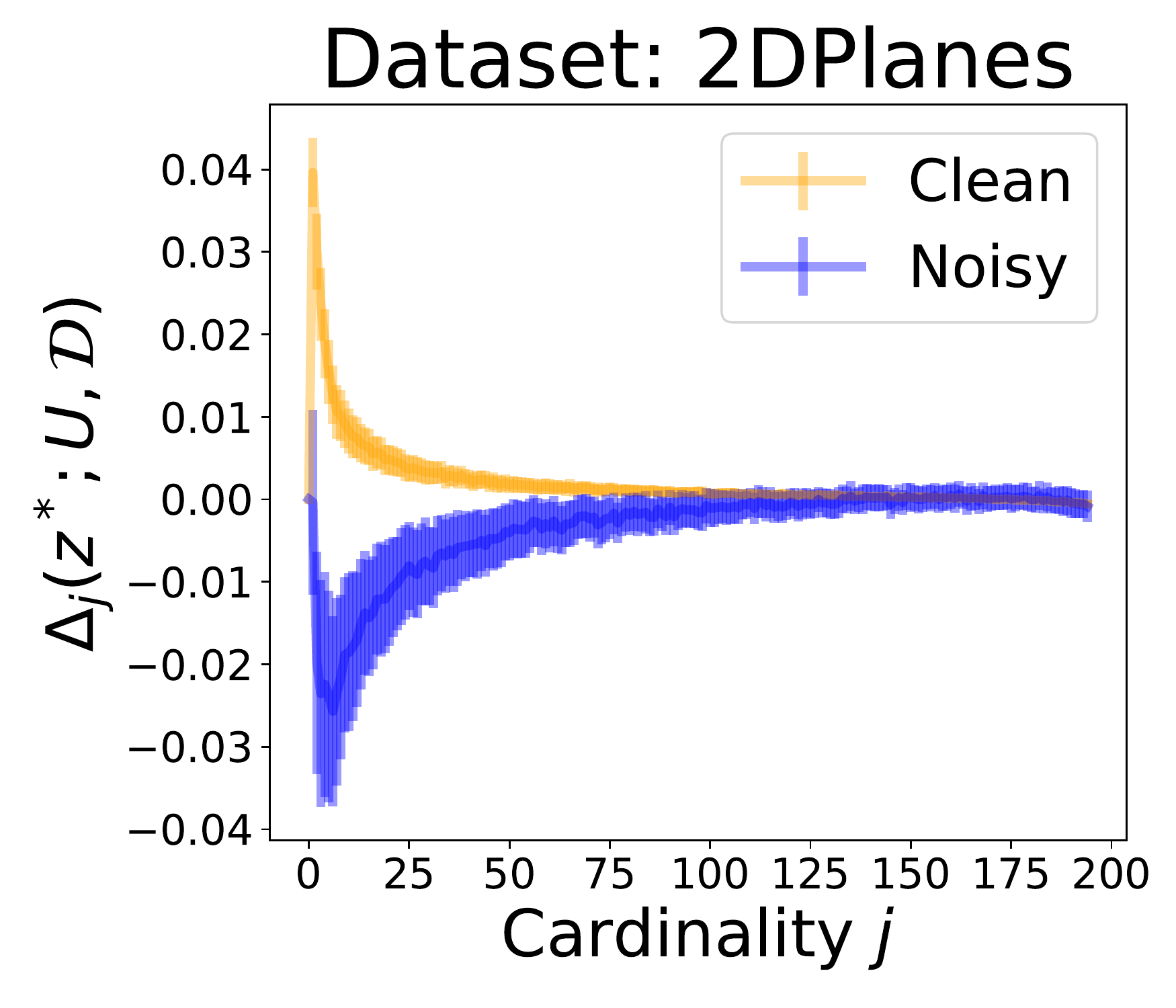}
    \caption{Illustrations of the marginal contribution $\Delta_j (z^*; U, \mathcal{D})$ as a function of the cardinality $j$ on the four datasets. Each color indicates a clean data point (yellow) and a noisy data point (blue). When the cardinality $j$ is large, the marginal contributions of the two groups become similar, so it is difficult to determine whether a data point is noisy by $\Delta_j (z^*; U, \mathcal{D})$. We provide additional results on different datasets in Appendix.} 
    \label{fig:clean_noisy_marginal_contributions}
\end{figure*}

Theorem \ref{thm:u_stat_for_large_cardinality} extends the previous asymptotic result in \eqref{eqn:u_stat_for_small_cardinality} to the case of diverging cardinality $j$. Note that for any $0 \leq \gamma < 1/2$, the cardinality $j = n^{\gamma}$ satisfies the condition $j=o(n^{1/2})$. This result provides a mathematical insight into the signal-to-noise ratio of $\Delta_j (z^*; U, \mathfrak{D})$ as the following remark.

\begin{remark}
Given that $|\mathbb{E}[\Delta_j (z^*; U, \mathfrak{D})]|$ usually decreases as $j$ increases in ML settings, the signal-to-noise ratio $|\mathbb{E}[\Delta_j (z^*; U, \mathfrak{D})]| / \sqrt{\mathrm{Var}(\Delta_j (z^*; U, \mathfrak{D}))}$ is expected to decrease as $j$ increases because $\mathrm{Var}(\Delta_j (z^*; U, \mathfrak{D}))$ is $O(j^2 \zeta_1/n)$.
\end{remark}

Figure \ref{fig:snr} illustrates the signal-to-noise ratio of $\Delta_j (z^*; U, \mathfrak{D})$ as a function of the cardinality $j$ for two example datasets with $n=500$. We denote a 95\% confidence band based on 50 repetitions under the assumption that the results follow the identical Gaussian distribution. We use a generalized linear model to generate data and use the negative mean squared error and the classification accuracy as a utility for regression and classification settings, respectively.
Figure \ref{fig:snr} clearly shows the signal-to-noise ratio decreases as $j$ increases. In other words, when $j$ is large, the signal of marginal contribution is more likely to be perturbed by noise. This result motivates us to assign large weight to small cardinality instead of using the uniform weight used in data Shapley. We provide details on implementation and additional results on real datasets in Appendix.

In Theorem \ref{thm:u_stat_for_large_cardinality}, we fix one data point $z^*$ and show that noise can be introduced when the cardinality is large. We now consider the entire dataset $\mathcal{D}$ and examine which cardinality is useful to capture the signal. To do so, we directly compare marginal contribution $\Delta_j (z^*; U, \mathcal{D})$ of mislabeled and correctly labeled data points in various classification settings. We set the sample size $n=200$ and assume that observed data can be mislabeled. As for mislabeled data, we flip the original label for a random 10\% of data points in $\mathcal{D}$. The utility function is the classification accuracy. Implementation details are provided in Appendix. 

Figure \ref{fig:clean_noisy_marginal_contributions} shows that there is a significant gap between the marginal contributions of the two groups when $j$ is small, but the gap becomes zero when $j$ is large. In particular, the 95\% confidence band for the clean data point (yellow) overlaps the confidence band for the mislabeled data point (blue) when $j$ is greater than $150$ in all datasets. This shows that using the uniform weight used in data Shapley \eqref{eqn:def_data_shapley_value} makes it difficult to tell whether or not a data point belongs to the clean group, and as a result, it can lead to undesirable decisions. 

One potential limit of Theorem \ref{thm:u_stat_for_large_cardinality} is that it is unknown whether the bound condition $\lim_{j \to \infty} \zeta_{j}/(j\zeta_{1})$ holds. We empirically show that this condition is plausible in Appendix.

\section{Proposed Beta-Shapley method}
\label{s:proposed}
Motivated by the results in Section \ref{s:analysis_marginal_contributions}, assigning large weights to small cardinality is expected to capture the impact of one datum better than data Shapley. In the following subsection, motivated by the idea of semivalue  \citep{dubey1981}, we show that removing the efficiency axiom gives a new form of data value that can assign larger weight on the marginal contribution based on small cardinality than large cardinality. 

\subsection{Data valuations without efficiency axiom}
\label{s:semivalue}
The efficiency axiom, which requires the total sum of data values to be equal to the utility, is not essential in ML settings. For example, multiplying data Shapley by a positive constant changes the sum of the data values but does not change the order between the data values. In other words, there are many data values that do not satisfy the efficiency but can equivalently identify low-quality data points as data Shapley. In this respect, we define a semivalue, which is characterized by all Shapley axioms except the efficiency axiom. 

\begin{definition}[semivalue]
We say a function $\psi$ is a semivalue if $\psi$ satisfies the linearity, null player, and symmetry axioms.
\end{definition}

In the following theorem, we now show how data values can be formulated without the efficiency axiom. 

\begin{theorem}[Representation of semivalues]
A value function $\psi_{\mathrm{semi}}$ is a semivalue, if and only if, there exists a weight function $w^{(n)}: [n] \to \mathbb{R}$ such that $\sum_{j=1} ^{n} \binom{n-1}{j-1} w^{(n)}(j)=n$ and the value function $\psi_{\mathrm{semi}}$ can be expressed as follows.
\begin{align}
    &\psi_{\mathrm{semi}}(z^*; U, \mathcal{D}, w^{(n)}) \notag \\
    &:= \frac{1}{n} \sum_{j=1} ^{n} \binom{n-1}{j-1} w^{(n)}(j)  \Delta_j (z^*; U, \mathcal{D}).   
    \label{eqn:def_prob_data_shapley_value}
\end{align}
\label{thm:semivalue_representation}
\end{theorem}

Theorem \ref{thm:semivalue_representation} shows that every semivalue $\psi_{\mathrm{semi}}$ can be expressed as a weighted mean of marginal contributions without the efficiency axiom. Compared to data Shapley, a semivalue provides flexible formulation of data valuation and includes various existing data valuation methods. For example, when $w^{(n)}(j) = \binom{n-1}{j-1} ^{-1}$, the semivalue $\psi_{\mathrm{semi}}(z^*; U, \mathcal{D}, w^{(n)})$ becomes the data Shapley \citep{ghorbani2019}, and when $w^{(n)}(j) = n \mathds{1}(j=n)$, and it reduces to the LOO method. Moreover, for any Borel probability measure $\xi$ defined on $[0,1]$, a function $w^{(n)}$ defined as
\begin{align}
    w^{(n)}(j) := n \int_{0} ^1 t^{j-1} (1-t)^{n-j} d\xi (t)
    \label{eqn:general_weights}
\end{align}
satisfies the condition $\sum_{j=1} ^{n} \binom{n-1}{j-1} w^{(n)}(j)=n$, and thus the corresponding function $\psi_{\mathrm{semi}}(z^*; U, \mathcal{D}, w^{(n)})$ is a semivalue by Theorem \ref{thm:semivalue_representation}. We provide a detailed proof in Appendix.

The semivalue expressed in  \eqref{eqn:def_prob_data_shapley_value} provides a unified data valuation framework, but loses the uniqueness of the weights. This can be a potential drawback, but the following proposition shows that a semivalue is unique up to the sum of data values.

\begin{proposition}
Let $\psi_1$ and $\psi_2$ be two semivalues such that for any $U$ the sum of data values are same, \textit{i.e.},
\begin{align*}
    \sum_{z\in \mathcal{D}} \psi_1 (z; U, \mathcal{D})=\sum_{z\in \mathcal{D}} \psi_2 (z; U, \mathcal{D}).
\end{align*}
Then, the two semivalues are identical, \textit{i.e.}, $\psi_1=\psi_2$.
\label{prop:uniqueness_semivalue}
\end{proposition}
Proposition \ref{prop:uniqueness_semivalue} shows that any two semivalues are identical if they have the same sum of data values. 
In other words, if there is a weight function that could reflect a practitioner's prior knowledge on the total sum of values, then the semivalue based on the weight function is unique.

\subsection{Beta Shapley: Efficient Semivalue}
\label{s:betavalue}
The exact computation of $w^{(n)}(j)$ in Equation \eqref{eqn:general_weights} can be expensive or infeasible due to the integral. To address this, we propose Beta Shapley that has a closed form for the weight. To be more specific, we consider the Beta distribution with a pair of positive hyperparameters $(\alpha, \beta)$ for $\xi$. Then, the weight can be expressed as
\begin{align*}
    w_{\alpha,\beta} ^{(n)}(j) &:= n \int_{0} ^1 t^{j-1} (1-t)^{n-j} \frac{t^{\beta-1} (1-t)^{\alpha-1}}{\mathrm{Beta}(\alpha,\beta)}  dt \\
    &= n\frac{ \mathrm{Beta}(j+\beta-1,n-j+\alpha)}{\mathrm{Beta}(\alpha,\beta)},
\end{align*}
where $\mathrm{Beta}(\alpha,\beta)= \Gamma(\alpha)\Gamma(\beta)/\Gamma(\alpha+\beta)$ is the Beta function and $\Gamma(\cdot)$ is the Gamma function. This can be further simplified as the following closed form.
\begin{align}
    w_{\alpha,\beta} ^{(n)}(j) = n \frac{ \prod_{k=1} ^{j-1} (\beta +k-1) \prod_{k=1} ^{n-j} (\alpha+k-1)}{\prod_{k=1} ^{n-1} (\alpha+\beta+k-1)}.
    \label{eqn:closed_form_beta_weight}
\end{align}
We propose to use $\psi_{\mathrm{semi}}(z^*; U, \mathcal{D}, w_{\alpha,\beta} ^{(n)})$ and call it Beta$(\alpha, \beta)$-Shapley value.

The pair of hyperparameters $(\alpha,\beta)$ decides the weight distribution on $[n]$. For instance, when $(\alpha,\beta)=(1,1)$, the normalized weight $\tilde{w}_{\alpha,\beta}^{(n)}(j) := \binom{n-1}{j-1} w_{\alpha,\beta} ^{(n)}(j) = 1$ for all $j \in [n]$, giving the uniform weight on marginal contributions, \textit{i.e.}, Beta(1,1)-Shapley $\psi_{\mathrm{semi}}(z^*; U, \mathcal{D}, w_{1,1} ^{(n)})$ is exactly the original data Shapley. Figure \ref{fig:weight_distribution} shows various weight distributions for different pairs of $(\alpha, \beta)$. For simplicity, we fix one of the hyperparameter to be one. When $\alpha \geq \beta=1$, the normalized weight assigns large weights on the small cardinality and remove noise from the large cardinality. Conversely, Beta(1,$\beta$) puts more weights on large cardinality and it approaches to the LOO as $\beta$ increases. 

\begin{figure}[t]
    \centering
    \includegraphics[width=0.85\columnwidth]{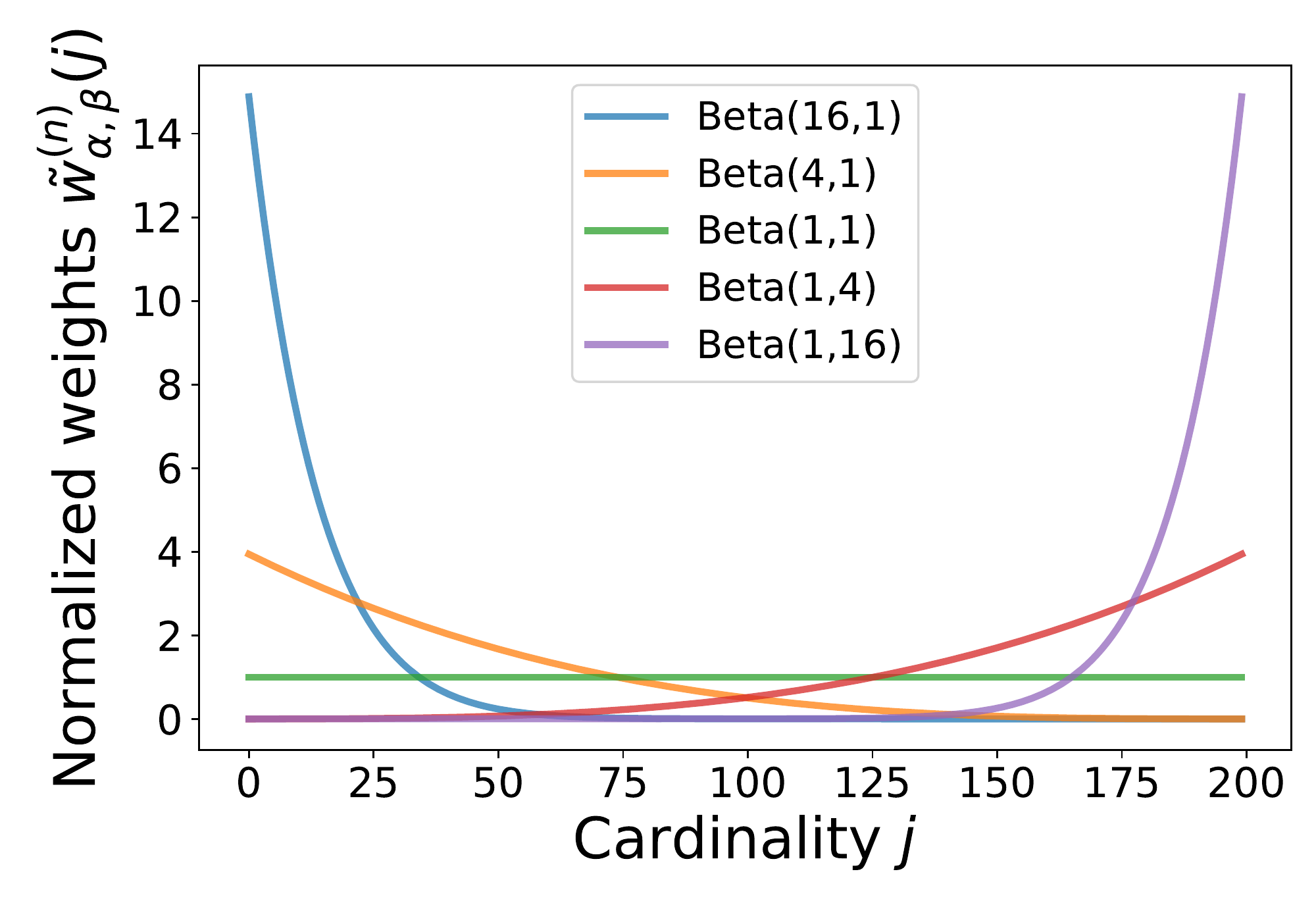}
    \caption{Illustration of the normalized weight $\tilde{w}_{\alpha,\beta}^{(n)}(j)$ for various pairs of $(\alpha,\beta)$ when $n=200$. Each color indicates a different hyperparameter pair $(\alpha,\beta)$.}
    \label{fig:weight_distribution}
\end{figure}

\paragraph{Beta($\alpha, \beta$)-Shapley for optimal subsampling weights}
When some data points in a given dataset are noisy or there are too many data points that could cause heavy computational costs, subsampling is an effective method to learn a model with a small number of high-quality samples. We show how Beta($\alpha, \beta$)-Shapley can be used to find the optimal importance weight that produces an estimator with the minimum variance in an asymptotic manner. We consider the Beta Shapley $\psi_{\mathrm{semi}}(z^*; h, \mathcal{D}, w_{\alpha,\beta} ^{(n)})$ for an M-estimator $h$. The M-estimator $h$ can include various estimators such as the maximum likelihood estimator \citep{van2000asymptotic}.

\begin{theorem}[Informal]
Let $\lambda_i$ be the importance weight for $i$-th sample $z_i$ and suppose $\lambda_i \propto \norm{\psi_{\mathrm{semi}}(z_i; h, \mathcal{D}, w_{\alpha,\beta} ^{(n)})}_2$ for some $\alpha \geq 1$. If we subsample based on the importance weight $\lambda_i$, then the M-estimator obtained by subsamples asymptotically achieves the minimum variance.
\label{thm:prob_shap_asymptotic}
\end{theorem}
We present a formal version of Theorem \ref{thm:prob_shap_asymptotic} with some technical conditions in Appendix.
This shows that the data value-based importance weight can be useful to capture the impact of one datum, and leads to the optimal estimator in the sense that it asymptotically achieves the minimum variance. 

\paragraph{Efficient estimation of Beta($\alpha, \beta$)-Shapley}
Although $w_{\alpha,\beta}^{(n)}(j)$ is easy-to-compute, the exact computation of Beta($\alpha, \beta$)-Shapley can be expensive because it requires an exponential number of model fittings. This could be a major challenge of using Beta($\alpha, \beta$)-Shapley in practice. To address this, we develop an efficient algorithm by adapting the Monte Carlo method \citep{maleki2013bounding, ghorbani2019}.

Beta($\alpha , \beta$)-Shapley can be expressed in the form of a weighted mean as follows.
\begin{align*}
    \frac{1}{n} \sum_{j=1} ^{n} \frac{1}{ |\mathcal{D}_{j} ^{\backslash z^*}| } \sum_{S \in \mathcal{D}_{j} ^{\backslash z^*} } \tilde{w}_{\alpha,\beta}^{(n)}(j) (h(S\cup \{z^*\})-h(S)).
\end{align*}
We note that $\tilde{w}_{\alpha,\beta}^{(n)}(j)=\binom{n-1}{j-1} w_{\alpha,\beta}^{(n)}(j)$ only needs to be calculated once using the closed form in Equation \eqref{eqn:closed_form_beta_weight}. As a result, it can be efficiently approximated by the Monte Carlo method: at each iteration, we first draw a number $j$ from the discrete uniform distribution from $[n]$, and randomly draw a subset $S$ is from a class of set $\mathcal{D}_{j} ^{\backslash z^*}$ uniformly at random. We then compute $\tilde{w}_{\alpha,\beta}^{(n)}(j) (h(S\cup \{z^*\})-h(S))$ and update the Monte Carlo estimates. We provide a pseudo algorithm in the Appendix.

\section{Numerical experiments}
\label{s:experiment}
In this section, we demonstrate the practical efficacy of Beta Shapley on various classification datasets. We conduct the three different ML tasks: noisy label detection, learning with subsamples, and point addition and removal experiments. We compare the eight methods: the \texttt{LOO-First} $\Delta_2 (z^*; U, \mathcal{D})$, the five variations of Beta$(\alpha,\beta)$-Shapley, the \texttt{LOO-Last} $\Delta_n (z^*; U, \mathcal{D})$ (which is the standard LOO), and the \texttt{KNN Shapley} proposed in \citet{jia2019b}. We use 15 standard datasets that are commonly used to benchmark classification methods and use a logistic regression classifier. Additional results with a support vector machine model and detailed information about experiment settings are provided in Appendix. 

\begin{table*}[t]
    \centering
    \caption{Comparison of mis-annotation detection ability of the eight data valuation methods on the fifteen classification datasets. The average and standard error of the F1-score are denoted by `average$\pm$standard error'. All the results are based on 50 repetitions. Boldface numbers denote the best method.}
    \resizebox{\textwidth}{!}{
    \begin{tabular}{l|cccccccc}
        \toprule
        Dataset & \texttt{LOO-First} & \texttt{Beta(16,1)} & \texttt{Beta(4,1)} & \texttt{Data Shapley} & \texttt{Beta(1,4)} & \texttt{Beta(1,16)} & \texttt{LOO-Last} & \texttt{KNN Shapley}\\
        \midrule
        Gaussian & $0.465\pm0.010$ & $\mathbf{0.470\pm0.010}$ & $0.454\pm0.012$ & $0.416\pm0.012$ & $0.255\pm0.012$ & $0.204\pm0.011$ & $0.147\pm0.013$ & $0.398\pm0.012$ \\
        Covertype & $0.324\pm0.011$ & $\mathbf{0.355\pm0.013}$ & $0.347\pm0.013$ & $0.337\pm0.012$ & $0.252\pm0.008$ & $0.236\pm0.008$ & $0.180\pm0.010$ & $0.278\pm0.012$ \\
        CIFAR10 & $0.252\pm0.011$ & $0.272\pm0.011$ & $\mathbf{0.276\pm0.011}$ & $0.272\pm0.012$ & $0.238\pm0.010$ & $0.213\pm0.008$ & $0.169\pm0.008$ & $0.259\pm0.010$ \\
        FMNIST & $0.487\pm0.012$ & $0.547\pm0.012$ & $\mathbf{0.555\pm0.011}$ & $0.523\pm0.013$ & $0.356\pm0.011$ & $0.271\pm0.011$ & $0.187\pm0.013$ & $0.484\pm0.017$ \\
        MNIST & $0.412\pm0.010$ & $0.482\pm0.011$ & $\mathbf{0.504\pm0.012}$ & $0.477\pm0.012$ & $0.345\pm0.011$ & $0.284\pm0.011$ & $0.203\pm0.013$ & $0.446\pm0.012$ \\
        Fraud & $\mathbf{0.623\pm0.009}$ & $0.591\pm0.016$ & $0.550\pm0.017$ & $0.427\pm0.022$ & $0.221\pm0.012$ & $0.233\pm0.015$ & $0.177\pm0.021$ & $0.491\pm0.013$ \\
        Apsfail & $0.624\pm0.015$ & $\mathbf{0.643\pm0.011}$ & $0.606\pm0.013$ & $0.494\pm0.019$ & $0.229\pm0.018$ & $0.236\pm0.017$ & $0.242\pm0.020$ & $0.483\pm0.014$ \\
        Click & $0.216\pm0.008$ & $\mathbf{0.218\pm0.009}$ & $0.214\pm0.009$ & $0.201\pm0.009$ & $0.174\pm0.009$ & $0.167\pm0.009$ & $0.140\pm0.011$ & $0.204\pm0.009$ \\
        Phoneme & $0.388\pm0.011$ & $0.409\pm0.011$ & $0.399\pm0.011$ & $0.350\pm0.012$ & $0.224\pm0.011$ & $0.192\pm0.011$ & $0.126\pm0.014$ & $\mathbf{0.446\pm0.011}$ \\
        Wind & $0.515\pm0.013$ & $\mathbf{0.521\pm0.015}$ & $0.515\pm0.016$ & $0.501\pm0.015$ & $0.248\pm0.011$ & $0.226\pm0.012$ & $0.163\pm0.016$ & $0.505\pm0.015$ \\
        Pol & $0.432\pm0.012$ & $0.451\pm0.011$ & $\mathbf{0.471\pm0.011}$ & $0.461\pm0.012$ & $0.229\pm0.010$ & $0.210\pm0.010$ & $0.184\pm0.015$ & $0.446\pm0.014$ \\
        Creditcard & $0.268\pm0.010$ & $\mathbf{0.270\pm0.010}$ & $0.265\pm0.012$ & $0.238\pm0.012$ & $0.194\pm0.011$ & $0.180\pm0.011$ & $0.152\pm0.010$ & $0.259\pm0.010$ \\
        CPU & $\mathbf{0.659\pm0.013}$ & $0.638\pm0.012$ & $0.613\pm0.014$ & $0.555\pm0.021$ & $0.281\pm0.018$ & $0.250\pm0.015$ & $0.212\pm0.021$ & $0.559\pm0.012$ \\
        Vehicle & $0.448\pm0.015$ & $0.469\pm0.015$ & $\mathbf{0.484\pm0.016}$ & $0.456\pm0.014$ & $0.360\pm0.013$ & $0.287\pm0.012$ & $0.217\pm0.015$ & $0.310\pm0.014$ \\
        2Dplanes & $0.518\pm0.009$ & $0.526\pm0.012$ & $0.505\pm0.012$ & $0.460\pm0.013$ & $0.278\pm0.013$ & $0.240\pm0.012$ & $0.177\pm0.018$ & $\mathbf{0.568\pm0.014}$ \\
        \midrule
        Average& 0.442& $\mathbf{0.458}$ & 0.451& 0.411& 0.259& 0.228& 0.178& 0.409\\
        \bottomrule
    \end{tabular}}
    \label{tab:noisy_label_detection}
\end{table*}

\subsection{Noisy label detection}
We first investigate the detection ability of Beta Shapley. As for detection rules, we use a clustering-based procedure as the number of mislabeled data points and the threshold for detecting noisy samples are usually unknown in practice. Specifically, we first divide all data values into two clusters using the K-Means clustering algorithm \citep{arthur2007k} and then classify a data point as a noisy sample if its value is less than the minimum of the two cluster centers. After this selection procedure, the F1-score is evaluated as a performance metric. We consider the two different types of label noise: synthetic noise and real-world label noise. As for the synthetic noise, we generate noisy samples by flipping labels for a random 10\% of training data points.

\paragraph{Synthetic noise}
Table~\ref{tab:noisy_label_detection} shows the F1-score of the eight data valuation methods. \texttt{Beta(16,1)}, which assigns larger weights to small cardinality, outperforms other variations of Beta Shapley as well as the baseline data valuation methods. In contrast, \texttt{Beta(1,4)} or \texttt{LOO-Last} perform much worse than other methods because they focus heavily on large cardinalities. \texttt{LOO-First} $\Delta_2 (z^*; U, \mathcal{D})$, which only considers the first marginal contribution, often suffers from a failure of training due to a small number of samples, and as a result, it performs worse than the \texttt{Beta(16,1)} and \texttt{Beta(4,1)} methods. This shows that focusing too much on only small cardinality can also degrade performance. 
We further compare Beta Shapley with the uncertainty-based method proposed by \citet{northcutt2021confident}. Although this method is not intended for data valuation, it is a state-of-the-art method of noise label detection and achieves average $F_1 = 0.42$ across the 15 datasets. This score is worse than \texttt{Beta(16,1)} or \texttt{Beta(4,1)}, showing \texttt{Beta(16,1)} Shapley values are very competitive in identifying mislabeled data points.

\begin{figure}[t]
    \centering
    \includegraphics[width=\columnwidth]{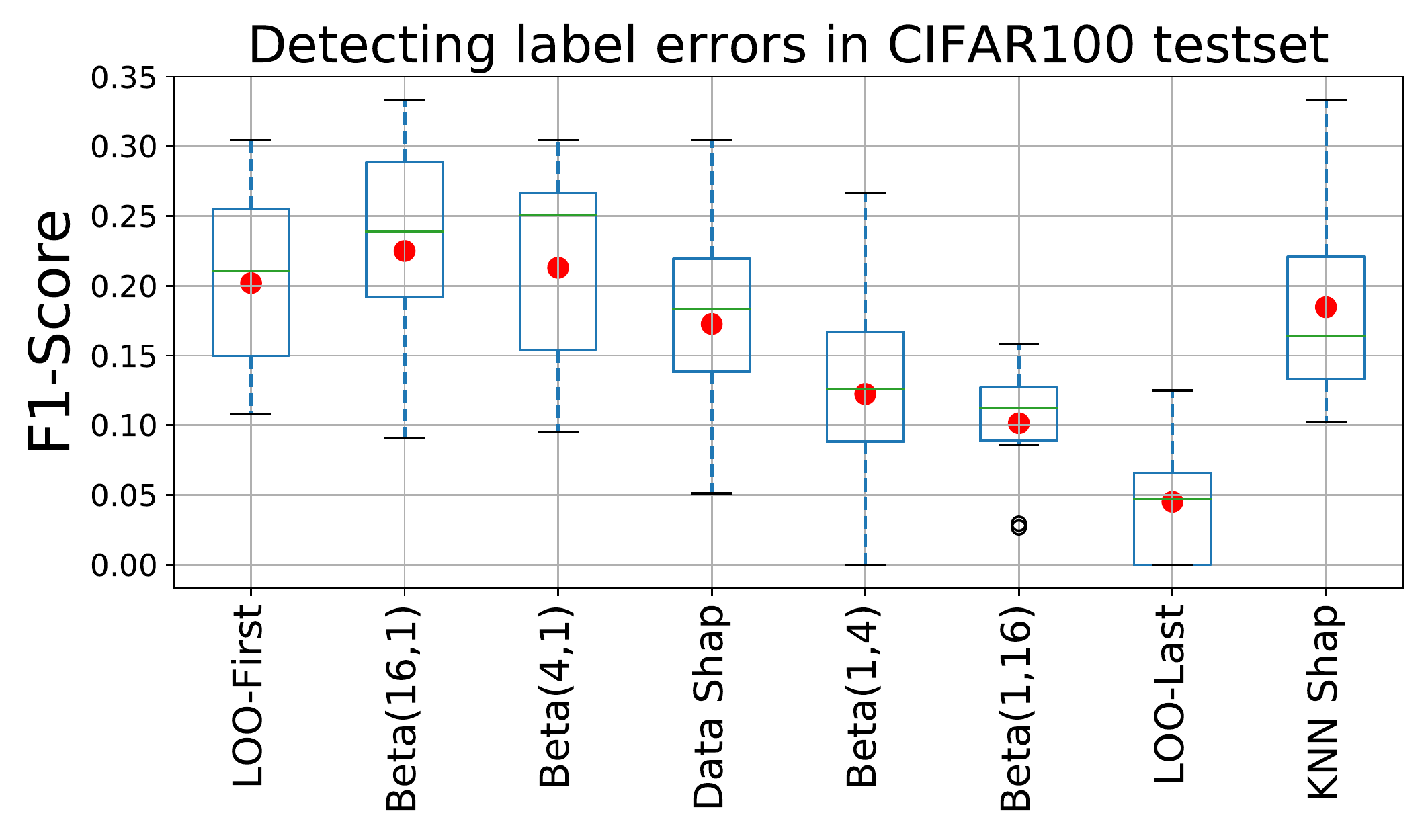}
    \caption{A boxplot of the F1-score of the data valuation methods on the CIFAR100 test dataset. The red dot indicates the mean of the F1-score. Beta Shapley that focuses on small cardinality detects mislabeled data points better than other baseline methods.}
    \label{fig:cifar100_boxplot}
\end{figure}

\begin{figure}[t]
    \centering
    \includegraphics[width=\columnwidth]{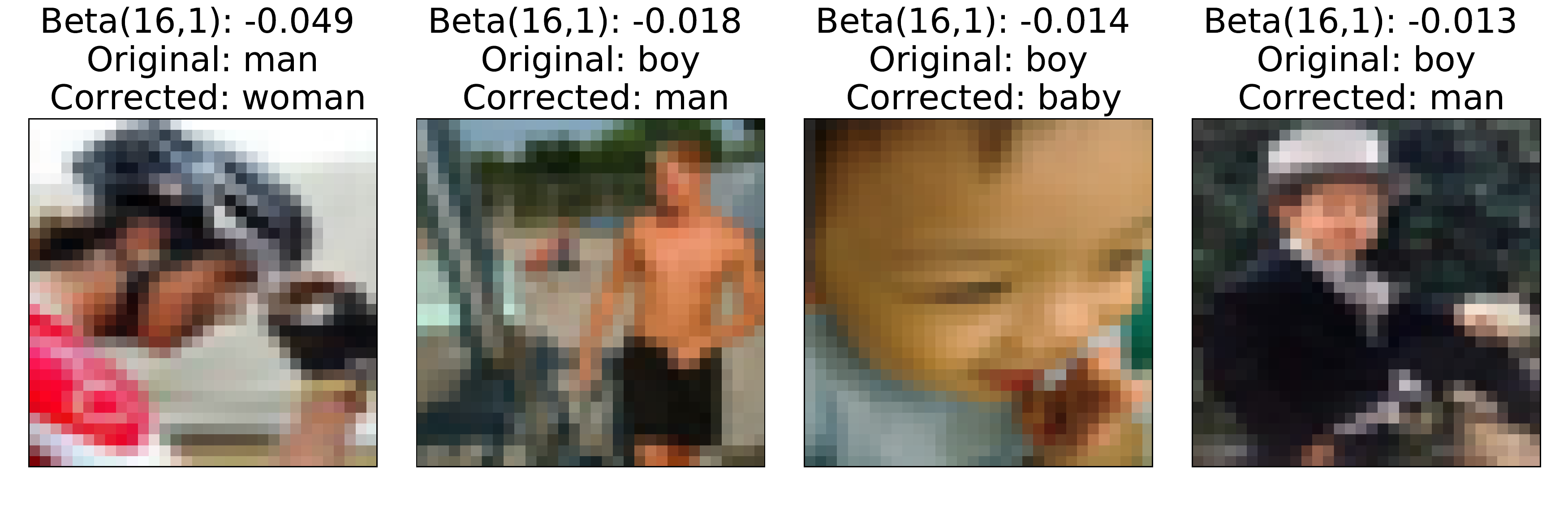}
    \includegraphics[width=\columnwidth]{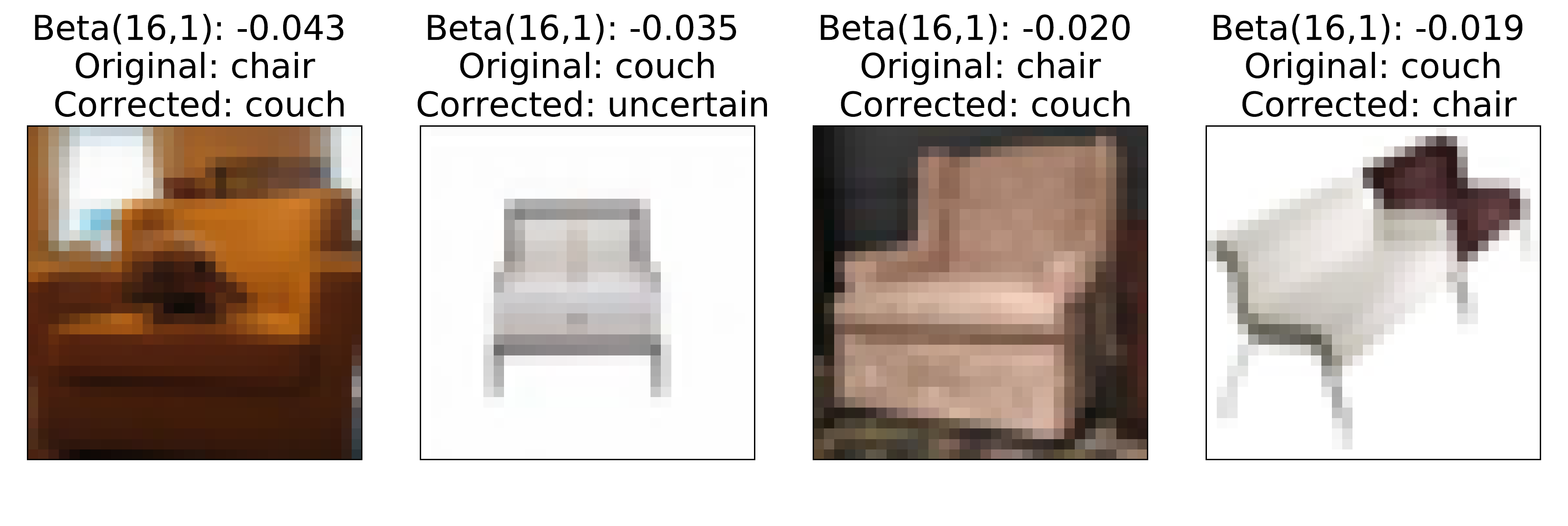}
    \caption{Examples of mislabeled images in the CIFAR100 test dataset. The corrected label suggested by \citet{northcutt2021pervasive} is provided for comparison. \texttt{Beta(16,1)} values for the mislabeled samples are negative, meaning that this type of labeling error can harm the model.}
    \label{fig:cifar_examples}
\end{figure}

\paragraph{Real-world label noise} 
We next apply data valuation methods to detect real-world label errors in the CIFAR100 test dataset. \citet{northcutt2021pervasive} estimated 5.85\% of the images in the CIFAR100 test dataset were indeed mislabeled. We choose the ten most confusing class pairs in the CIFAR100 test dataset. For each pair of classes, we detect mislabeled points in a binary classification setting. 
Figure~\ref{fig:cifar100_boxplot} shows a boxplot of the F1-score of the different data valuation methods. \texttt{Beta(16,1)} achieves $0.225$ and outperforms other methods.
For the three pairs of classes with the largest number of mislabels---(willow tree, maple tree), (pine tree, oak tree), (oak tree, maple tree)---we also compared the detection performance of \texttt{Beta(16,1)} with the state-of-the-art uncertainty-based method proposed in \citet{northcutt2021confident}. The uncertainty-based method was developed in the same paper that identified the CIFAR100 misannotations, so it is a strong benchmark. \texttt{Beta(16,1)} and the uncertainty-based methods achieve 0.307 and 0.273 F1-score, respectively, showing \texttt{Beta(16,1)} is effective in identifying real-world label errors. 
Figure~\ref{fig:cifar_examples} shows representative examples of mislabeled images and Beta Shapley rightfully assigned negative values for them.
 
\begin{table*}[t]
    \centering
    \caption{Accuracy comparison of models trained with subsamples. We compare the eight data valuation methods on the fifteen classification datasets. The \texttt{Random} denotes learning a model with subsamples drawn uniformly at random. The average and standard error of classification accuracy are denoted by `average$\pm$standard error'. All the results are based on 50 repetitions. Boldface numbers denote the best method.}
    \resizebox{\textwidth}{!}{
    \begin{tabular}{lccccccccccccccccc}
        \toprule
        Dataset & \texttt{LOO-First} & \texttt{Beta(16,1)}& \texttt{Beta(4,1)}& \texttt{Data Shapley} & \texttt{Beta(1,4)}& \texttt{LOO-Last}& \texttt{KNN Shapley} & \texttt{Random}\\
        \midrule
        Gaussian & $0.763\pm0.003$ & $\mathbf{0.765\pm0.002}$ & $0.760\pm0.003$ & $0.732\pm0.005$ & $0.598\pm0.008$ & $0.569\pm0.015$ & $0.749\pm0.005$ & $0.727\pm0.007$ \\
        Covertype & $0.645\pm0.005$ & $0.661\pm0.005$ & $\mathbf{0.670\pm0.005}$ & $0.661\pm0.004$ & $0.607\pm0.007$ & $0.567\pm0.008$ & $0.635\pm0.006$ & $0.636\pm0.006$ \\
        CIFAR10 & $0.625\pm0.004$ & $\mathbf{0.628\pm0.003}$ & $0.624\pm0.003$ & $0.617\pm0.003$ & $0.575\pm0.004$ & $0.553\pm0.006$ & $0.580\pm0.004$ & $0.582\pm0.005$ \\
        FMNIST & $0.823\pm0.004$ & $\mathbf{0.842\pm0.003}$ & $0.840\pm0.004$ & $0.830\pm0.003$ & $0.726\pm0.008$ & $0.614\pm0.012$ & $0.801\pm0.006$ & $0.752\pm0.007$ \\
        MNIST & $0.762\pm0.004$ & $\mathbf{0.773\pm0.004}$ & $0.770\pm0.003$ & $0.753\pm0.004$ & $0.673\pm0.006$ & $0.607\pm0.009$ & $0.725\pm0.005$ & $0.702\pm0.006$ \\
        Fraud & $0.883\pm0.003$ & $0.881\pm0.003$ & $0.883\pm0.002$ & $0.873\pm0.006$ & $0.567\pm0.019$ & $0.637\pm0.037$ & $\mathbf{0.886\pm0.003}$ & $0.866\pm0.004$ \\
        Apsfail & $0.866\pm0.003$ & $0.877\pm0.003$ & $\mathbf{0.878\pm0.003}$ & $0.870\pm0.004$ & $0.671\pm0.023$ & $0.477\pm0.034$ & $0.864\pm0.003$ & $0.858\pm0.003$ \\
        Click & $0.566\pm0.004$ & $\mathbf{0.567\pm0.003}$ & $0.566\pm0.003$ & $0.561\pm0.004$ & $0.535\pm0.004$ & $0.520\pm0.004$ & $0.551\pm0.004$ & $0.538\pm0.005$ \\
        Phoneme & $0.741\pm0.002$ & $\mathbf{0.744\pm0.003}$ & $0.743\pm0.002$ & $0.738\pm0.003$ & $0.581\pm0.009$ & $0.567\pm0.017$ & $0.727\pm0.004$ & $0.712\pm0.005$ \\
        Wind & $0.801\pm0.003$ & $0.804\pm0.002$ & $0.809\pm0.003$ & $0.796\pm0.004$ & $0.569\pm0.013$ & $0.549\pm0.028$ & $\mathbf{0.811\pm0.003}$ & $0.800\pm0.003$ \\
        Pol & $0.750\pm0.004$ & $0.731\pm0.004$ & $0.748\pm0.004$ & $0.746\pm0.005$ & $0.543\pm0.013$ & $0.532\pm0.018$ & $\mathbf{0.762\pm0.006}$ & $0.734\pm0.005$ \\
        Creditcard & $0.625\pm0.003$ & $0.632\pm0.003$ & $\mathbf{0.637\pm0.003}$ & $0.632\pm0.004$ & $0.571\pm0.006$ & $0.528\pm0.006$ & $0.595\pm0.004$ & $0.584\pm0.007$ \\
        CPU & $0.848\pm0.004$ & $0.870\pm0.003$ & $\mathbf{0.872\pm0.004}$ & $0.862\pm0.004$ & $0.628\pm0.015$ & $0.545\pm0.029$ & $0.862\pm0.004$ & $0.858\pm0.004$ \\
        Vehicle & $0.754\pm0.005$ & $0.770\pm0.003$ & $\mathbf{0.772\pm0.004}$ & $0.761\pm0.005$ & $0.675\pm0.009$ & $0.592\pm0.010$ & $0.729\pm0.005$ & $0.728\pm0.007$ \\
        2Dplanes & $0.802\pm0.003$ & $\mathbf{0.806\pm0.002}$ & $0.803\pm0.003$ & $0.796\pm0.003$ & $0.631\pm0.009$ & $0.615\pm0.018$ & $0.777\pm0.005$ & $0.755\pm0.006$ \\
        \midrule
        Average& 0.750& $\mathbf{0.757}$& $\mathbf{0.757}$& 0.749& 0.612& 0.564& 0.735& 0.722\\
        \bottomrule
    \end{tabular}}
    \label{tab:subsampling}
\end{table*}

\subsection{Learning with subsamples}
We now examine how the data value-based importance weight can be applied to subsample data points. We train a model with $25\%$ of the given dataset by using the importance weight $\max ( \psi_{\mathrm{semi}}(z_i; U, \mathcal{D}, w_{\alpha,\beta} ^{(n)}),0 )$ for the $i$-th sample. With this importance weight, data points with higher values are more likely to be selected and data points with negative values are not used. We train a classifier to minimize the weighted loss then evaluate the accuracy on the held-out test dataset.
As Table~\ref{tab:subsampling} shows, \texttt{Beta(4,1)} shows the best overall performance, but \texttt{Beta(16,1)} also shows very similar performance. \texttt{Beta(1,4)} and \texttt{LOO-Last} perform worse than uniform sampling, suggesting that the marginal contributions at large cardinality are not useful to capture the importance of data.

\begin{figure}[t]
    \centering
    \includegraphics[width=0.47\columnwidth]{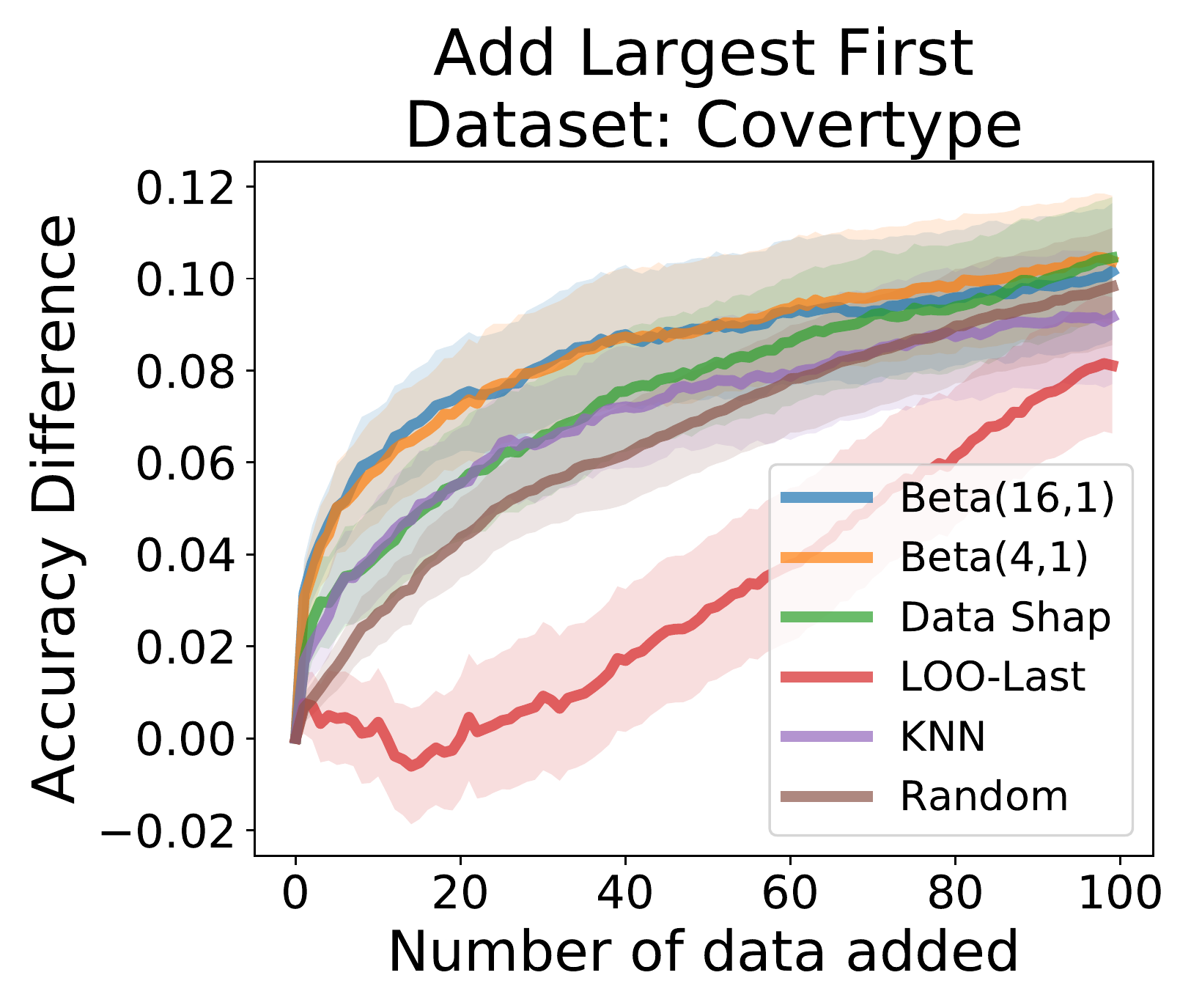}
    \includegraphics[width=0.47\columnwidth]{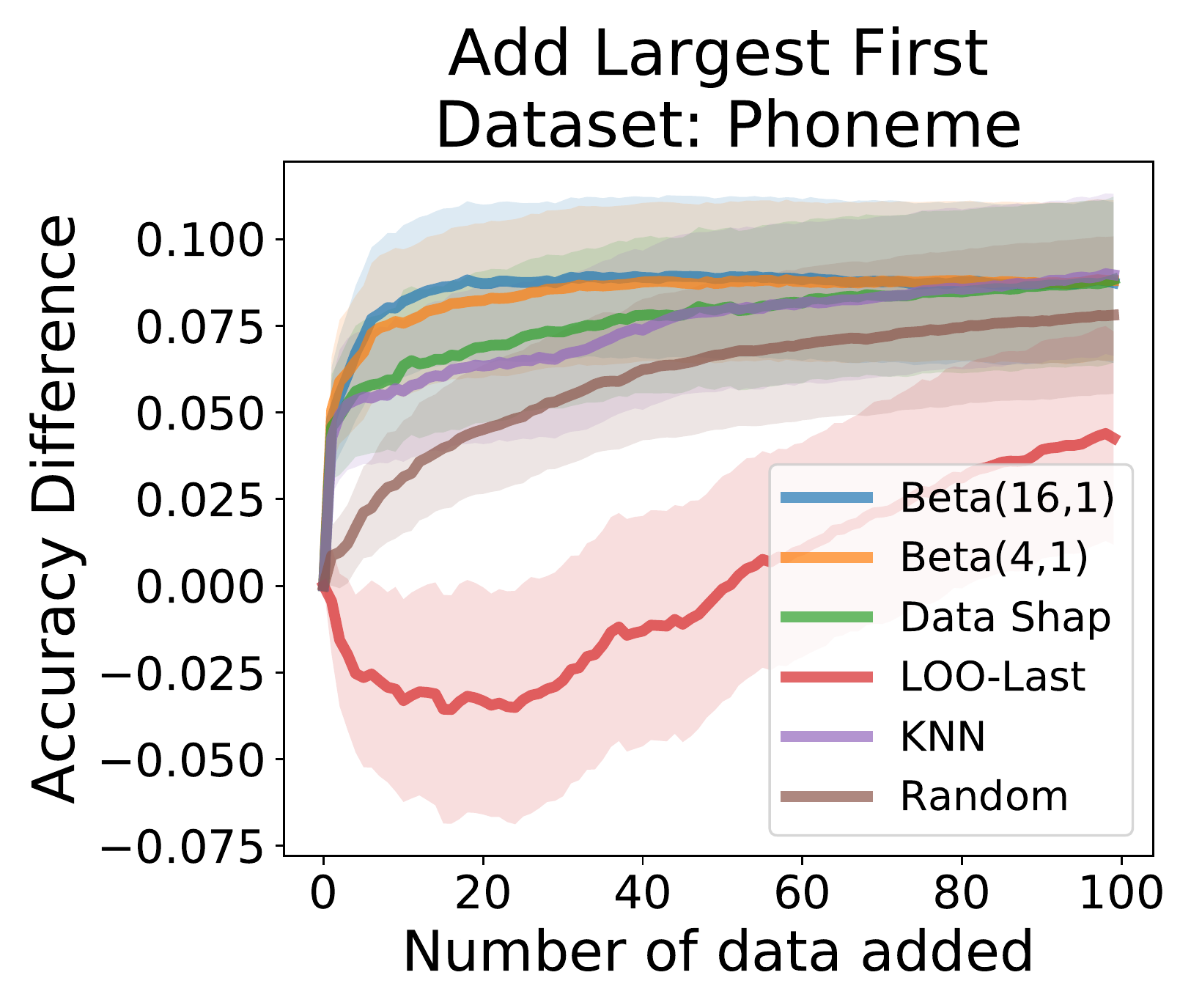}\\
    \includegraphics[width=0.47\columnwidth]{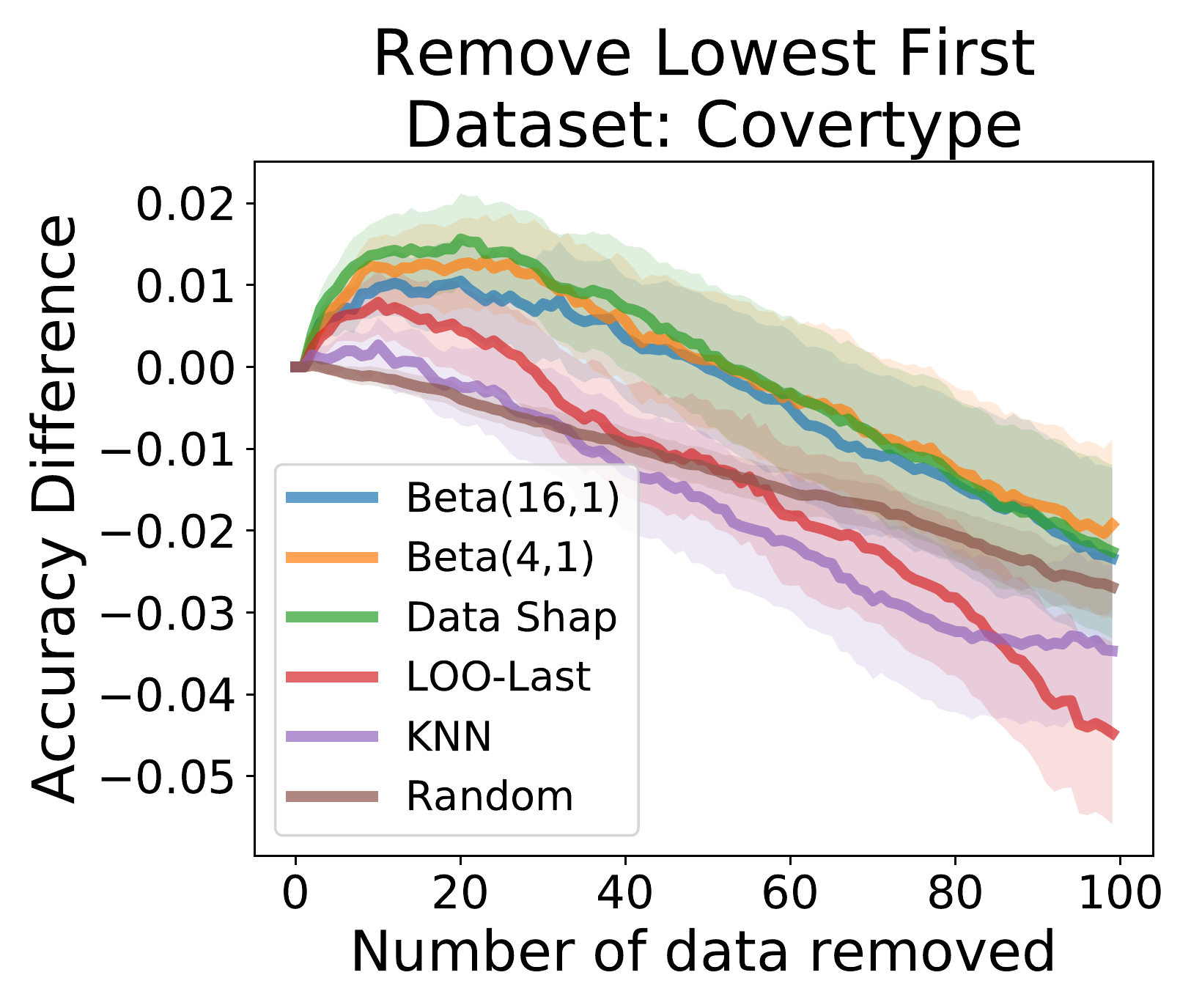}
    \includegraphics[width=0.47\columnwidth]{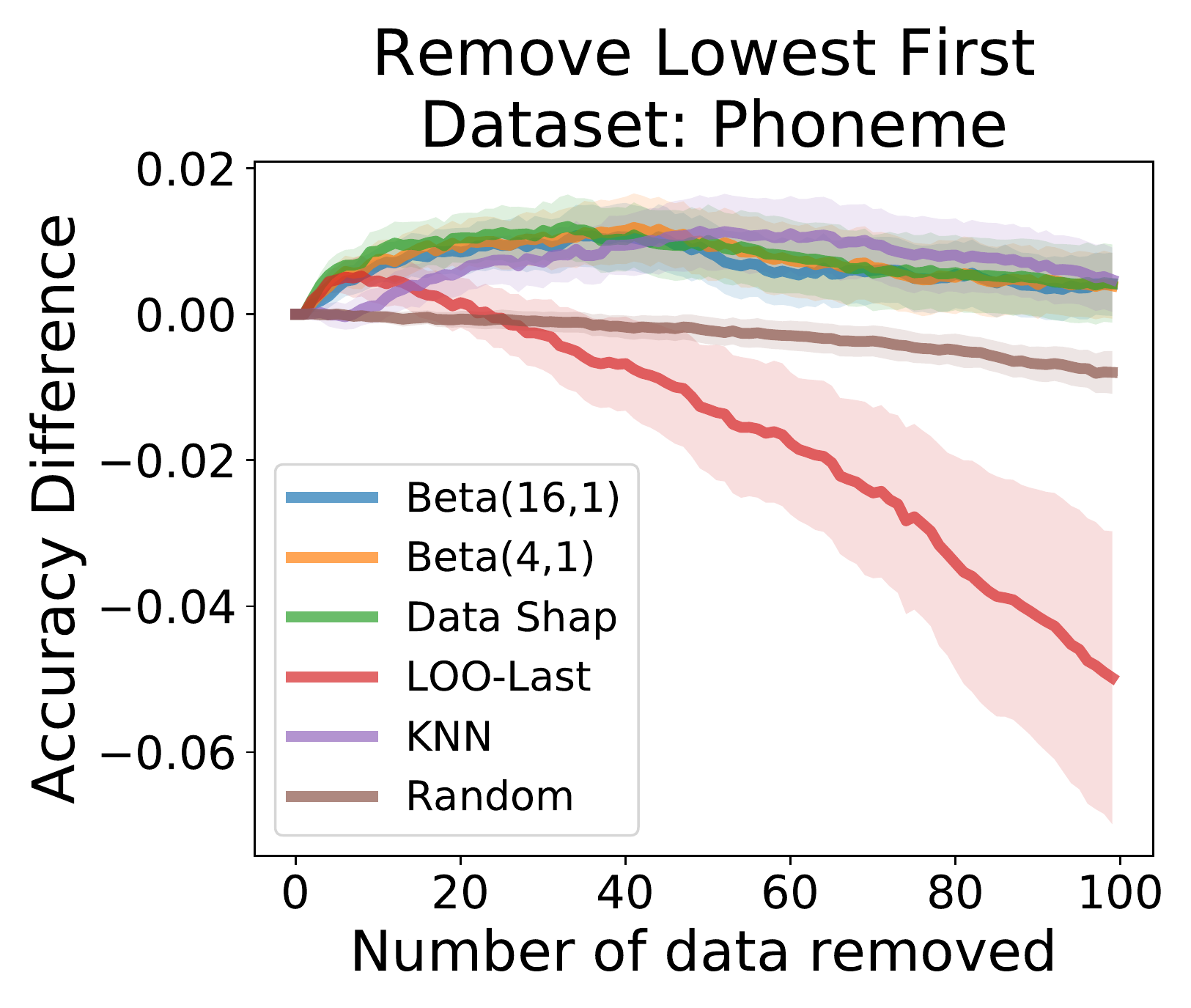}
    \caption{Accuracy difference as a function of the number of data points (top) added or (bottom) removed. We add (\textit{resp.} remove) data points whose value is large (\textit{resp.} small) first. We denote a 95\% confidence band based on 50 repetitions. We provide additional results on different datasets in Appendix. }
    \label{fig:point_addition_removal_experiment}
\end{figure}

\begin{figure}[t]
    \centering
    \hspace{-0.3in}
    \includegraphics[width=\columnwidth]{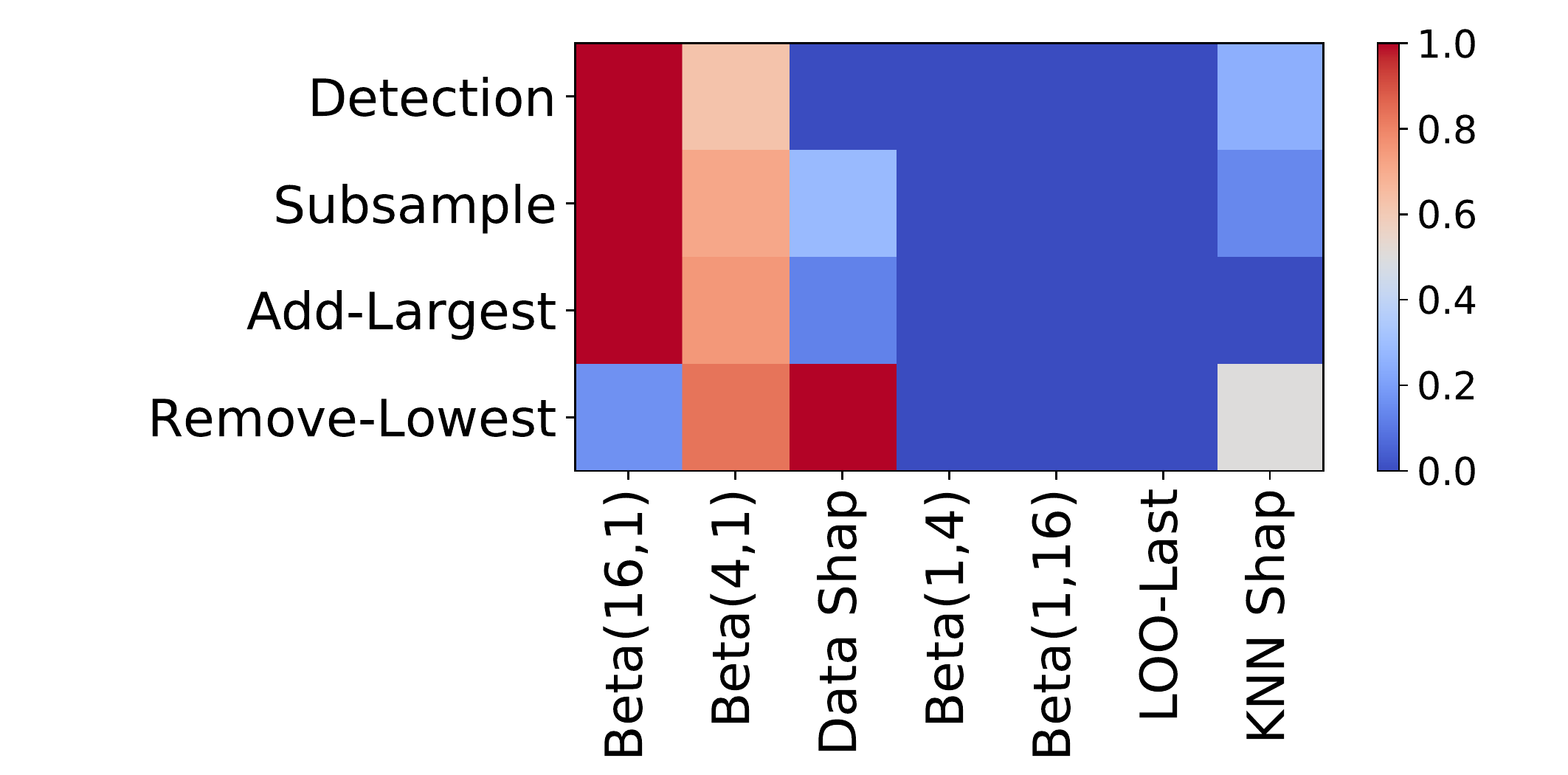}
    \caption{A summary of performance comparison on the fifteen datasets. Each element of the heatmap represents a linearly scaled frequency for each task to be between 0 and 1. Better and worse methods are depicted in red and blue respectively.} 
    \label{fig:heatmap_summary_count}
\end{figure}

\subsection{Point addition and removal experiments}
We now conduct point addition and removal experiments which are used to evaluate previous data valuation methods \citet{ghorbani2019}. For point addition experiments, we add data points from largest to lowest values because adding helpful data points first is expected to increase performance (similar to active learning setup). For the removal experiments, we remove data points from lowest to largest values because it is desirable to remove harmful or noisy data points first to increase model performance. At each step of addition or removal, we retrain a model with the current dataset and evaluate the accuracy changes on the held-out test dataset.
For point addition, \texttt{Beta(16,1)} shows the most rapid gain from identifying valuable points by putting more weights on small cardinality marginals ( Figure~\ref{fig:point_addition_removal_experiment}).
For removal, data Shapley performs slightly better than other methods. This is because we remove data points from the entire dataset, so the uniform weight can capture the effect of large cardinality parts better than other methods.

Finally, we summarize all of our experiments in a heatmap (Figure~\ref{fig:heatmap_summary_count}). For each ML task and data valuation method, we count the number of datasets where the method is the best performer and linearly transform it to be between 0 and 1. \texttt{Beta(16,1)}, which focuses on small cardinalities, is consistently the best method on the detection, subsampling, and point addition tasks. Data Shapley, which is \texttt{Beta(1,1)} and puts equal weights on all cardinalities, is the top performer in the point removal task. 

\section{Concluding remarks}
This work develops Beta Shapley to unify and extend popular data valuation methods like LOO and Data Shapley. Beta Shapley has desirable statistical and computational properties. We find that marginal contributions based on small cardinality are likely to have larger signal-to-noise, which is why \texttt{Beta(16,1)} works well in many settings. Our extensive experiments show that Beta Shapley that weighs small cardinalities more (e.g. \texttt{Beta(16,1)}) outperforms Data Shapley, LOO and other state-of-the-art methods. 

There are many interesting future works in this area. Our sampling-based algorithm provides an efficient implementation of data valuation, but the development of scalable algorithms that can be applied to a large-scale dataset is critical for the practical use of data values in practice. As an orthogonal direction, Beta Shapley opens up a new question about how to define and obtain the optimal weight representing data values. We believe it can depend on several factors including ML task or data distribution. 

\section*{Acknowledgment}
This research is supported by funding from Stanford AI Lab, Chan-Zuckerberg Biohub and the NSF CAREER \#1942926.

\bibliographystyle{apalike}
\bibliography{ref}

\onecolumn
\aistatstitle{Appendix of Beta Shapley: a Unified and Noise-reduced Data Valuation Framework for Machine Learning}

\vspace{-0.6in}

\appendix

In Appendix, we provide implementation details in Section~\ref{app:imple_details}, additional explanations in Section~\ref{app:missing_details}, and proofs in Section~\ref{app:proofs}. In addition, we provide additional numerical results and demonstrate robustness of our results against different datasets and a model using a support vector machine in Section~\ref{app:add_exp}. Our implementation codes are available at \url{https://github.com/ykwon0407/beta_shapley}.

\section{Implementation details}
\label{app:imple_details}

\paragraph{The proposed algorithm}
We propose to use the sampling-based Monte Carlo (MC) method to approximate Beta Shapley value: at each iteration, we first draw a number $k$ from the discrete uniform distribution from $[n]$, and randomly draw a subset $S$ is from a class of set $\mathcal{D}_{k} ^{\backslash z^*}$ uniformly at random. We then compute $\tilde{w}_{\alpha,\beta}^{(n)}(k) (h(S\cup \{z^*\})-h(S))$ and update the MC estimates. As for the utility computation, we can use a held-out validation set of samples from $P_{Z}$.

\paragraph{Accuracy of the proposed algorithms}
The propose algorithm is based on the MC method, and thus it guarantees to converge to the true value if we repeat the sampling procedure. In our experiments, we stop the sampling procedure when the new increment is small enough compared to the current MC estimate. To do this, we evaluate the Gelman-Rubin statistic for data values, which is well known for one of the most popular convergence diagnostic methods \citep[Equation (4)]{vats2021revisiting}. We set the number of Markov chains as 10 and terminate the sampling procedure if the Gelman-Rubin statistic for all data values is less than $1.0005$ to ensure accurate approximation, which is much less than a typical terminating threshold $1.1$ \citep{gelman1995bayesian}. We provide a pseudo algorithm in Algorithm \ref{alg:Beta_shapley}.

\begin{algorithm}[h]
\caption{Efficient computation algorithm for Beta$(\alpha,\beta)$-Shapley}
\begin{algorithmic}
\Require A set to be valued $\mathcal{D}=\{z_1, \dots, z_n\}$. A utility function $U$. A terminating threshold $\rho$ (in our experiment $\rho=1.0005$).
\Procedure{}{}
\State Initialize $\hat{\rho}=2\rho$, $B=1$, $\nu(j) = 0$ for all $j\in[n]$.
\State Compute $\tilde{w}_{\alpha,\beta}^{(n)} (j) =\binom{n-1}{j-1} w_{\alpha,\beta} ^{(n)}(j)$ for all $j \in [n]$.
\While{$\hat{\rho} \geq \rho$}
\For{$j \in [n]$}
\State Sample $k$ a uniform distribution from $[n]$.
\State Sample $S \in \mathcal{D}_{k} ^{\backslash z_{j}}$ uniformly at random.
\State Update $\nu(j) \leftarrow \frac{B-1}{B}\nu(j) +  \frac{1}{B}\tilde{w}_{\alpha,\beta}^{(n)} (k)(U(S\cup\{z_{j}\})-U(S))$
\EndFor
\State Update the Gelman-Rubin statistic $\hat{\rho}$. 
\State $B \leftarrow B + 1$.
\EndWhile
\EndProcedure
\end{algorithmic}
\label{alg:Beta_shapley}
\end{algorithm}

\begin{table}[h]
    \centering
    \caption{A summary of datasets used in numerical experiments.}
    \begin{tabular}{lcccccccccccc}
    \toprule
    Dataset & Sample size  & Input dimension & Source    \\ 
    \midrule
    \texttt{Gaussian} & 50000  & 5  & Synthetic dataset \\
    \texttt{Covertype} & 581012 & 54 & \citet{blackard1998comparison}  \\
    \texttt{CIFAR10} & 60000  & 32 & \citet{krizhevsky2009learning} \\
    \texttt{Fashion-MNIST} & 60000  & 32 & \citet{xiao2017fashion} \\
    \texttt{MNIST} & 60000  & 32  & \citet{lecun2010mnist} \\
    \texttt{Fraud} & 284807 & 31 & \citet{dal2015calibrating} \\
    \texttt{Creditcard} & 30000 & 24  & \citet{yeh2009comparisons} \\
    \texttt{Vehicle} & 98528 & 101  &  \citet{duarte2004vehicle}  \\
    \texttt{Apsfail} & 76000 & 171 & \url{https://www.openml.org/d/41138} \\
    \texttt{Click} & 1997410 & 12 &  \url{https://www.openml.org/d/1218} \\
    \texttt{Phoneme} & 5404 & 6 & \url{https://www.openml.org/d/1489}  \\
    \texttt{Wind} & 6574 & 15 &  \url{https://www.openml.org/d/847}  \\
    \texttt{Pol} & 15000 & 49 & \url{https://www.openml.org/d/722} \\
    \texttt{CPU} & 8192 & 22 & \url{https://www.openml.org/d/761} \\
    \texttt{2DPlanes} & 40768 & 11 & \url{https://www.openml.org/d/727} \\
    \bottomrule
    \end{tabular}
    \label{tab:summary_real_datasets}
\end{table}

\paragraph{Datasets used in Figure~\ref{fig:snr} of the manuscript}
We use the two synthetic datasets, regression and classification settings. As for the regression dataset, we generate input data from a 10-dimensional multivariate Gaussian with zero mean and the identity covariance matrix, \textit{i.e.}, $x_i \sim \mathcal{N}(0, I_{10})$. The output is generated as $y_i = x_i ^T \beta_0 + \varepsilon_i$, where $\beta_0 \sim \mathcal{N}(0, I_{10})$ and $\varepsilon_i \sim \mathcal{N}(0, 1)$. As for the classification dataset, we generate input data $x_i \sim \mathcal{N}(0, I_3)$. For outputs, we draw from a Bernoulli distribution $y_i = \mathbf{Bern}(\pi_i)$ for all $i \in [n]$. Here $\pi_i := \exp(x_i ^T \beta)/(1+\exp(x_i ^T \beta))$ for $\beta=(5,0,0)$. 

\paragraph{Datasets used in Figure~\ref{fig:clean_noisy_marginal_contributions} and Section~\ref{s:experiment} of the manuscript}
We use the one synthetic dataset and the fourteen real datasets.
For the synthetic dataset, \texttt{Gaussian}, we generate data as follows. Given a sample size $n$, we generate input data from a 5-dimensional multivariate Gaussian with zero mean and the identity covariance matrix, \textit{i.e.}, $x_i \sim \mathcal{N}(0, I_5)$. For outputs, we draw from a Bernoulli distribution $y_i = \mathbf{Bern}(\pi_i)$ for all $i \in [n]$. Here $\pi_i := \exp(x_i ^T \beta)/(1+\exp(x_i ^T \beta))$ for $\beta=(2,1,0,0,0)$. 
For the real datasets, we collect datasets from multiple sources including \texttt{OpenML}\footnote{https://www.openml.org/}. A comprehensive list of datasets and details on sample size and data source are provided in Table \ref{tab:summary_real_datasets}. We preprocess datasets to ease the training.

If the original dataset is the multi-class classification dataset (e.g. \texttt{Covertype}), we binarize the label by considering $\mathds{1}(y=1)$. For \texttt{OpenML} and \texttt{Covertype} datasets, we consider oversampling a minor class to balance positive and negative labels. For the image datasets \texttt{Fashion-MNIST}, \texttt{MNIST} and \texttt{CIFAR10}, we follow the common procedure in prior works \citep{ghorbani2020distributional, kwon2021efficient}: we extract the penultimate layer outputs from the pre-trained ResNet18 \citep{he2016deep}. The pre-training is done with the ImageNet dataset \citep{russakovsky2015imagenet} and the weight is publicly available from \texttt{Pytorch} \citep{NEURIPS2019_9015}. Using the extracted outputs, we fit a principal component analysis model and select the first 32 principal components.

\subsection{Model}
Throughout the experiment, we use a logistic regression model or a support vector machine model using the Python module \texttt{scikit-learn} \citep{scikitlearn}. As for KNN Shapley \citep{jia2019}, we used the $k$-nearest neighborhood classifier with $k=10$.   

\subsection{Experiment settings}
As for Figure~\ref{fig:snr} of the manuscript, we consider 500, 500, and 2000 samples for the dataset to be valued $\mathcal{D}$, the validation dataset, and the held-out test dataset, respectively. The validation dataset is used to estimate utility, and all the results are based on this held-out dataset. Except for this experiment, we use 200 samples for $\mathcal{D}$ and 200 samples for the validation dataset. For the held-out test dataset, we randomly choose 1000 samples. As for the experiments in Figure~\ref{fig:clean_noisy_marginal_contributions} and Section~\ref{s:experiment}, we randomly flip a label for 10\% of samples for $\mathcal{D}$ and the validation datasets. Below we provide details on ML tasks.

\paragraph{Noisy label detection}
Suppose $z^{(1)}, \dots, z^{(n)}$ are data points such that they satisfy the ordering $\psi(z^{(1)}) \leq \dots \leq \psi(z^{(n)})$. We fit the K-Means clustering algorithm on $\{ \psi(z^{(1)}) , \dots , \psi(z^{(n)}) \}$, diving into the two non-intersect sets $\{ \psi(z^{(1)}) , \dots , \psi(z^{(B)}) \}$ and $\{ \psi(z^{(B+1)}) , \dots , \psi(z^{(n)}) \}$. Note that $B$ is not necessarily the number of noisy samples. We define a detection rule as follows: we select $z$ is noisy if the data value is less than or equal to the lower cluster mean. That is, if $\psi(z) \leq \frac{1}{B} \sum_{i=1} ^B  \psi(z^{(i)})$, then $z$ is a noisy data point. After that we compute a F1-score, a harmonic mean of precision and recall of the rule, where
\begin{align*}
    \mathrm{Recall} &= \frac{|\{z : z \text{ is flipped and selected by the rule} \}|}{|\{z : z \text{ is flipped} \}|}\\
    \mathrm{Precision} &= \frac{|\{z : z \text{ is flipped and selected by the rule} \}|}{|\{z : z \text{ is selected by the rule} \}|}.
\end{align*}
    
\paragraph{Learning with subsamples}
We consider a situation where we select 50 samples among 200 samples, which is 25\% of $\mathcal{D}$. For a data valuation $\nu$, let $\lambda_i(\nu) = \max(\nu(z_i), 0)$ be the importance weight for sample $z_i=(x_i,y_i)$. Then, we compare the test accuracy of a weighted risk minimizer $f_{\nu}$ defined as
\begin{align*}
    f_{\nu} := \mathrm{argmin}_{f} \sum_{j \in \mathcal{S}_{50}} \frac{1}{\lambda_i(\nu)} (y_i \log f(x_i) + (1-y_i)\log(1-f(x_i)),
\end{align*}
where $\mathcal{S}_{50}$ is a set of the 50 subsamples. Here, the inverse weight $\frac{1}{\lambda_i(\nu)}$ is used to consider an unbiased risk minimizer. Note that this inverse propensity is used in the Horvitz-Thompson empirical measure. After obtaining $f_\nu$, we compute unweighted test accuracy using the held-out test dataset.

\paragraph{Point addition and removal experiments}
In Figure~\ref{fig:heatmap_summary_count}, we use the relative area as the performance of point addition and removal experiments. Specifically, for the removal task, we compute the relative area as follows.
\begin{align*}
    \text{Relative area-removal}(\psi) := \sum_{k=1} ^{n/2} \left\{ U (\mathcal{D} \backslash \{ z^{(1)}, \dots,  z^{(k)} \}) - U(\mathcal{D}) \right\},
\end{align*}
where $z^{(1)}, \dots, z^{(n)}$ satisfy the ordering $\psi(z^{(1)}) \leq \dots \leq \psi(z^{(n)})$. Similarly, for the addition task, we consider
\begin{align*}
    \text{Relative area-addition}(\psi) := \sum_{k=1} ^{n/2} \left\{ U (\mathcal{S} \cup \{ z^{(n)}, \dots,  z^{(n-k+1)} \}) - U(\mathcal{S}) \right\},
\end{align*}
where $\mathcal{S}$ is the initial set. In our experiment, $\mathcal{S}$ is randomly selected from $\mathcal{D}$ and $|\mathcal{S}|=10$.

\section{Additional explanations}
\label{app:missing_details}
In this section, we provide further details regarding Theorem \ref{thm:u_stat_for_large_cardinality}, Theorem \ref{thm:prob_shap_asymptotic}, and the Equation \ref{eqn:general_weights} of the manuscript.

\subsection{A bound condition in Theorem \ref{thm:u_stat_for_large_cardinality}}
One drawback of Theorem \ref{thm:u_stat_for_large_cardinality} is that it is unknown whether $\lim_{j \to \infty} \zeta_{j}/(j\zeta_{1})$ is bounded. Although Theorem 1 in \citet{hoeffiding1948class} showed a lower bound is greater than 1, \textit{i.e.}, $1 \leq \zeta_{j}/(j\zeta_{1})$ for any $j$, the existence of an upper bound has not been shown in literature. In Figure \ref{fig:condition_for_thm}, we show that this condition is plausible in our numerical examples. The details on dataset is given in Appendix~\ref{app:imple_details}.

\begin{figure}[h]
    \centering
    \includegraphics[width=0.35\columnwidth]{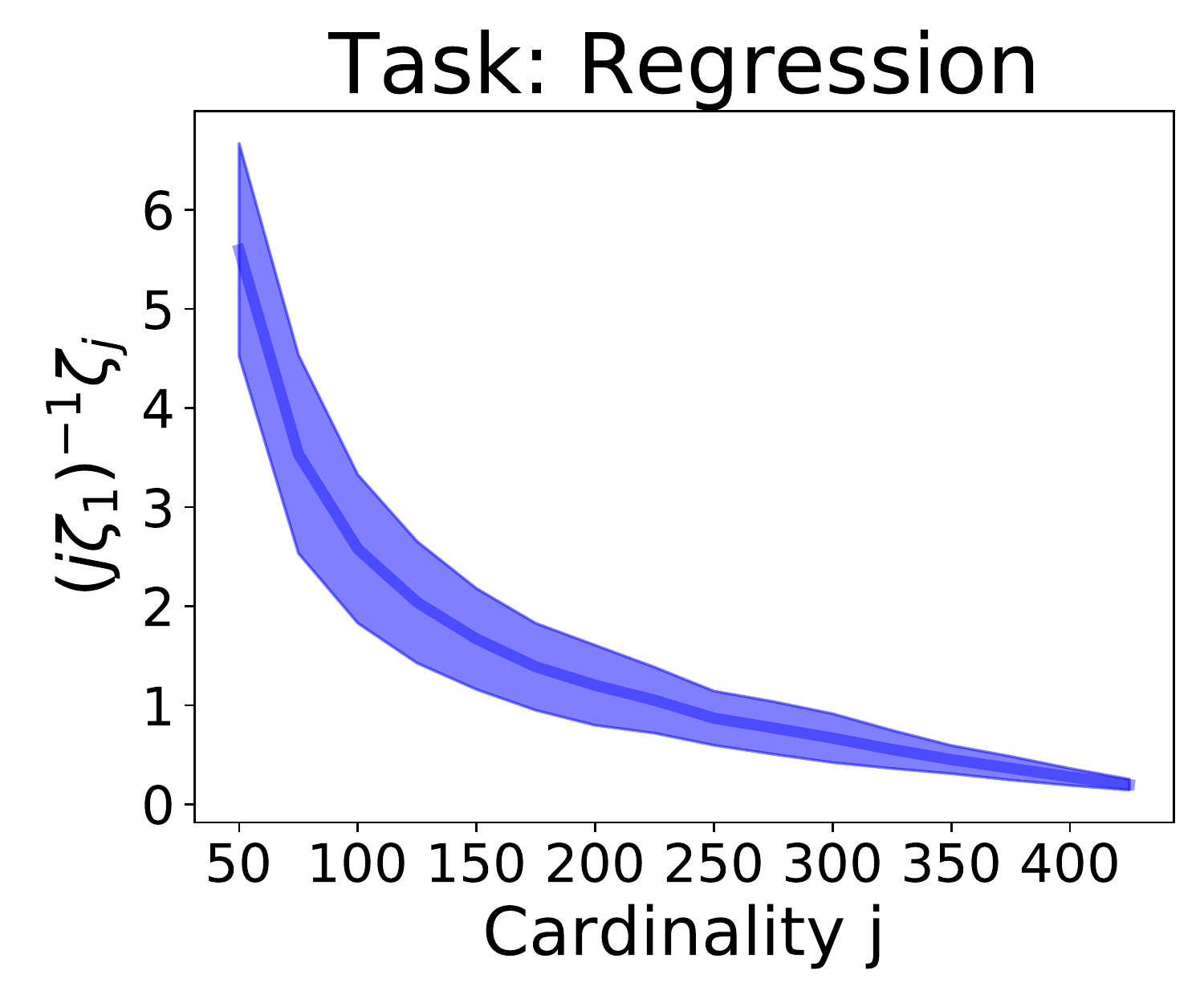}
    \includegraphics[width=0.35\columnwidth]{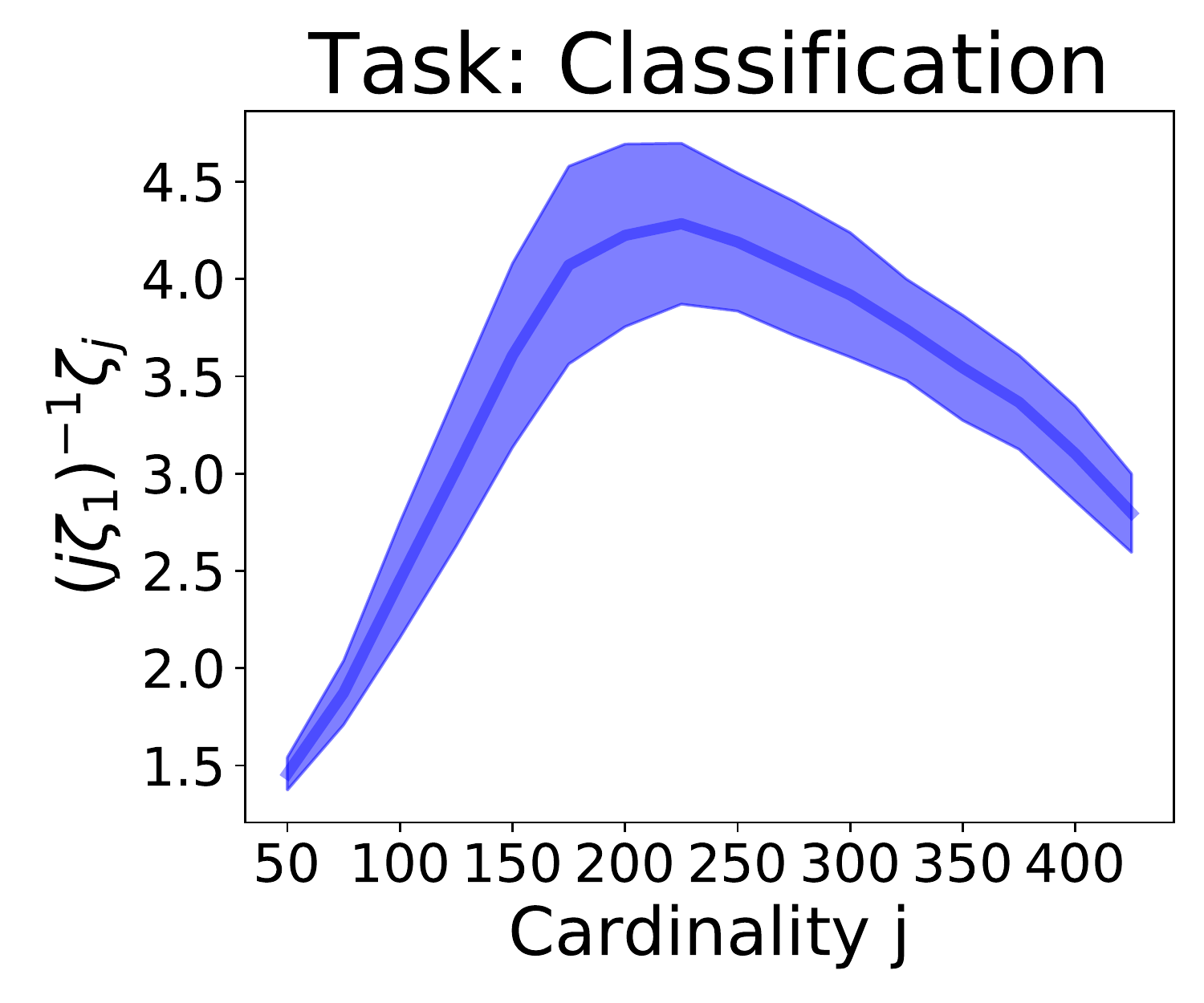}
    \caption{Illustration of $\zeta_{j}/(j\zeta_{1})$ as a function of cardinality $j$ in (left) linear regression and (right) logistic regression settings. The curve and band indicate mean and 95\% standard error of $\zeta_{j}/(j\zeta_{1})$ among different samples. Note that the quantity $\zeta_{j}/(j\zeta_{1})$ decreases as the cardinality increases and it empirically shows the validity of the condition in Theorem \ref{thm:u_stat_for_large_cardinality}.}
    \label{fig:condition_for_thm}
\end{figure}

\subsection{Details on Theorem \ref{thm:prob_shap_asymptotic}.}
\label{app:s:details_theorem}
For fixed positive constants $c_1$ and $c_2$, let $\Pi (c_1, c_2)$ be a set of measure $Q$ such that (i) the total measure of $Q$ is $c_1$ and (ii) it mutually absolutely continuous with $P_Z$ and $c_2 \leq dQ/dP_Z \leq 1$. For $Q \in \Pi (c_1, c_2)$, let $\mathbb{P}_n ^{Q}:= \frac{1}{n} \sum_{i=1} ^n A_i \frac{dP_Z}{dQ}(Z_i) \delta_{Z_i}$ be the Horvitz–Thompson empirical measure, where $A_i$ be a Bernoulli random variable with a probability $dQ/dP_Z (Z_i)$, and $\delta_z$ be the Dirac delta measure on $\mathcal{Z}$ \citep{sarndal2003model}.

\citet{ting2018optimal} showed $\sqrt{\frac{n}{c_1}}(h(\mathbb{P}_n ^{Q})-h(P_Z))$ converges in distribution to a Gaussian distribution with zero mean and $\nu^{Q}$ variance, where $\nu^{Q}:= \int \mathcal{I}(z) \mathcal{I}(z)^T (dP_Z/dQ (z)) dP_Z (z)$ and $\mathcal{I}$ is the Hadamard derivative of $h$ at $P_Z$, which is the influence function when it exists. 
Since the variance $\nu^{Q}$ is the function of $Q\in \Pi(c_1,c_2)$, the problem of finding the optimal subsampling weights can be formulated as finding $Q$ that minimizes the trace of variance, \textit{i.e.}, $\mathrm{argmin}_{Q \in \Pi(c_1,c_2)} \mathrm{Tr}(\nu^{Q})$. In the following theorem, we show that the Beta Shapley-based Horvitz-Thompson empirical measure produces the optimal estimator with the smallest variance.

\begin{theorem}[Formal version of Theorem \ref{thm:prob_shap_asymptotic}]
Suppose $h$ is Hadamard differentiable at $P_Z$ and the importance weight $\lambda_i$ for $i$-th sample $z_i$ is $\lambda_i \propto \norm{\psi_{\mathrm{semi}}(z_i; h, \mathcal{D}, w_{\alpha,\beta} ^{(n)})}_2$ and $n^{-1} \sum_{i=1} ^n \lambda_i= c_1$ for some $\beta \geq 1$. If there is $\mathbb{P}_n ^{Q_{\psi}} \in \Pi(c_1,c_2)$ such that $dQ_{\psi}/dP_Z (Z_i) =\lambda_i$, then the asymptotic variance of $h(\mathbb{P}_n ^{Q_{\psi}})$ is $\mathrm{min}_{Q \in \Pi(c_1,c_2)} \mathrm{Tr}(\nu^{Q})$.
\label{thm:prob_shap_asymptotic_full}
\end{theorem}

Theorem \ref{thm:prob_shap_asymptotic_full} shows asymptotic convergence of $\mathrm{Var}(h(\mathbb{P}_n ^{Q_{\psi}}))$, unfortunately, its convergence rate is unknown. We believe different choices of $(\alpha, \beta)$ can affect the convergence rate, and thus it can be used to choose the optimal hyperparameter.

\subsection{Derivation of Equation \ref{eqn:general_weights}}
\begin{proof}[Proof of Equation \ref{eqn:general_weights}]
For $j \in [n]$ and any probability density function $\xi$ defined on $[0,1]$, we set 
\begin{align*}
    w^{(n)}(j, \xi) := n\int_{0} ^{1} t^{j-1} (1-t)^{n-j} \xi(t) dt.
\end{align*}
Then, we have
\begin{align*}
    \frac{1}{n} \sum_{j=1} ^{n } w^{(n)}(j, \xi) \binom{n-1}{j-1} &= \int_{0} ^{1} \sum_{j=1} ^{n}  \binom{n-1}{j-1} t^{j-1} (1-t)^{n-j} \xi(t) dt \\
    &= \int_{0} ^{1} \xi(t) dt = 1.
\end{align*}
The last equality is due to the definition of $\xi$. This concludes a proof.
\end{proof}

\section{Proofs}
\label{app:proofs}
\begin{proof}[Proof of Theorem \ref{thm:u_stat_for_large_cardinality}]
The first result, $(j^2 \zeta_1 / n)^{-1} \mathrm{Var}(\Delta_j (z^*, U, \mathfrak{D})) \to 1$ as $n$ increases, is from the central limit theorem of U-statistics when $j$ increases. Our result is from Theorem 3.1(i) of \citet{diciccio2020clt}. 
\end{proof}

\begin{proof}[Proof of Theorem \ref{thm:semivalue_representation}]
If there exists a weight function $w^{(n)}: [n] \to \mathbb{R}$ such that $\sum_{j=1} ^{n} \binom{n-1}{j-1} w^{(n)}(j)=n$ and the value function $\psi_{\mathrm{semi}}$ can be expressed as 
\begin{align*}
    \psi_{\mathrm{semi}}(z^*; U, \mathcal{D}, w^{(n)}) = \frac{1}{n} \sum_{j=1} ^{n} \binom{n-1}{j-1} w^{(n)}(j)  \Delta_j (z^*; U, \mathcal{D}).   
\end{align*}
Then by the form of the function $\psi_{\mathrm{semi}}(z^*; U, \mathcal{D}, w^{(n)})$, it satisfies the linearity, null player and symmetry axioms.  

Now, we show the inverse. By Theorem 3 of \citet{dubey1977probabilistic}, when a function satisfies the linearity and null player, there exists a weight function $\mu: \cup_{j=0} ^{\infty} \mathcal{Z}^j \to \mathbb{R}$ such that $\sum_{S \subseteq \mathcal{D}\backslash\{z^*\}} \mu^{(n)}(S)=1$ and the value function $\psi$ can be expressed as 
\begin{align*}
    \psi(z^*; U, \mathcal{D}, \mu^{(n)}) = \sum_{S \subseteq \mathcal{D}\backslash\{z^*\}} \mu^{(n)}(S)  (U(S\cup\{z^*\})-U(S)).   
\end{align*}
In addition, due to the symmetry axiom, if there are two subsets $S_1$ and $S_2$ such that $|S_1|=|S_2|$ and $S_1, S_2 \subseteq \mathcal{D}\backslash\{z^*\}$, then $\mu(S_1)=\mu(S_2)$.\footnote{For a function $v_S(A):= \mathds{1}(S\subsetneq A)$ and a permutation $\pi^*$ that sends $\pi^*(S_1)=S_2$ and fixes others, $\mu(S_1)=\psi(z^*; v_{S_1}, \mathcal{D}, \mu^{(n)})=\psi(z^*; v_{S_2}, \mathcal{D}, \mu^{(n)})=\mu(S_2)$.} Therefore, for $\nu:[n] \to \mathbb{R}$, we can further simplify $\psi$ as follows.
\begin{align*}
    \psi(z^*; U, \mathcal{D}, \mu^{(n)}) &= \sum_{j=1} ^{n}  \sum_{S \subseteq \mathcal{D}_j ^{\backslash\{z^*\}}} \mu^{(n)}(S) (U(S\cup\{z^*\})-U(S)) \\
    &= \sum_{j=1} ^{n} \nu(j)  \sum_{S \subseteq \mathcal{D}_j ^{\backslash\{z^*\}}} (U(S\cup\{z^*\})-U(S)),
\end{align*}
where $\sum_{j=1} ^{n} \nu(j) \binom{n-1}{j-1}=1$. Thus, considering $w^{(n)}(j) := n^{-1} \nu (j)$, it concludes a proof.
\end{proof}

\begin{proof}[Proof of Proposition \ref{prop:uniqueness_semivalue}]
The uniqueness of semivalues shown in Proposition \ref{prop:uniqueness_semivalue} is directly from Theorem 10 of \citet{MONDERER20022055}.
\end{proof}

\begin{proof}[Proof of Theorem \ref{thm:prob_shap_asymptotic}]
We provide a proof for Theorem \ref{thm:prob_shap_asymptotic_full}, which is a detailed version of Theorem \ref{thm:prob_shap_asymptotic}. Theorem \ref{thm:prob_shap_asymptotic_full} is specified in Appendix \ref{app:s:details_theorem}.
A key component of this proof is to show for some sequence of explicit constants $(\gamma_{\alpha, \beta} ^{(j)})$, $(\gamma_{\alpha, \beta} ^{(n)}/n)^{-1} \psi_{\mathrm{semi}}(z^*; U, \mathcal{D}, w_{\alpha,\beta} ^{(n)})$ converges to the influence function as $n$ increases by using the Silverman-Toeplitz theorem. 
    
To be more specific, 
\begin{align*}
    \psi_{\mathrm{semi}}(z^*; h, \mathcal{D}, w_{\alpha,\beta} ^{(n)}) &= \frac{1}{n} \sum_{j=1} ^{n} w_{\alpha,\beta} ^{(n)}(j) \sum_{ S \subseteq \mathcal{D}_{j} ^{\backslash z^*} } h(S\cup \{z^*\})-h(S) \\
    &= \frac{1}{n} \sum_{j=1} ^{n} w_{\alpha,\beta} ^{(n)}(j) \binom{n-1}{j-1} \frac{1}{j} \frac{1}{\binom{n-1}{j-1}} \sum_{ S \subseteq \mathcal{D}_{j} ^{\backslash z^*} } \frac{h(S\cup \{z^*\})-h(S)}{1/j} \\
    &= \frac{\gamma_{\alpha, \beta} ^{(n)} }{n} \sum_{j=1} ^{n} \frac{w_{\alpha,\beta} ^{(n)}(j) \binom{n-1}{j-1} \frac{1}{j}}{\gamma_{\alpha, \beta} ^{(n)}} \psi_j(z^*, h),
\end{align*}
where
\begin{align*}
    \gamma_{\alpha, \beta} ^{(n)} &:= \sum_{j=1} ^{n} w_{\alpha,\beta} ^{(n)}(j) \binom{n-1}{j-1} \frac{1}{j}\\
    \psi_j (z^*, h) &:= \frac{1}{\binom{n-1}{j-1}} \sum_{ S \subseteq \mathcal{D}_{j} ^{\backslash z^*} } \frac{h(S\cup \{z^*\})-h(S)}{1/j}.
\end{align*}
Note that $\psi_j (z^*, h) \to \mathcal{I}(z^*; h, P_Z)$ as $j \to \infty$ because $\mathcal{I}$ is a Hadamard derivative of $h$ at $P_Z$. We formally denote this by $\mathcal{I}(z^*; h, P_Z)$. That is, if 
\begin{align*}
    \lim_{n \to \infty} (\gamma_{\alpha, \beta} ^{(n)})^{-1} \left( w_{\alpha,\beta} ^{(n)}(j) \binom{n-1}{j-1} \frac{1}{j} \right) = 0    
\end{align*}
for all fixed $j \in[n]$, then by the Silverman-Toeplitz theorem, we have $(\gamma_{\alpha, \beta} ^{(n)}/n)^{-1} \psi_{\mathrm{semi}}(z^*; h, \mathcal{D}, w_{\alpha,\beta} ^{(n)}) \to \mathcal{I}(z^*; h, P_Z)$.

Note that
\begin{align*}
    w_{\alpha,\beta} ^{(n)}(j) \binom{n-1}{j-1} \frac{1}{j} &= \frac{\mathrm{Beta}(j+\beta-1, n-j+\alpha)}{\mathrm{Beta}(\alpha, \beta)} \binom{n}{j} \\
    &= \frac{1}{\mathrm{Beta}(\alpha, \beta)} \binom{n}{j} \int_{0} ^{1} t^{j+\beta-2} (1-t)^{n-j+\alpha-1} dt,
\end{align*}
and thus
\begin{align*}
    \gamma_{\alpha, \beta} ^{(n)} &= \sum_{j=1} ^{n} w_{\alpha,\beta} ^{(n)}(j) \binom{n-1}{j-1} \frac{1}{j} \\
    &= \frac{1}{\mathrm{Beta}(\alpha, \beta)} \int_{0} ^{1} \sum_{j=1} ^{n} \binom{n}{j} t^{j+\beta-2} (1-t)^{n-j+\alpha-1} dt \\
    &= \frac{1}{\mathrm{Beta}(\alpha, \beta)} \int_{0} ^{1} t^{\beta-2} (1-t)^{\alpha-1} \sum_{j=1} ^{n} \binom{n}{j} t^{j} (1-t)^{n-j} dt\\
    &= \frac{1}{\mathrm{Beta}(\alpha, \beta)} \int_{0} ^{1} t^{\beta-2} (1-t)^{\alpha-1} \left(1- (1-t)^{n} \right) dt \\
    &= \frac{1}{\mathrm{Beta}(\alpha, \beta)} \int_{0} ^{1} (1-t)^{\beta-2} t^{\alpha-1} \left(1- t^{n} \right) dt \\
    &= \frac{1}{\mathrm{Beta}(\alpha, \beta)} \int_{0} ^{1} (1-t)^{\beta-1} t^{\alpha-1} \frac{1- t^{n}}{1-t} dt \\
    &= \frac{1}{\mathrm{Beta}(\alpha, \beta)} \int_{0} ^{1} (1-t)^{\beta-1} \sum_{k=0} ^{n-1} t^{\alpha-1+k}  dt \\
    &= \frac{1}{\mathrm{Beta}(\alpha, \beta)} \sum_{k=0} ^{n-1} \mathrm{Beta}(\alpha+k, \beta).
\end{align*}

[Step 1] When $\beta > 1$, due to $\mathrm{Beta}(\alpha+k, \beta-1)+\mathrm{Beta}(\alpha+k-1, \beta)=\mathrm{Beta}(\alpha+k-1, \beta-1)$ for any $k \in [n]$, we have
\begin{align*}
    \gamma_{\alpha, \beta} ^{(n)} &= \frac{1}{\mathrm{Beta}(\alpha, \beta)} \left( \mathrm{Beta}(\alpha, \beta-1) - \mathrm{Beta}(\alpha+n, \beta-1) \right) \\
    &=\left( \frac{\mathrm{Beta}(\alpha, \beta-1)}{\mathrm{Beta}(\alpha, \beta)} - \frac{\Gamma(\beta-1)}{\mathrm{Beta}(\alpha, \beta)} \frac{\Gamma(n+\alpha)}{\Gamma(n+\beta+\alpha-1)} \right)\\
    &\approx \left( \frac{\mathrm{Beta}( \alpha, \beta-1)}{\mathrm{Beta}(\alpha, \beta)} - \frac{\Gamma(\beta-1)}{\mathrm{Beta}(\alpha, \beta)} n^{1-\beta} \right)=O(1).
\end{align*}
The last approximation is due to Equation (1) of \citet{tricomi1951asymptotic} when $n$ is large enough.

[Step 2] When $\beta \leq 1$, we have
\begin{align*}
    \gamma_{\alpha, \beta} ^{(n)} &= \sum_{k=0} ^{n-1} \frac{\mathrm{Beta}(\alpha+k, \beta)}{\mathrm{Beta}(\alpha, \beta)} \\
    &= \sum_{k=0} ^{n-1} \frac{\Gamma(\alpha+k)\Gamma( \beta)}{\Gamma(\alpha+\beta+k)}\frac{\Gamma(\alpha+\beta)}{\Gamma(\alpha)\Gamma(\beta)}\\
    &= \frac{\Gamma(\alpha+\beta)}{\Gamma(\alpha)} \sum_{k=0} ^{n-1} \frac{\Gamma(\alpha+k)}{\Gamma(\alpha+\beta+k)} \\ 
    &= \frac{\Gamma(\alpha+\beta)}{\Gamma(\alpha)} \sum_{k=0} ^{n-1} \frac{\Gamma(1+\alpha+k)}{\Gamma(\alpha+\beta+k)} \frac{1}{\alpha+k}\\
    &\approx \frac{\Gamma(\alpha+\beta)}{\Gamma(\alpha)} \sum_{k=0} ^{n-1} (\alpha+k)^{1-\beta} \frac{1}{\alpha+k}.
\end{align*}
The approximation is from Equation (7) of \citet{gautschi1959some} when $0<\beta \leq1$. Therefore, we have
\begin{align*}
    \gamma_{\alpha, \beta} ^{(n)} &\approx \begin{cases}
            \frac{\Gamma(\alpha+\beta)}{\Gamma(\alpha)} \log \frac{n+\alpha}{\alpha}=O(\log n), & \text{if } \beta=1\\
            \frac{\Gamma(\alpha+\beta)}{\Gamma(\alpha)}\frac{(n+\alpha)^{1-\beta} - \alpha^{1-\beta}}{1-\beta}=O(n^{1-\beta}),  & \beta<1
    \end{cases}.
\end{align*}

[Step 3] For a fixed $s$, we have
\begin{align*}
    w_{\alpha,\beta} ^{(n)}(j) \binom{n-1}{j-1} \frac{1}{j} &= n \frac{\mathrm{Beta}(j+\beta-1, n-j+\alpha)}{\mathrm{Beta}(\alpha, \beta)} \frac{(n-1)!}{(j-1)! (n-j)!} \frac{1}{j} \\
    &= n\frac{\Gamma(j+\beta-1)\Gamma(n-j+\alpha)}{\Gamma(n+\alpha+\beta-1)} \frac{\Gamma(n)}{j! \Gamma(n-j+1)} \frac{1}{\mathrm{Beta}(\alpha, \beta)}\\
    &= n\frac{\Gamma(n-j+\alpha) \Gamma(n)}{\Gamma(n+\alpha+\beta-1)\Gamma(n-j+1)} \frac{\Gamma(j+\beta-1)}{\Gamma(j+1)}\frac{1}{\mathrm{Beta}(\alpha, \beta)}\\
    &\approx n^{1-\beta} \frac{\Gamma(j+\beta-1)}{\Gamma(j+1)}\frac{1}{\mathrm{Beta}(\alpha, \beta)}=O(n^{1-\beta}).
\end{align*}

[Step 4] Now we consider $\lim_{n \to \infty} (\gamma_{\alpha, \beta} ^{(n)})^{-1} w_{\alpha,\beta} ^{(n)}(j) \binom{n-1}{j-1} \frac{1}{j}$. In case of $\beta >1$, from [Step 1] and [Step 3], we have $\gamma_{\alpha, \beta} ^{(n)} =O(1)$ and $w_{\alpha,\beta} ^{(n)}(j) \binom{n-1}{j-1} \frac{1}{j}=O(n^{1-\beta})$, thus $\lim_{n \to \infty} (\gamma_{\alpha, \beta} ^{(n)})^{-1} w_{\alpha,\beta} ^{(n)}(j) \binom{n-1}{j-1} \frac{1}{j}=0$. Similarly, in case of $\beta = 1$, from [Step 2] and [Step 3], we also have $\lim_{n \to \infty} (\gamma_{\alpha, \beta} ^{(n)})^{-1} w_{\alpha,\beta} ^{(n)}(j) \binom{n-1}{j-1} \frac{1}{j}=0$. Thus, by Silverman-Toeplitz theorem, if $\beta \geq 1$, then
\begin{align*}
    \frac{n}{\gamma_{\alpha, \beta} ^{(n)}} \psi_{\mathrm{semi}}(z^*; h, \mathcal{D}, w_{\alpha,\beta} ^{(n)})  = \mathcal{I}(z^*; h, P_Z)+o_p(1).
\end{align*} 
Moreover, by Theorem 2 of \citet{ting2018optimal}, learning with a Horvitz–Thompson empirical measure defined with the importance weight $\lambda_i$ such that $\lambda_i \propto \norm{ \frac{n}{\gamma_{\alpha, \beta} ^{(n)}}\psi_{\mathrm{semi}}(z^*; h, \mathcal{D}, w_{\alpha,\beta} ^{(n)})}_2 \propto \norm{ \psi_{\mathrm{semi}}(z^*; h, \mathcal{D}, w_{\alpha,\beta} ^{(n)})}_2$ gives an estimator with the minimum variance.
\end{proof}

\section{Additional Experiments}
\label{app:add_exp}

\subsection{Additional results using different datasets}
In Figure~\ref{fig:app_snr_real_datasets}, we shows the signal-to-noise ratio of $\Delta_j (z^*; U, \mathfrak{D})$ as a function of the cardinality $j$ for the four real datasets. As in other figures, we denote a 95\% confidence band based on 50 repetitions under the assumption that the results follow the identical Gaussian distribution. In Figure~\ref{fig:app_clean_noisy_marginal_contributions}, we illustrate the marginal contributions as a function of the cardinality using the eleven datasets. This figure shares the same setting with Figure~\ref{fig:clean_noisy_marginal_contributions} of the manuscript, but different datasets are used. In Figures~\ref{fig:app_point_addition_experiment} and~\ref{fig:app_point_removal_experiment}, we present additional point addition and removal experiment results using the thirteen datasets. Details on datasets and experiments settings are provided in Section~\ref{app:imple_details}.

\begin{figure*}[h]
    \centering
    \includegraphics[width=0.24\textwidth]{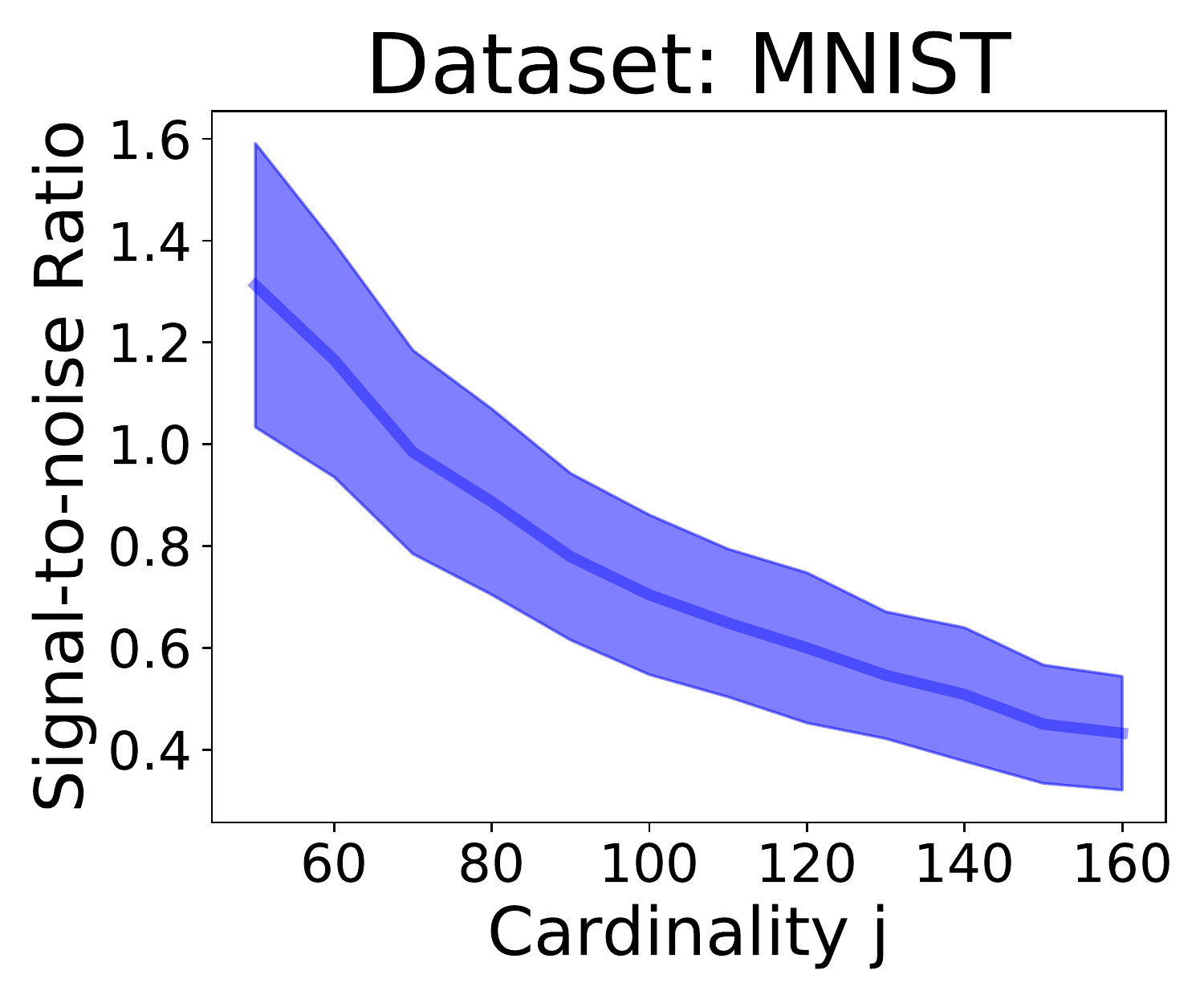}
    \includegraphics[width=0.24\textwidth]{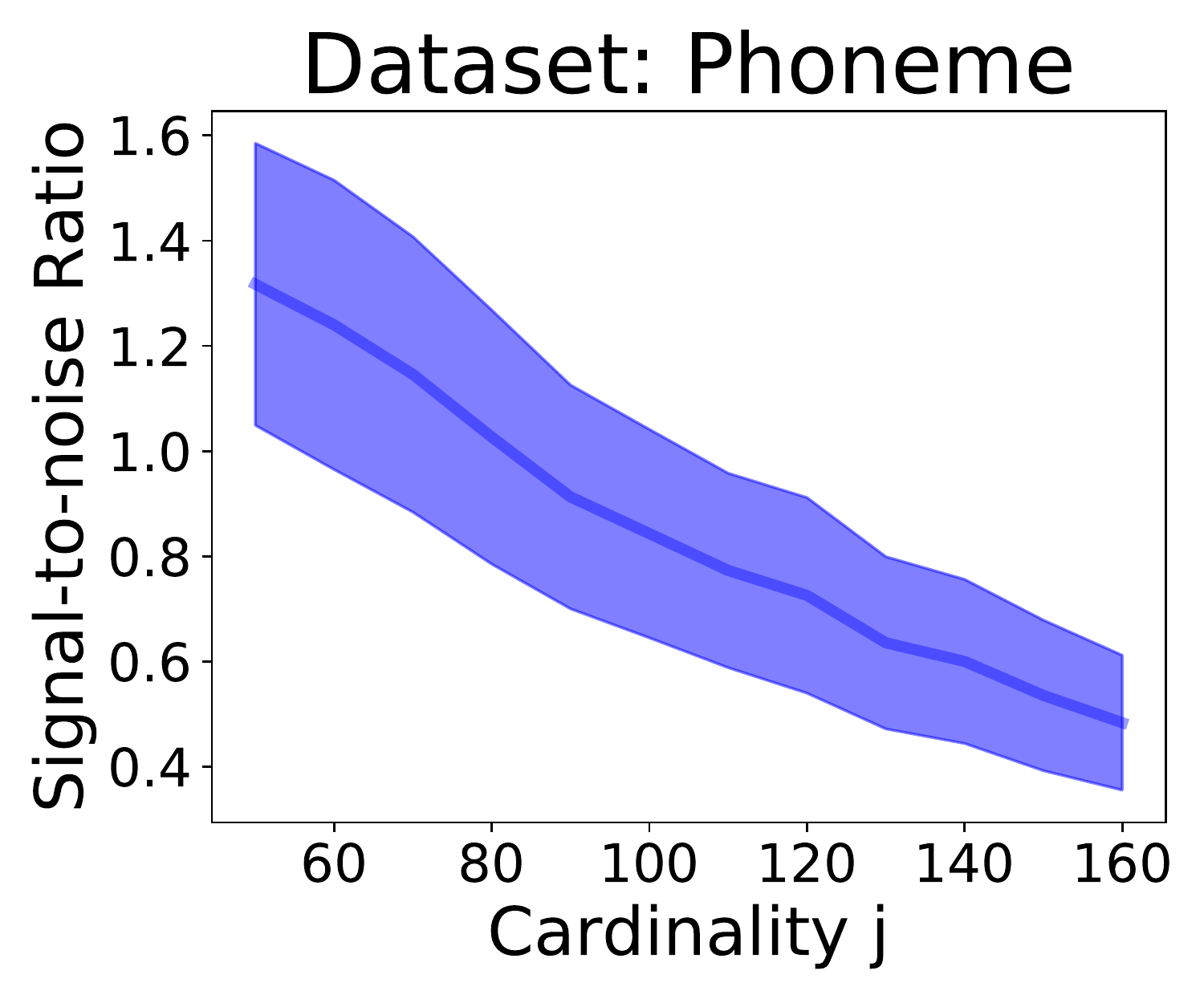}
    \includegraphics[width=0.24\textwidth]{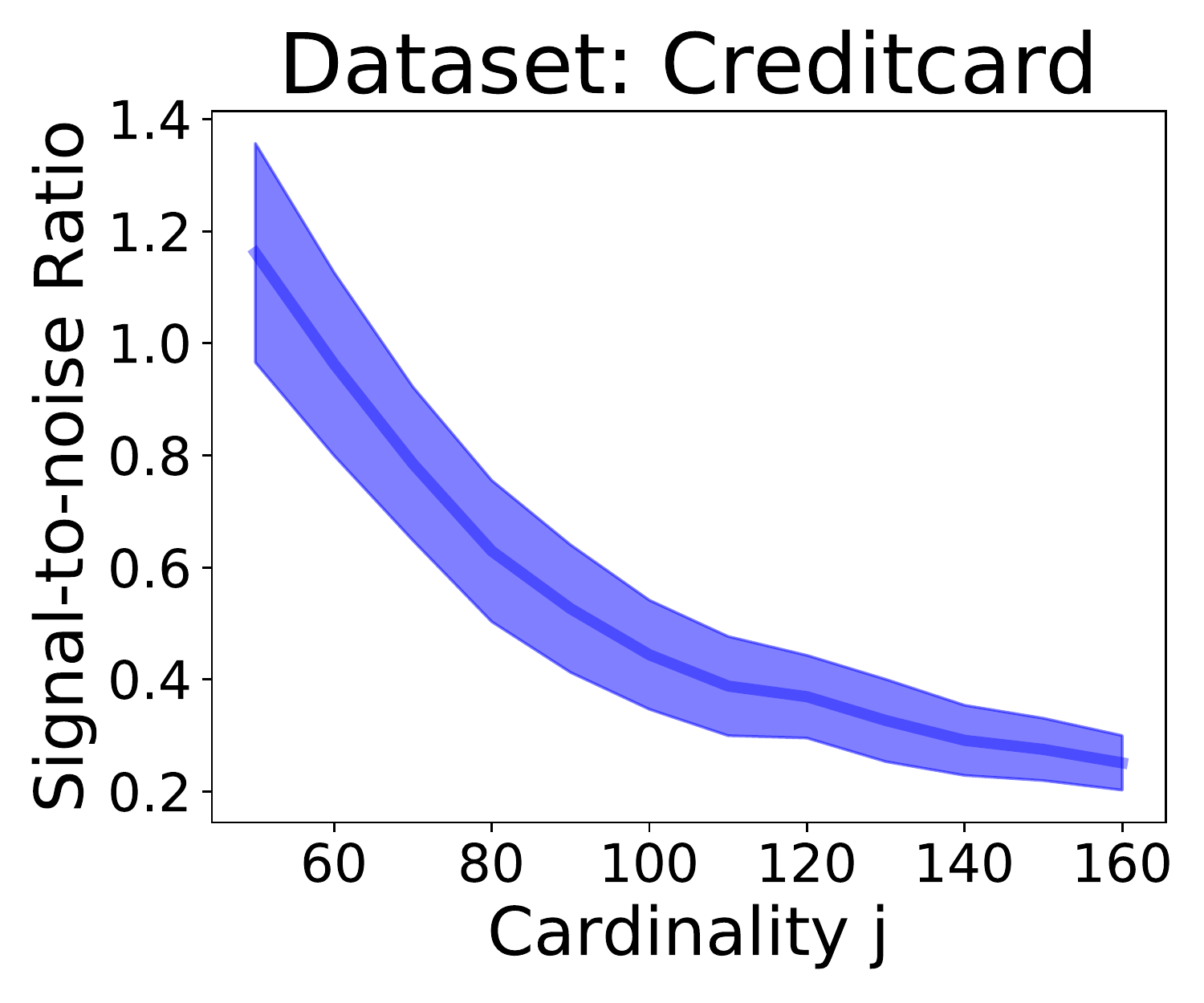}
    \includegraphics[width=0.24\textwidth]{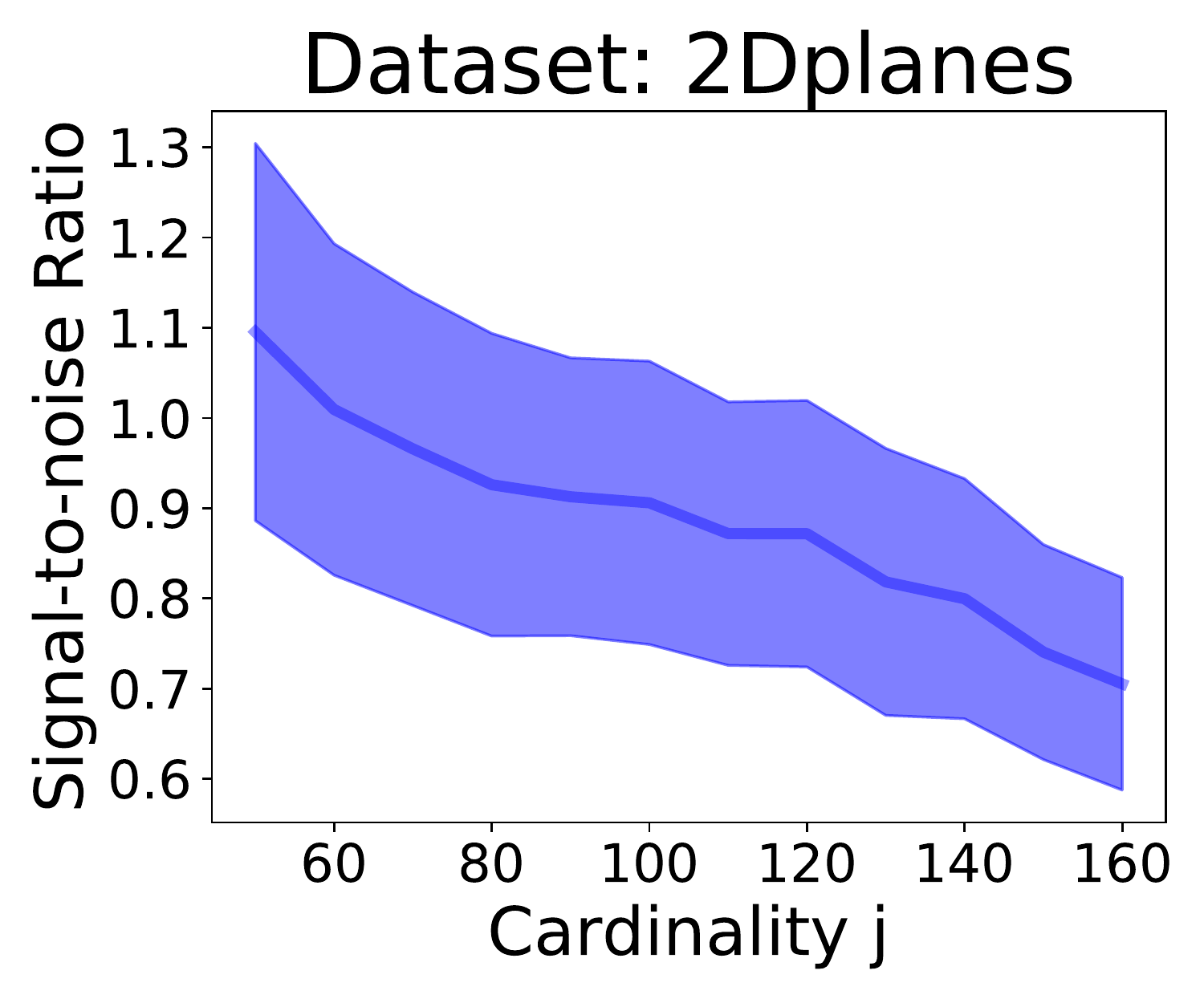}
    \caption{The signal-to-noise ratio of $\Delta_j (z^*; U, \mathfrak{D})$ as a function of the cardinality $j$ for four real datasets. Similar to Figure~\ref{fig:snr}, the signal-to-noise ratio decreases as the cardinality $j$ increases.}
    \label{fig:app_snr_real_datasets}
\end{figure*}

\begin{figure*}[h]
    \centering
    \includegraphics[width=0.245\textwidth]{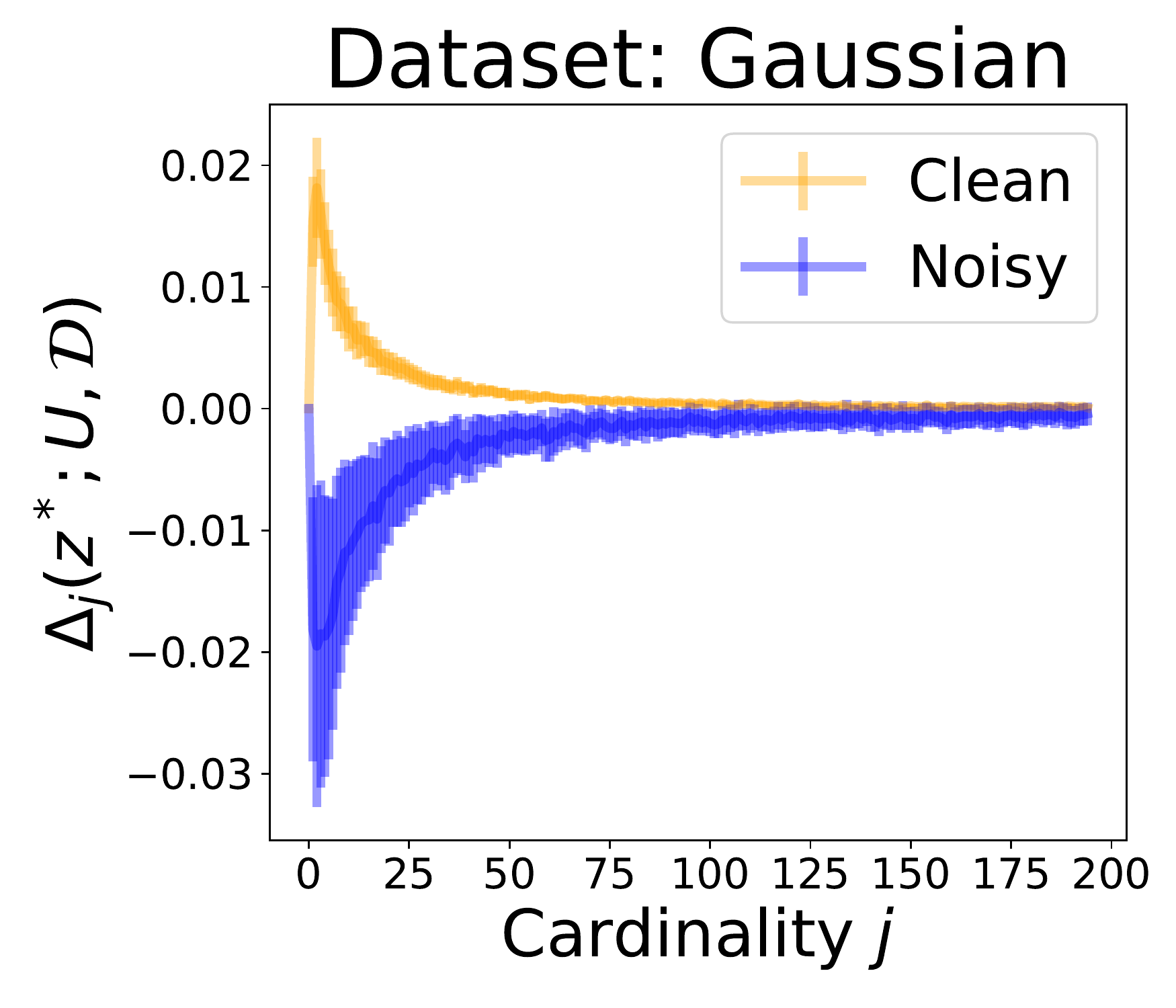}
    \includegraphics[width=0.245\textwidth]{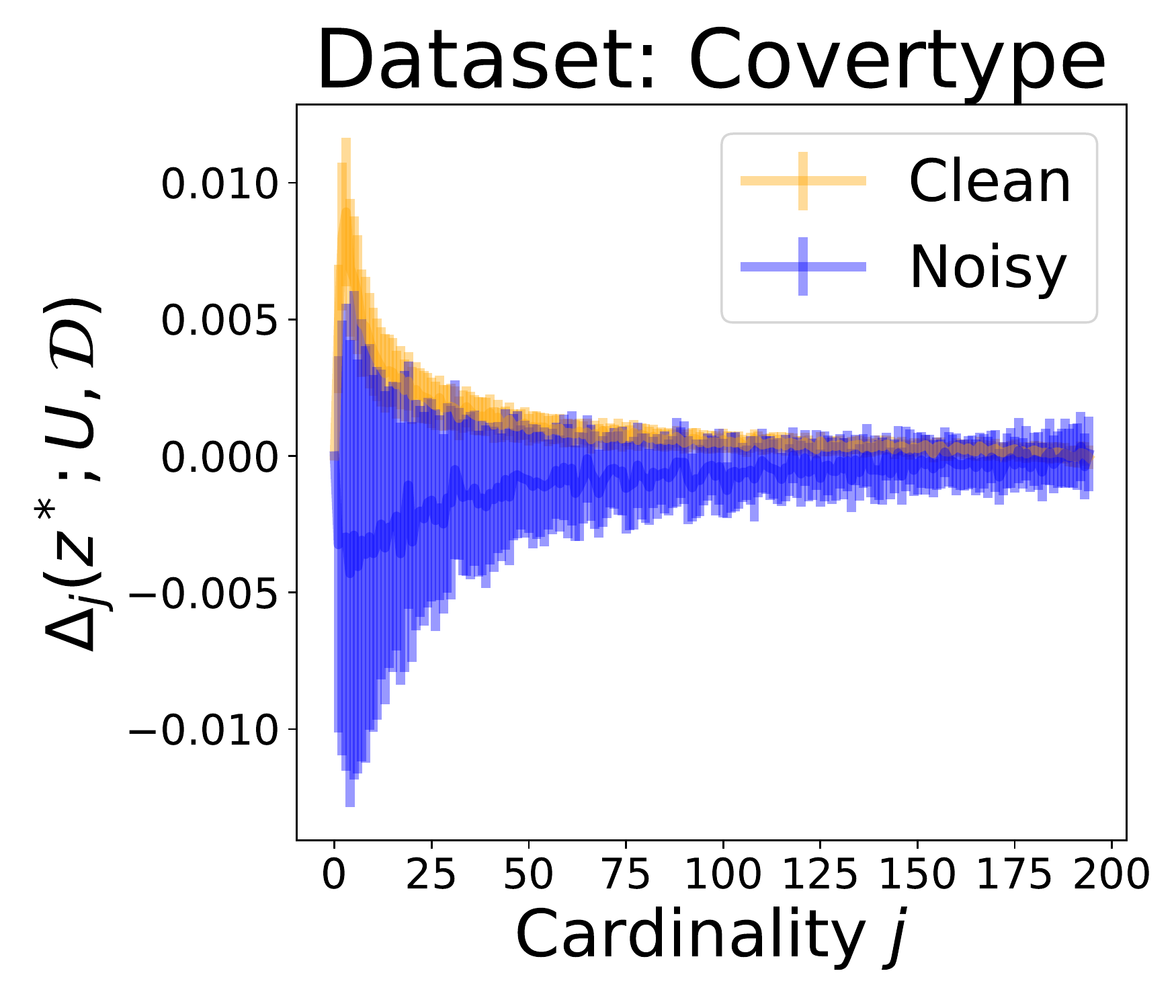}
    \includegraphics[width=0.245\textwidth]{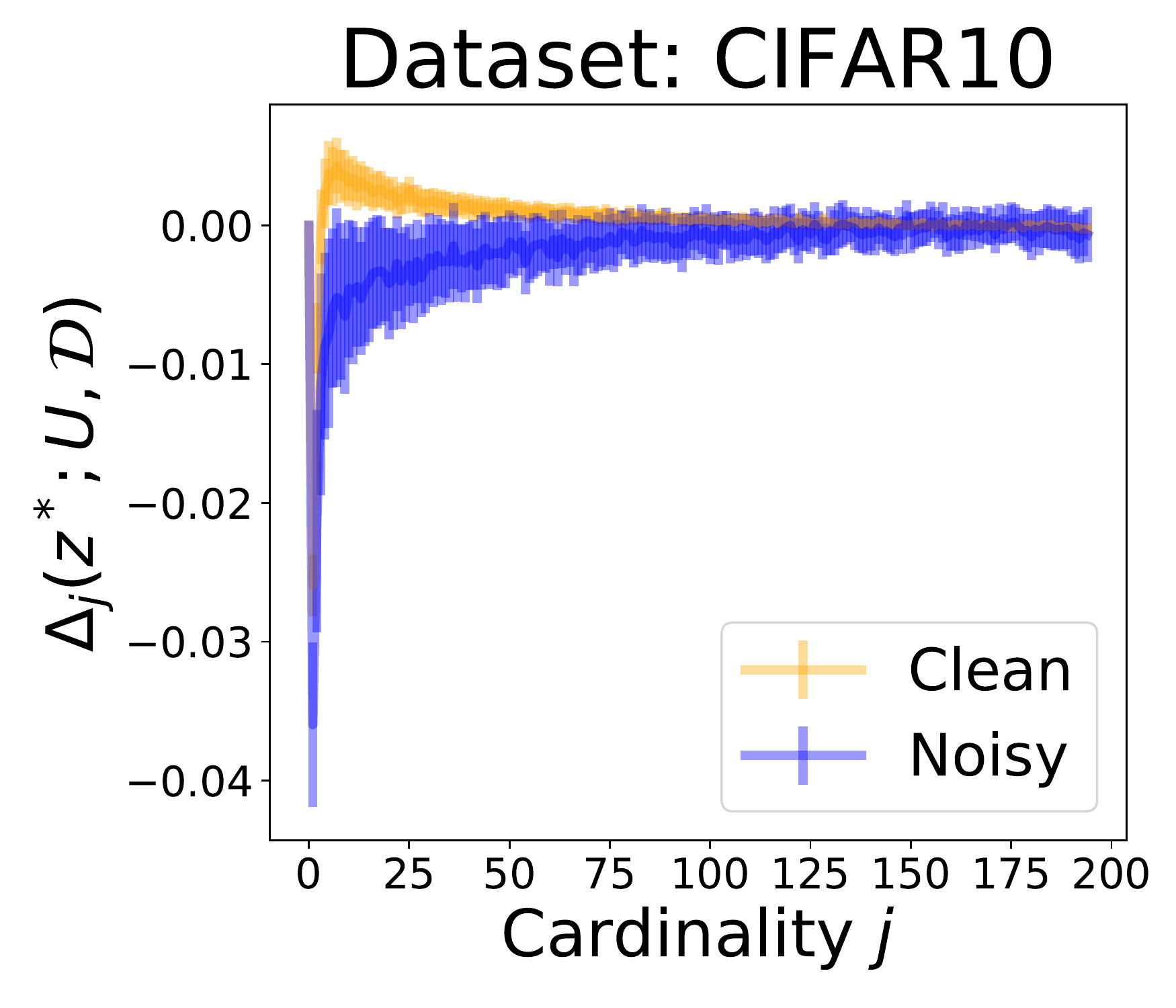}
    \includegraphics[width=0.245\textwidth]{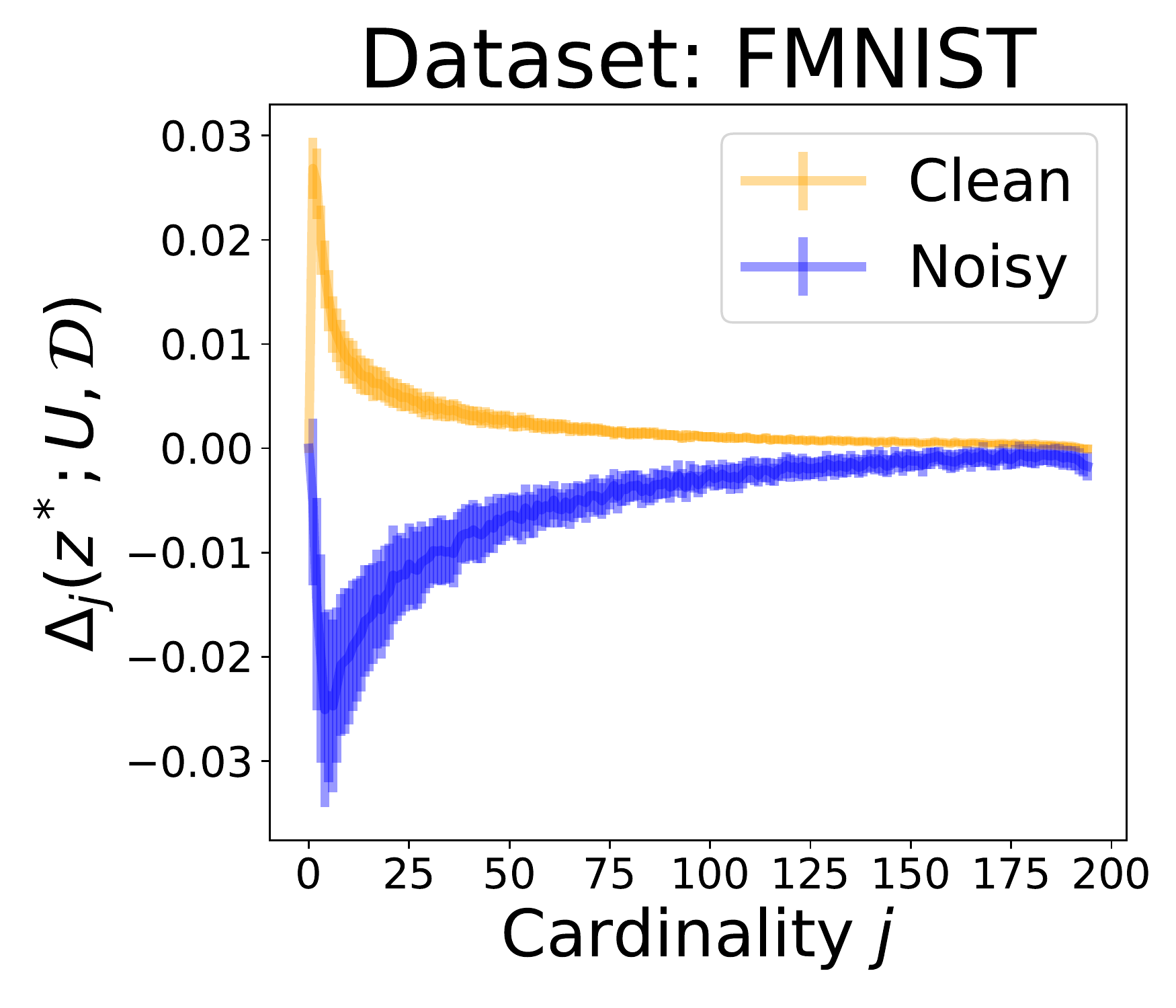}\\
    \includegraphics[width=0.245\textwidth]{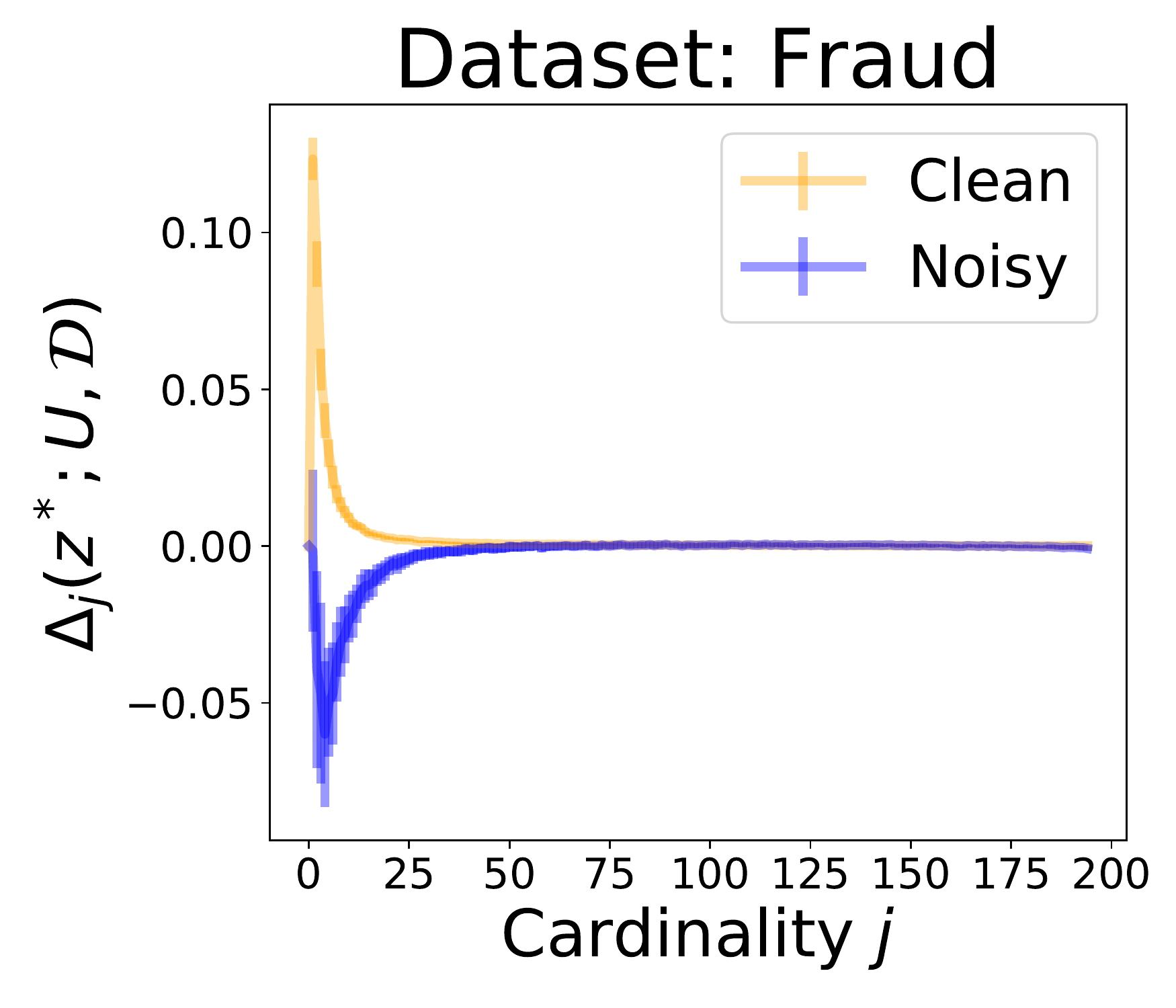}
    \includegraphics[width=0.245\textwidth]{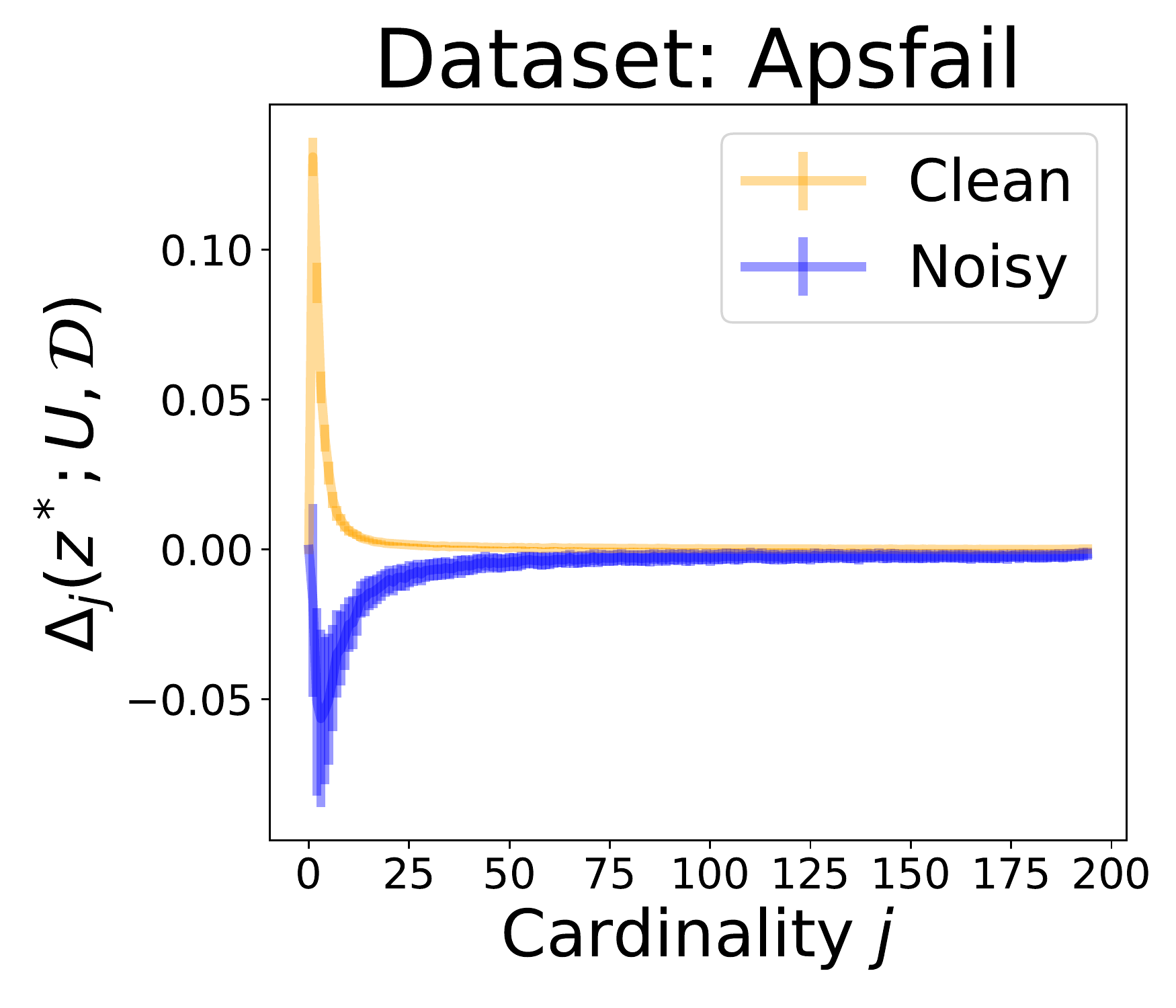}
    \includegraphics[width=0.245\textwidth]{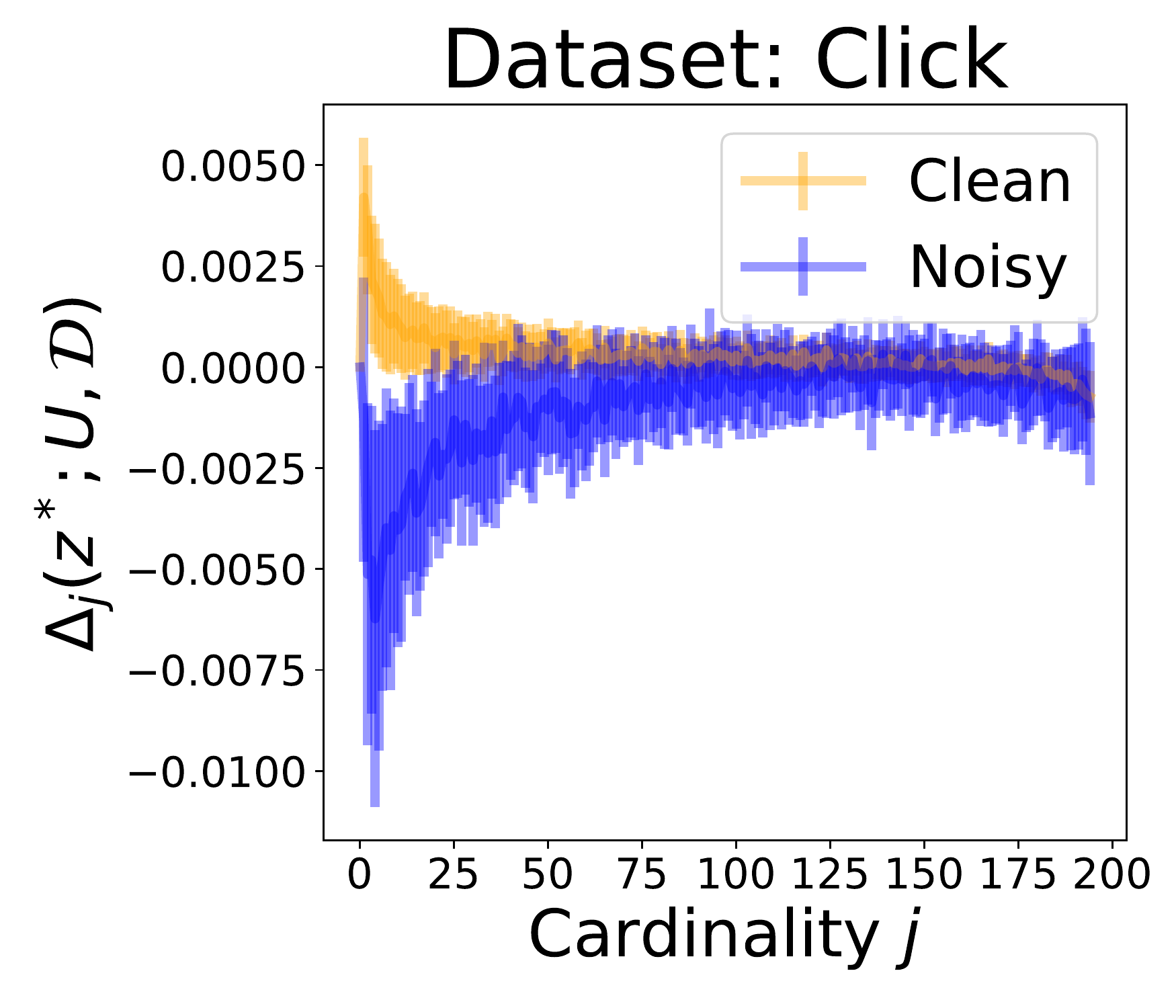}
    \includegraphics[width=0.245\textwidth]{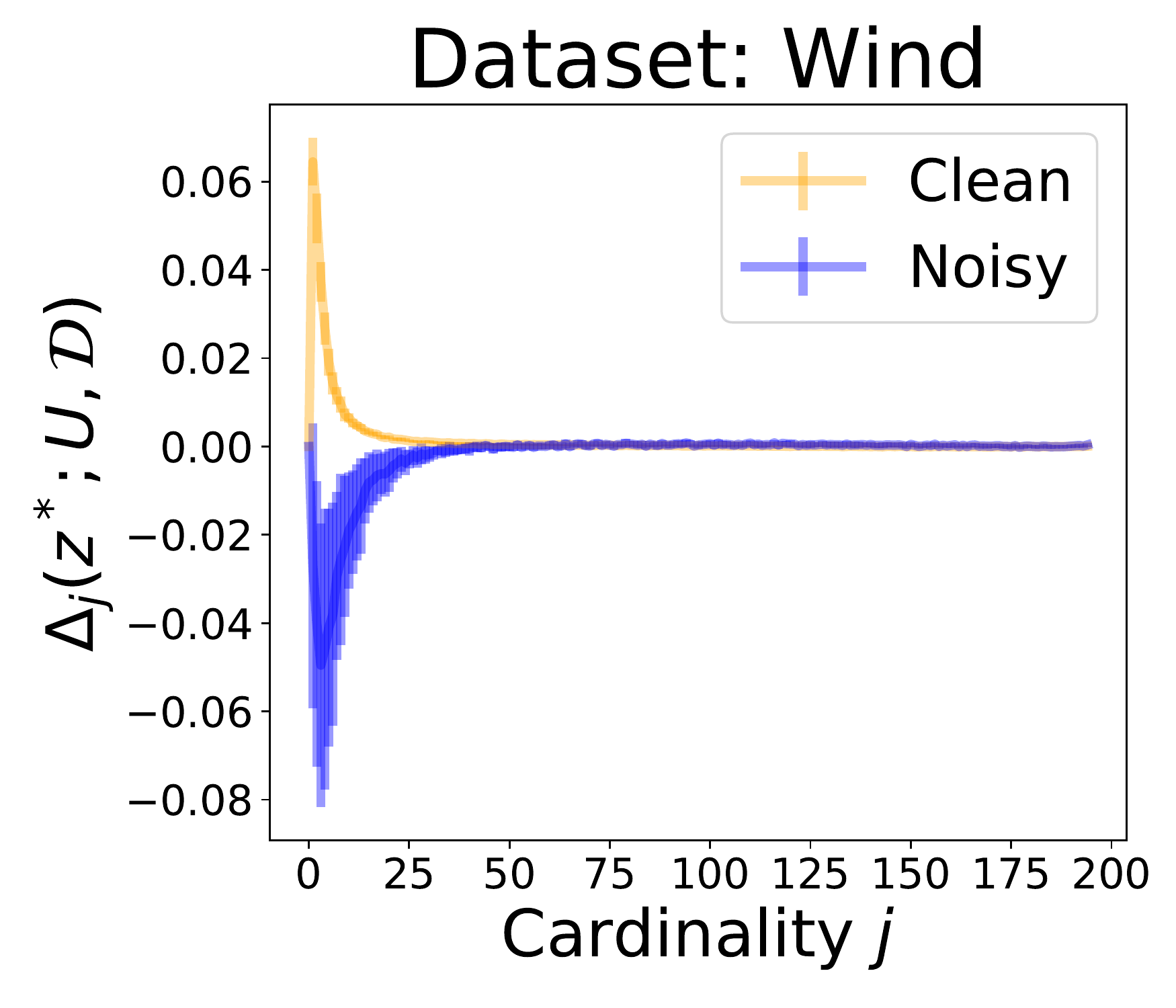}\\
    \includegraphics[width=0.245\textwidth]{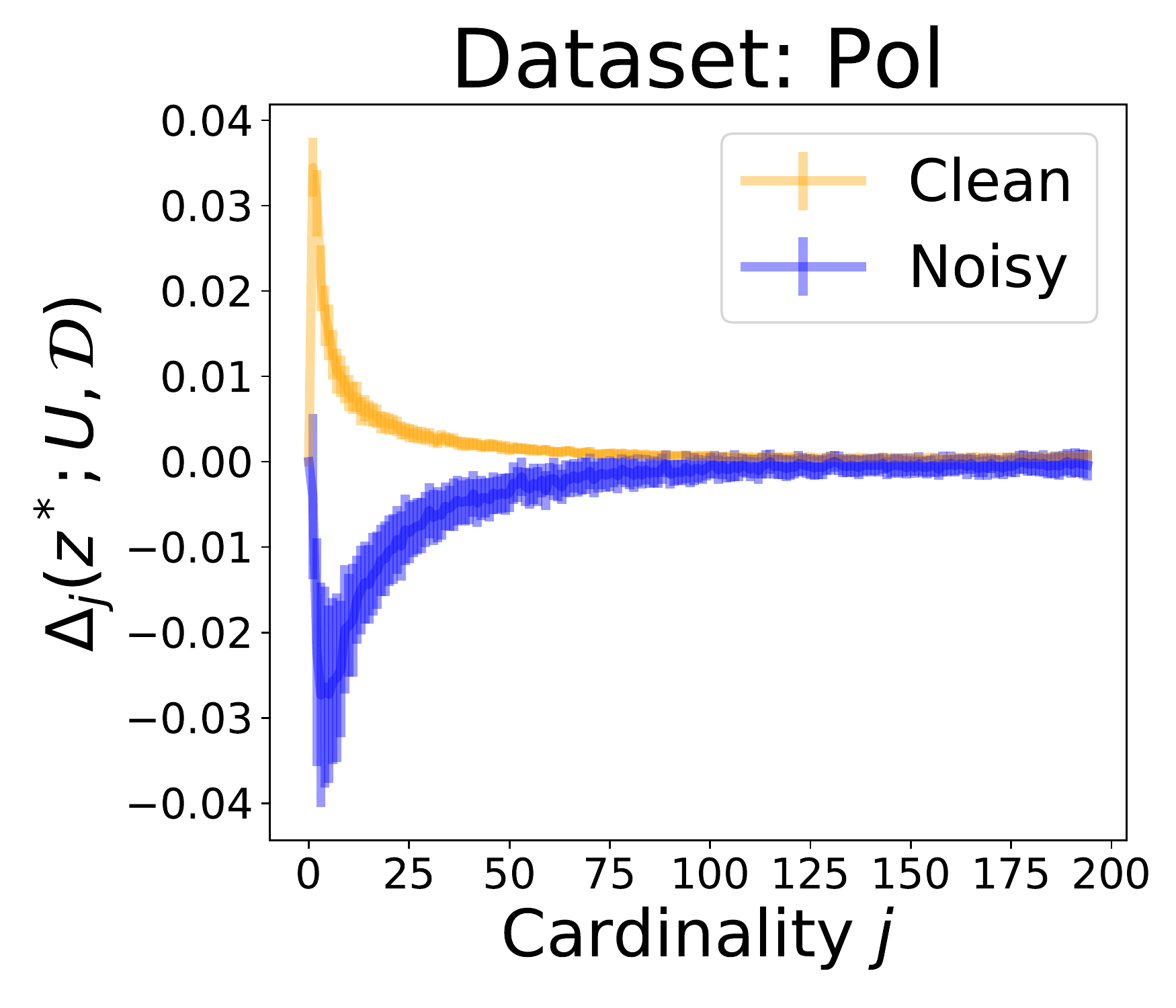}
    \includegraphics[width=0.245\textwidth]{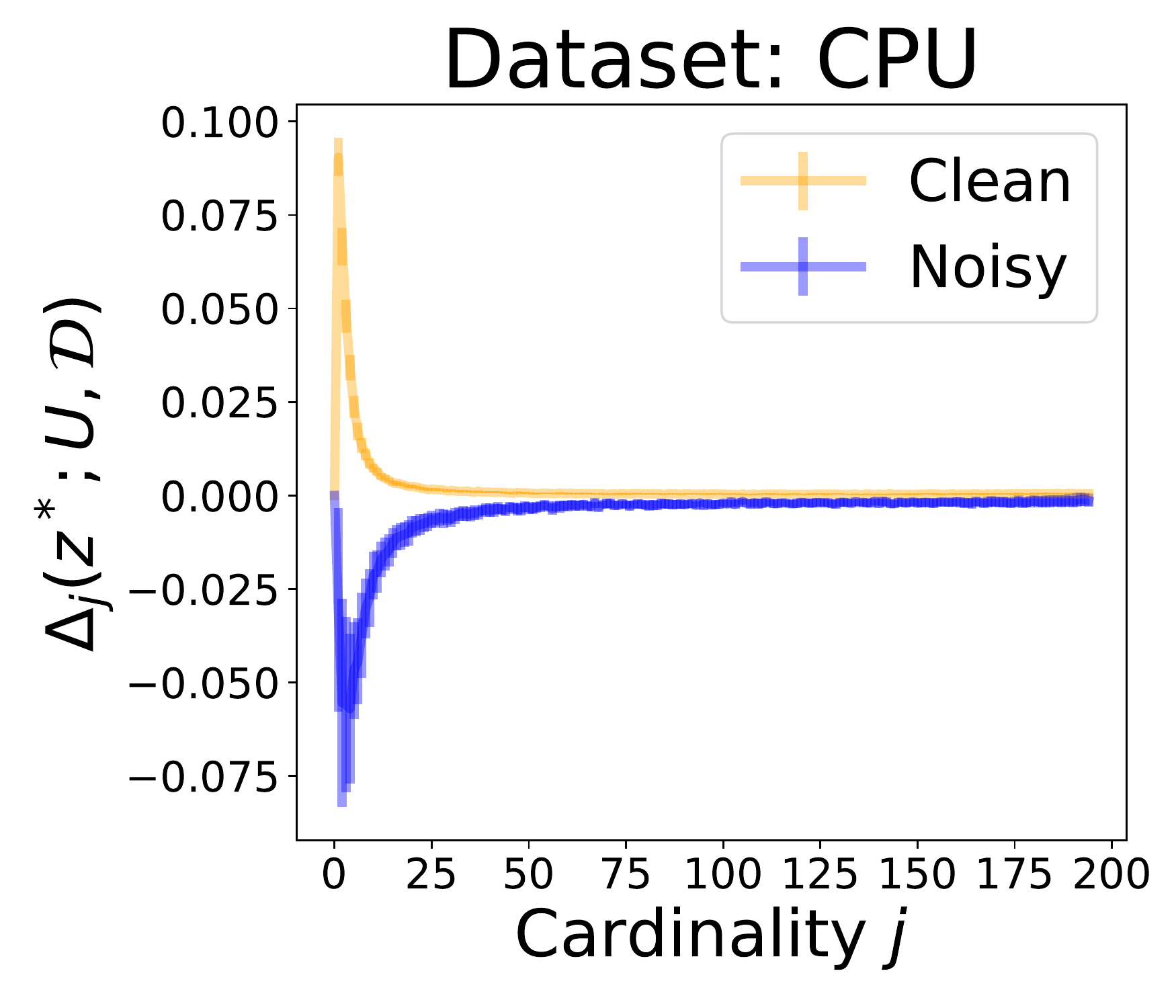}
    \includegraphics[width=0.245\textwidth]{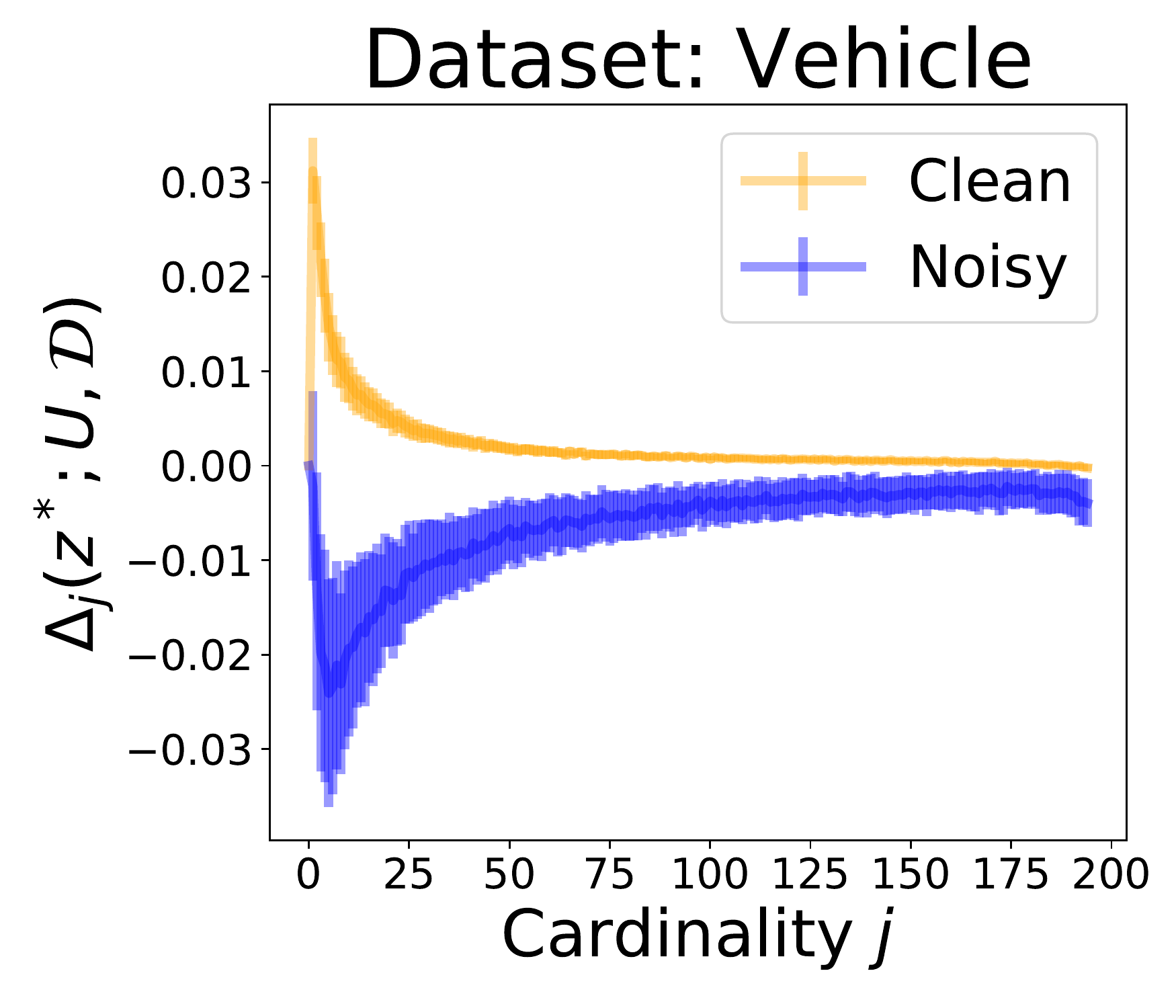}
    \caption{Illustrations of the marginal contributions $\Delta_j (z^*; U, \mathcal{D})$ as a function of the cardinality $j$ on the eleven datasets. Each color indicates a noisy (blue) and a clean (yellow) data point. When the cardinality $j$ is large, it is hard to tell if point is noisy or not as they become similar or even reversed.}
    \label{fig:app_clean_noisy_marginal_contributions}
\end{figure*}

\begin{figure*}[t]
    \vspace{-0.05in}
    \centering
    \includegraphics[width=0.245\textwidth]{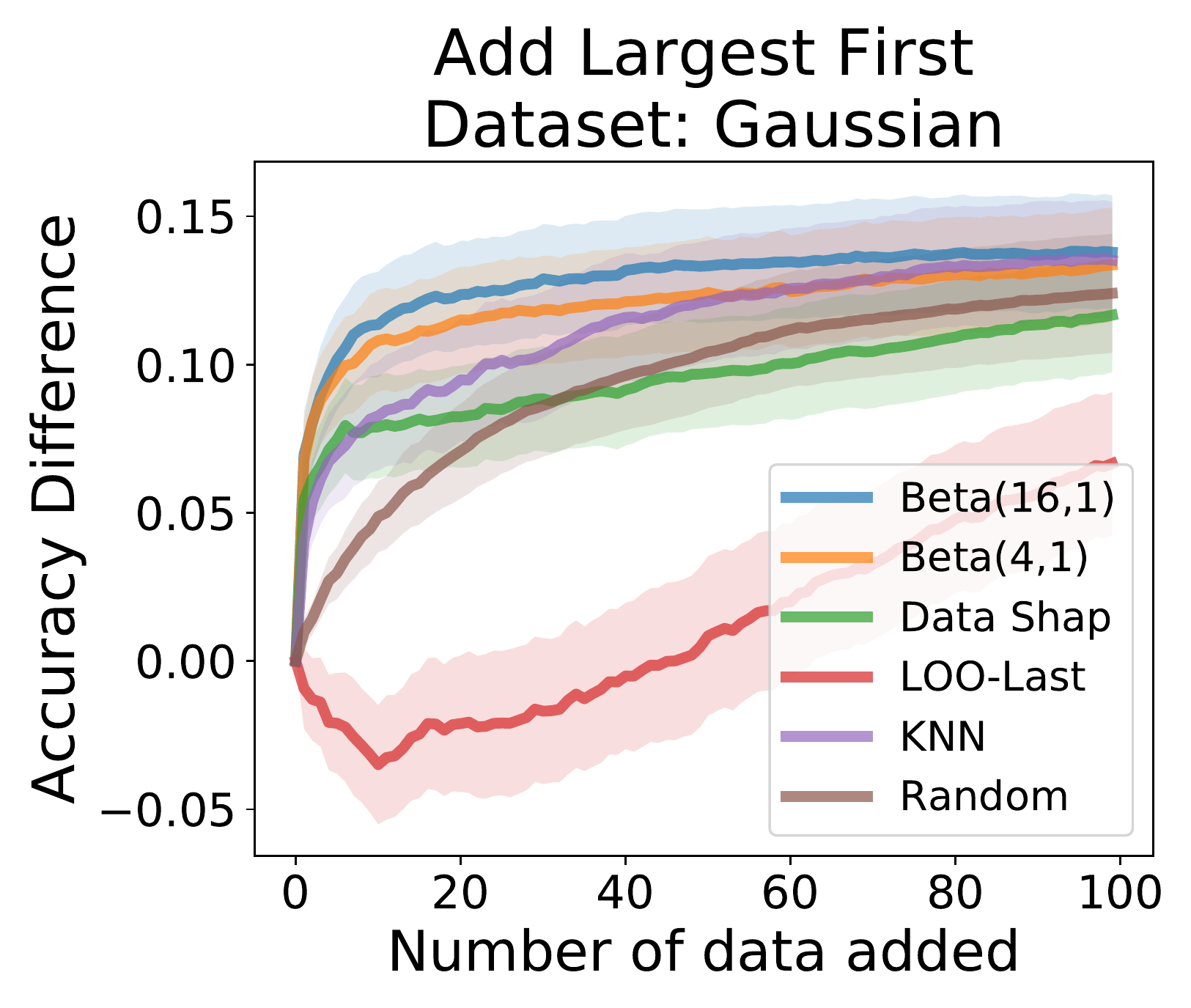}
    \includegraphics[width=0.245\textwidth]{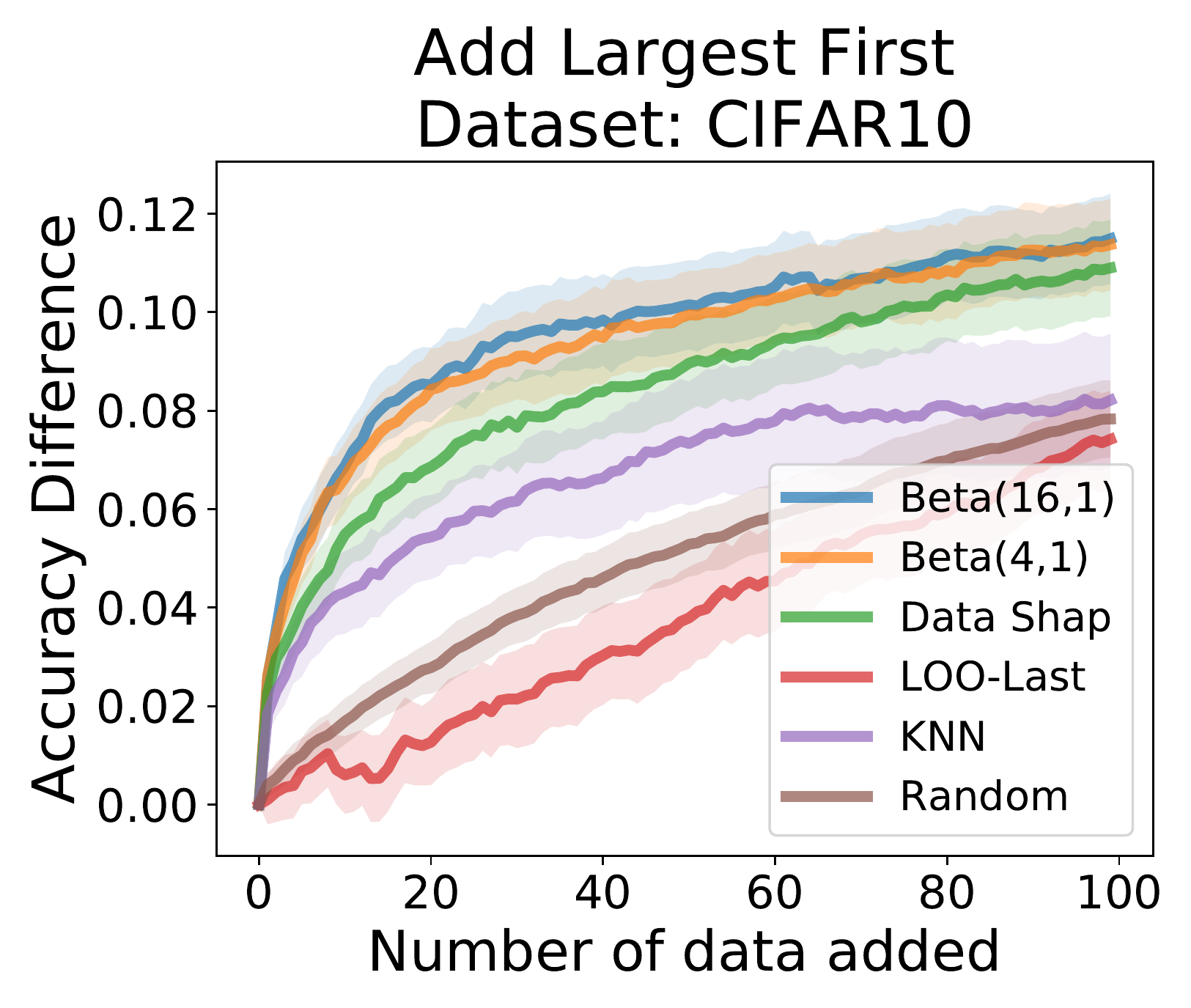}
    \includegraphics[width=0.245\textwidth]{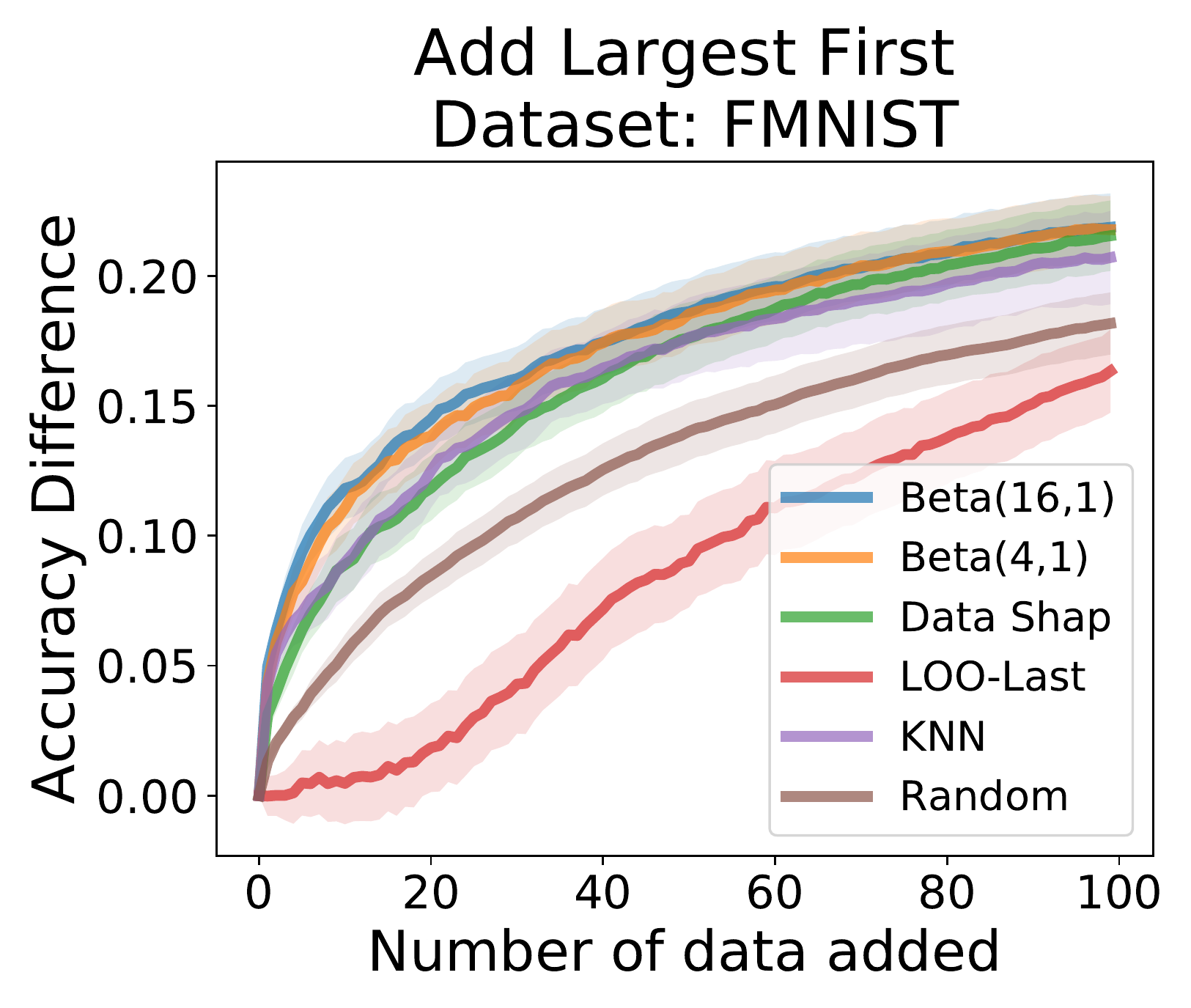}\\
    \includegraphics[width=0.245\textwidth]{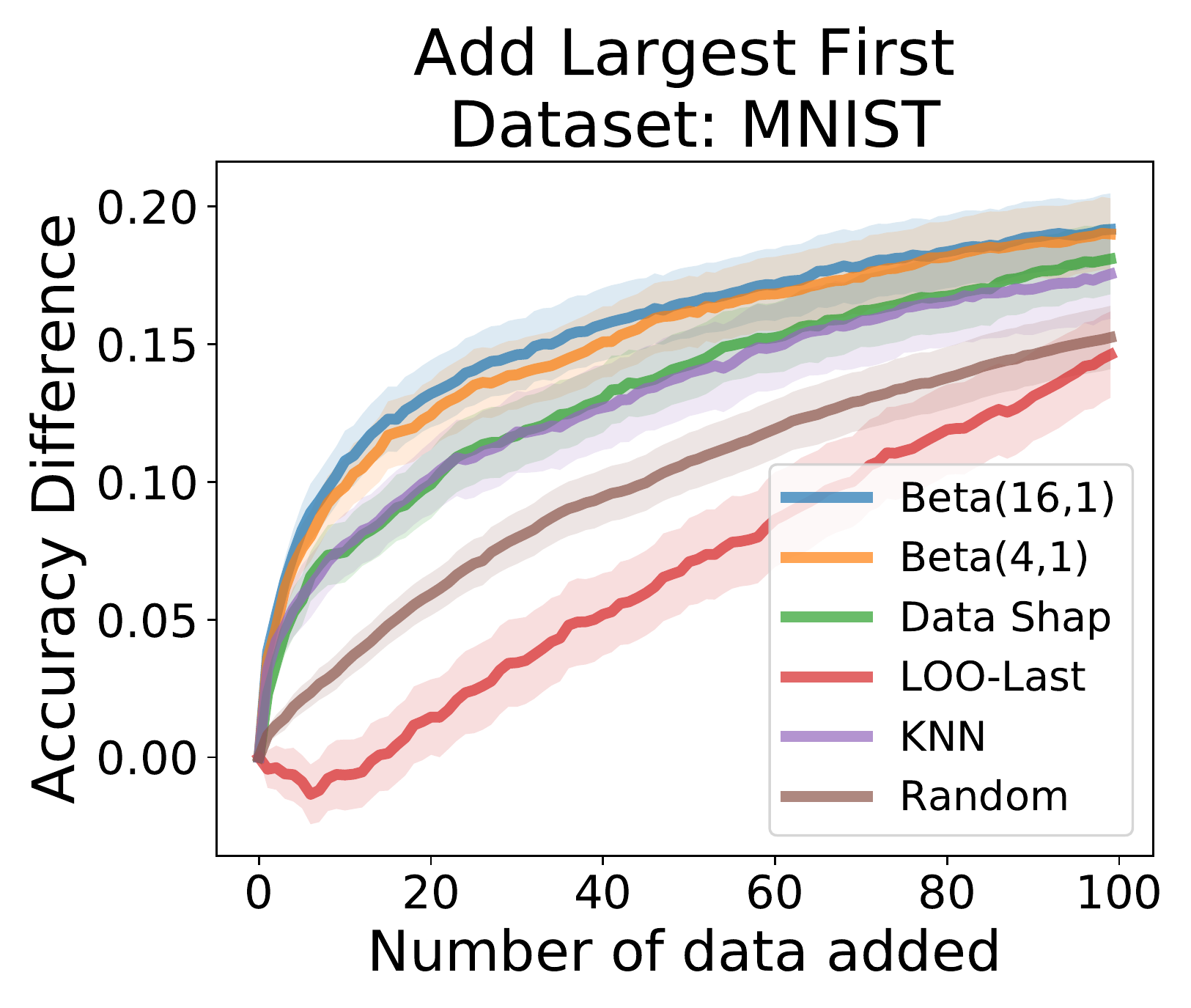}
    \includegraphics[width=0.245\textwidth]{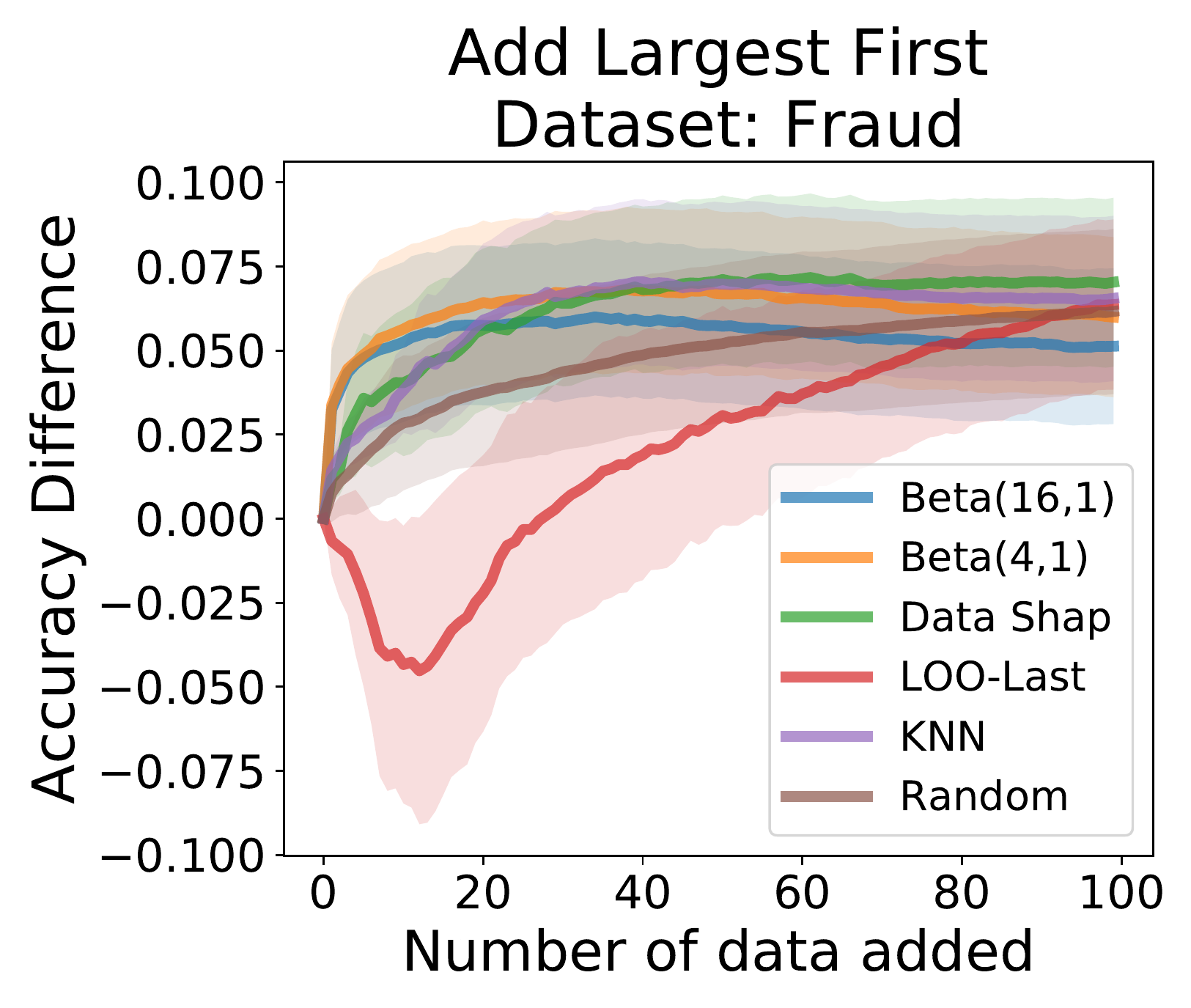}
    \includegraphics[width=0.245\textwidth]{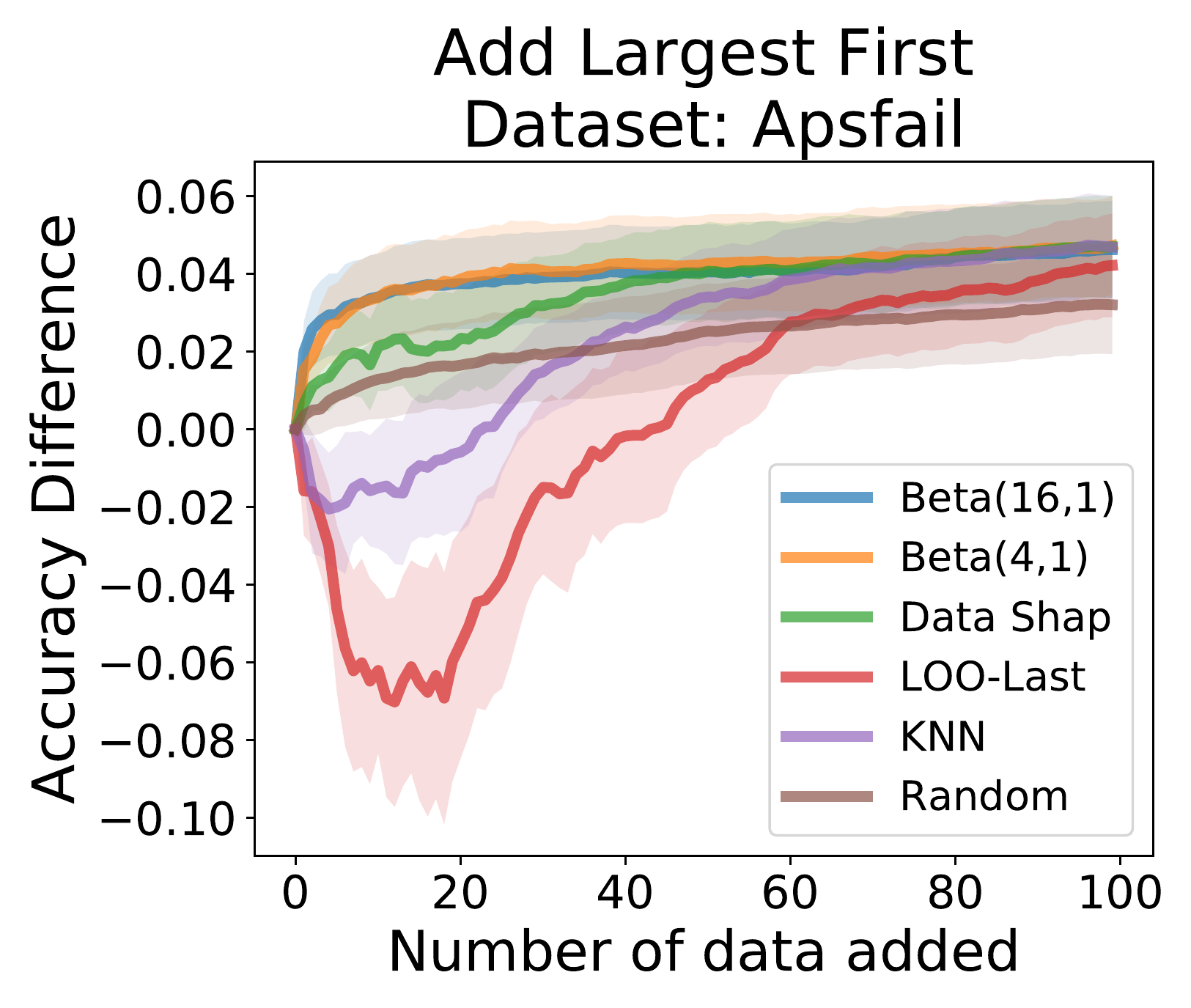}\\
    \includegraphics[width=0.245\textwidth]{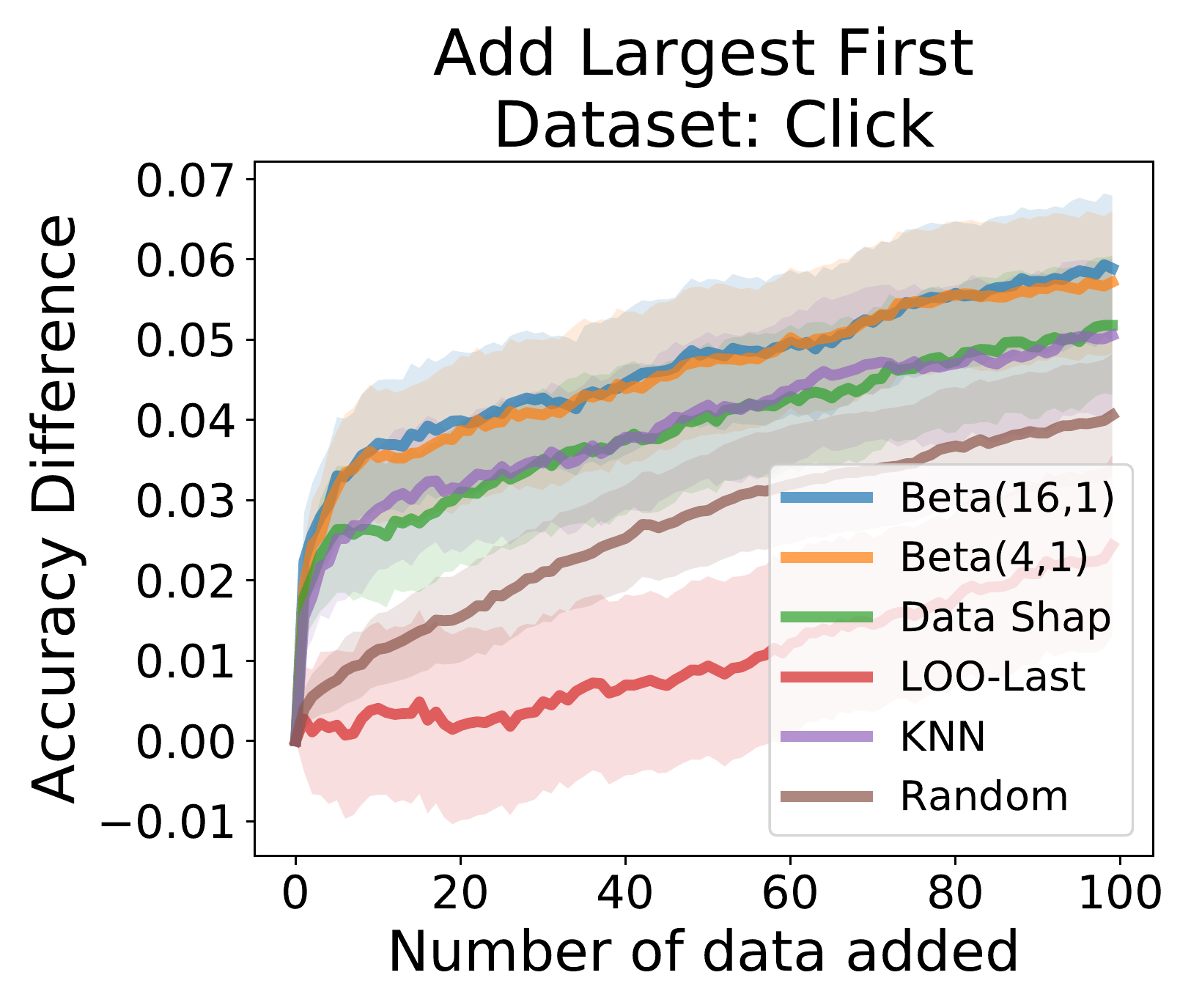}
    \includegraphics[width=0.245\textwidth]{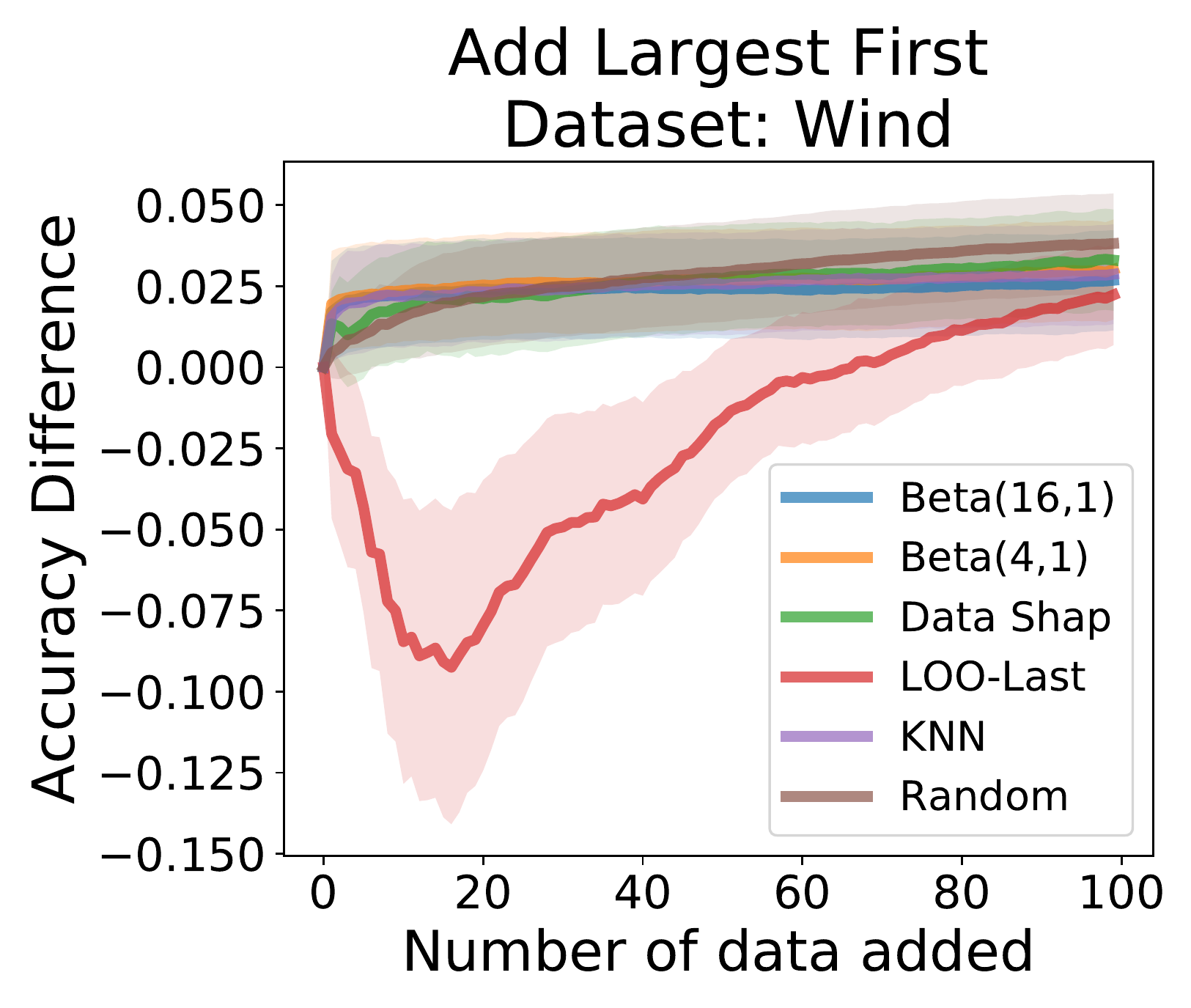}
    \includegraphics[width=0.245\textwidth]{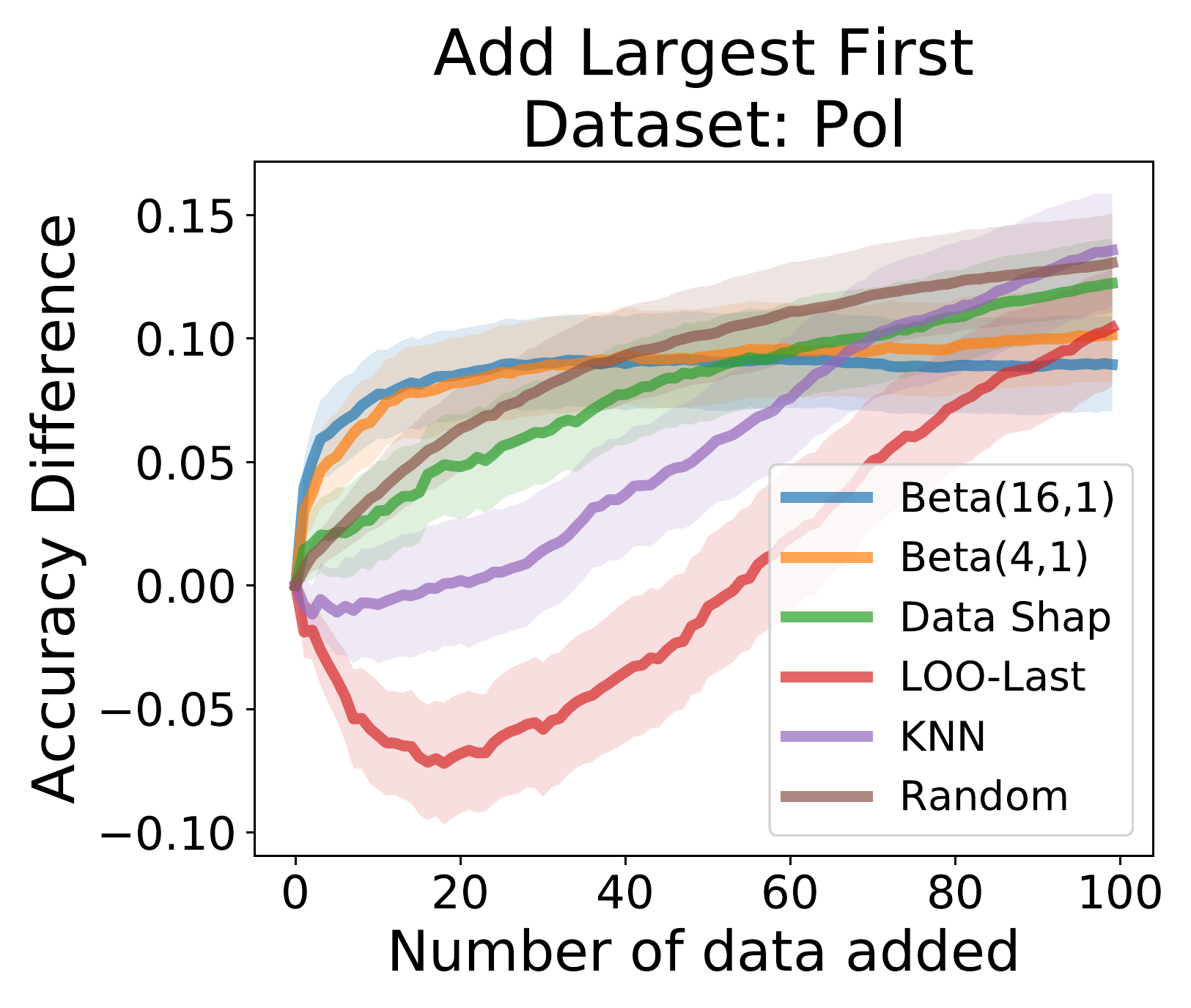}\\
    \includegraphics[width=0.245\textwidth]{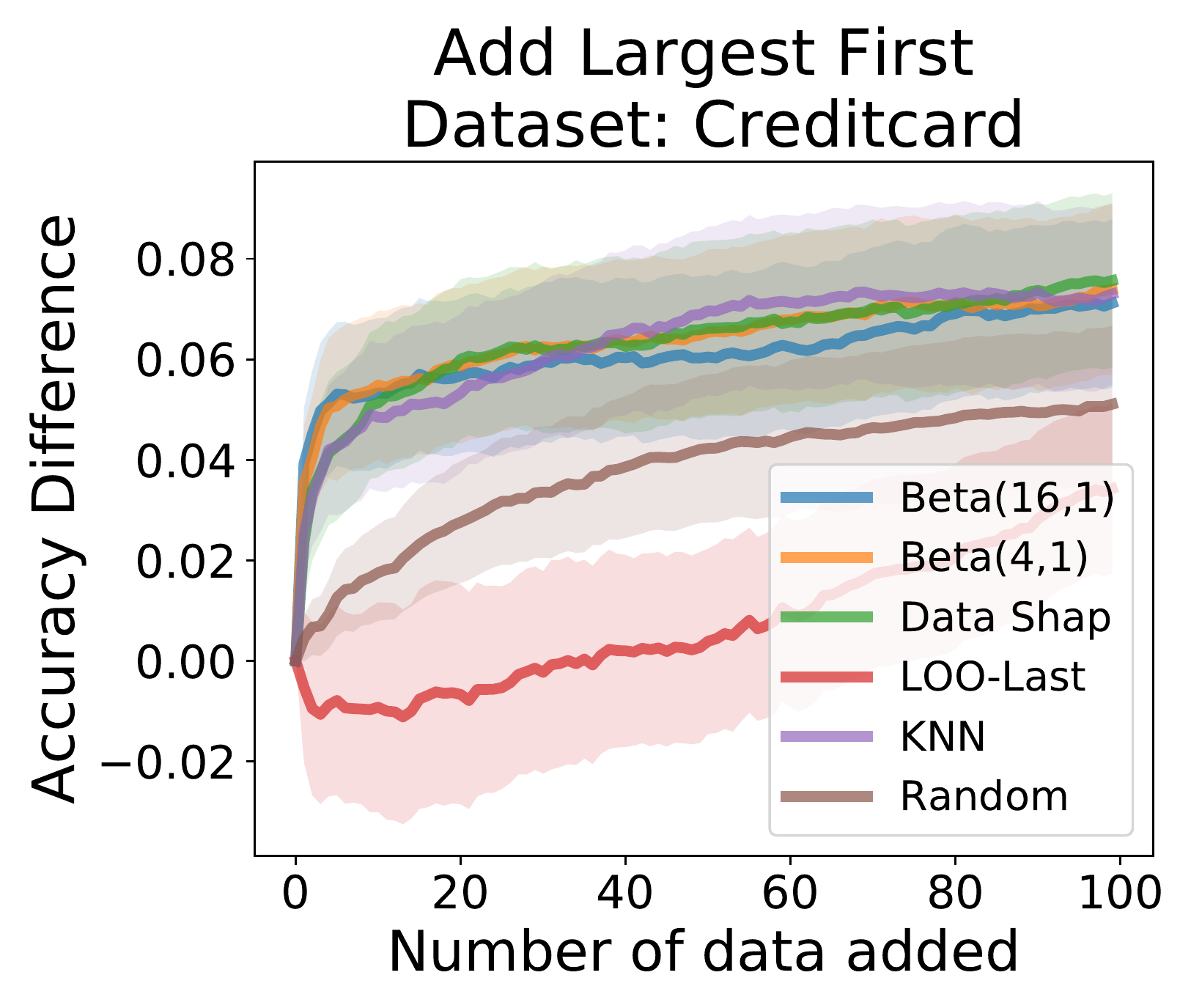}
    \includegraphics[width=0.245\textwidth]{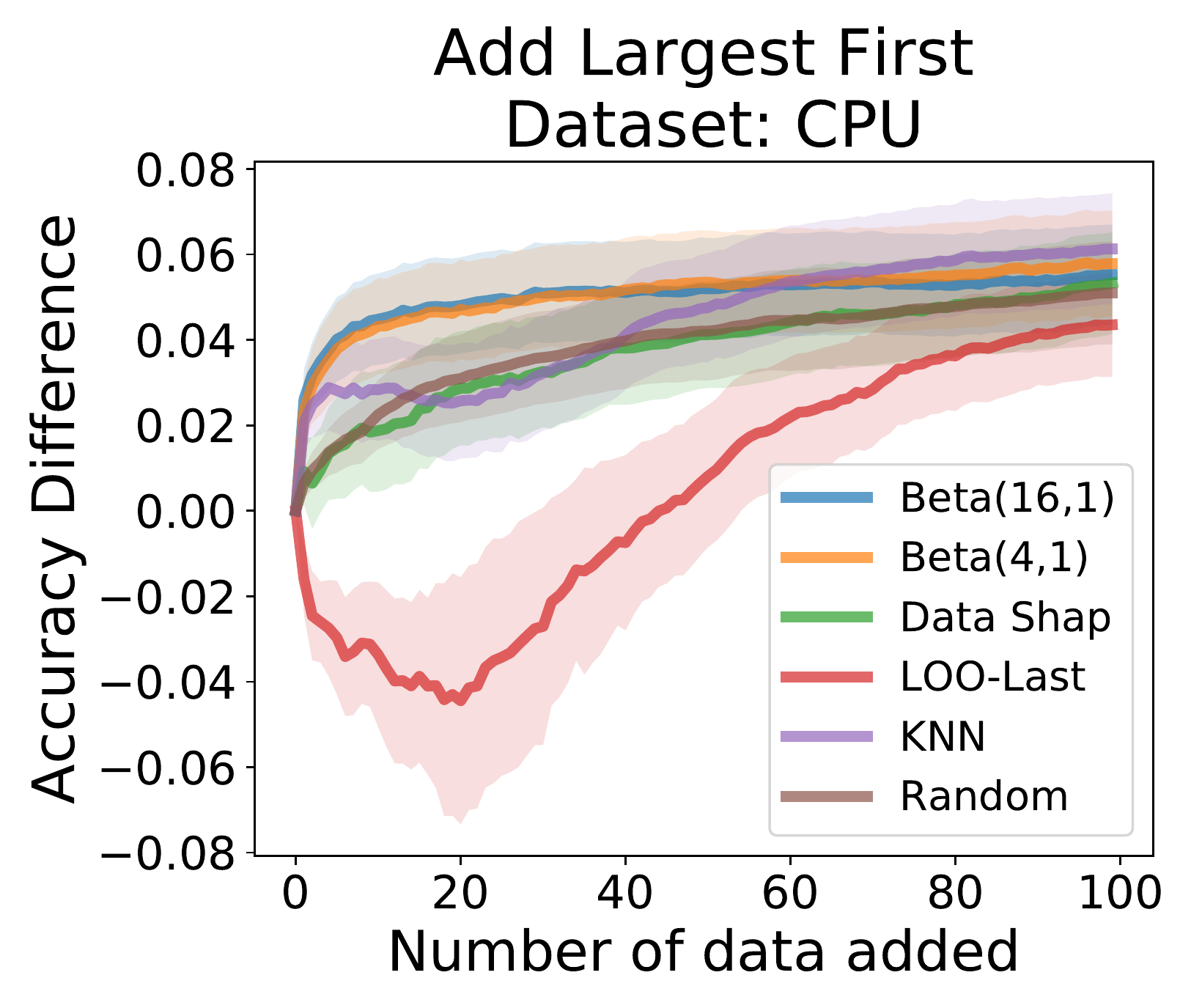}
    \includegraphics[width=0.245\textwidth]{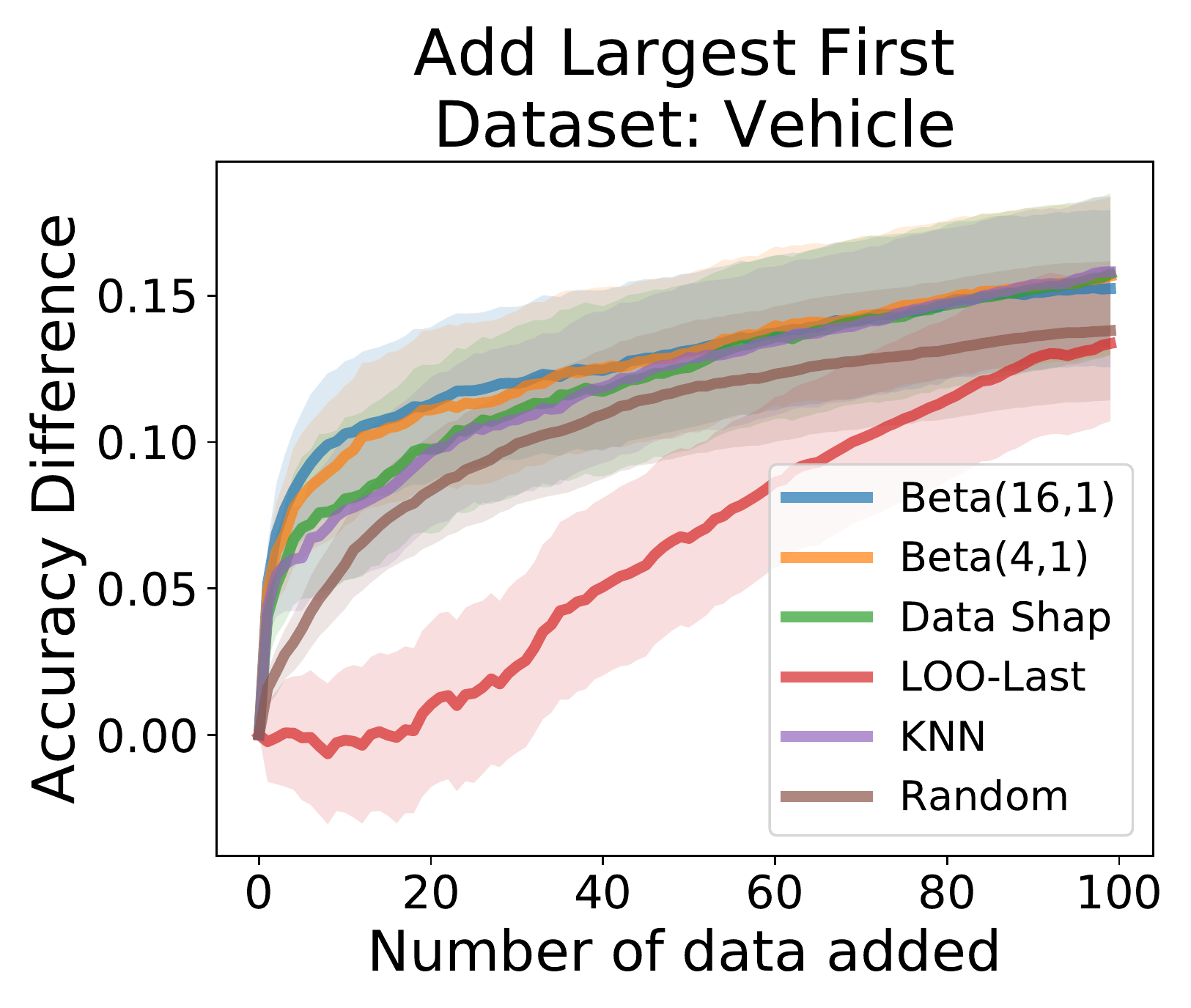}
    \includegraphics[width=0.245\textwidth]{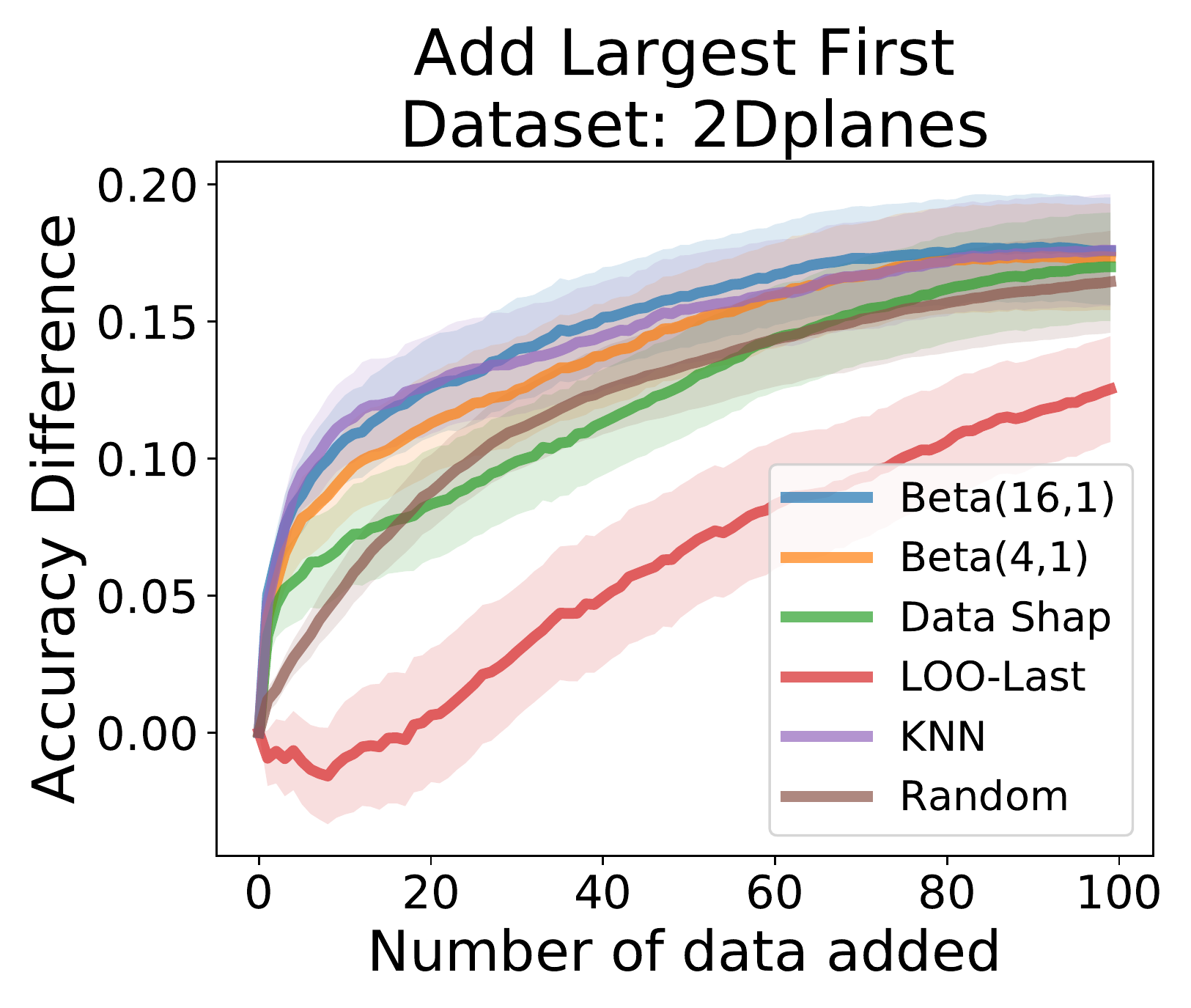}
    \caption{Accuracy change as a function of the number of data points added on the thirteen datasets. We add data points whose value is small first.}
    \label{fig:app_point_addition_experiment}
\end{figure*}

\begin{figure*}[t]
    \vspace{-0.05in}
    \centering
    \includegraphics[width=0.245\textwidth]{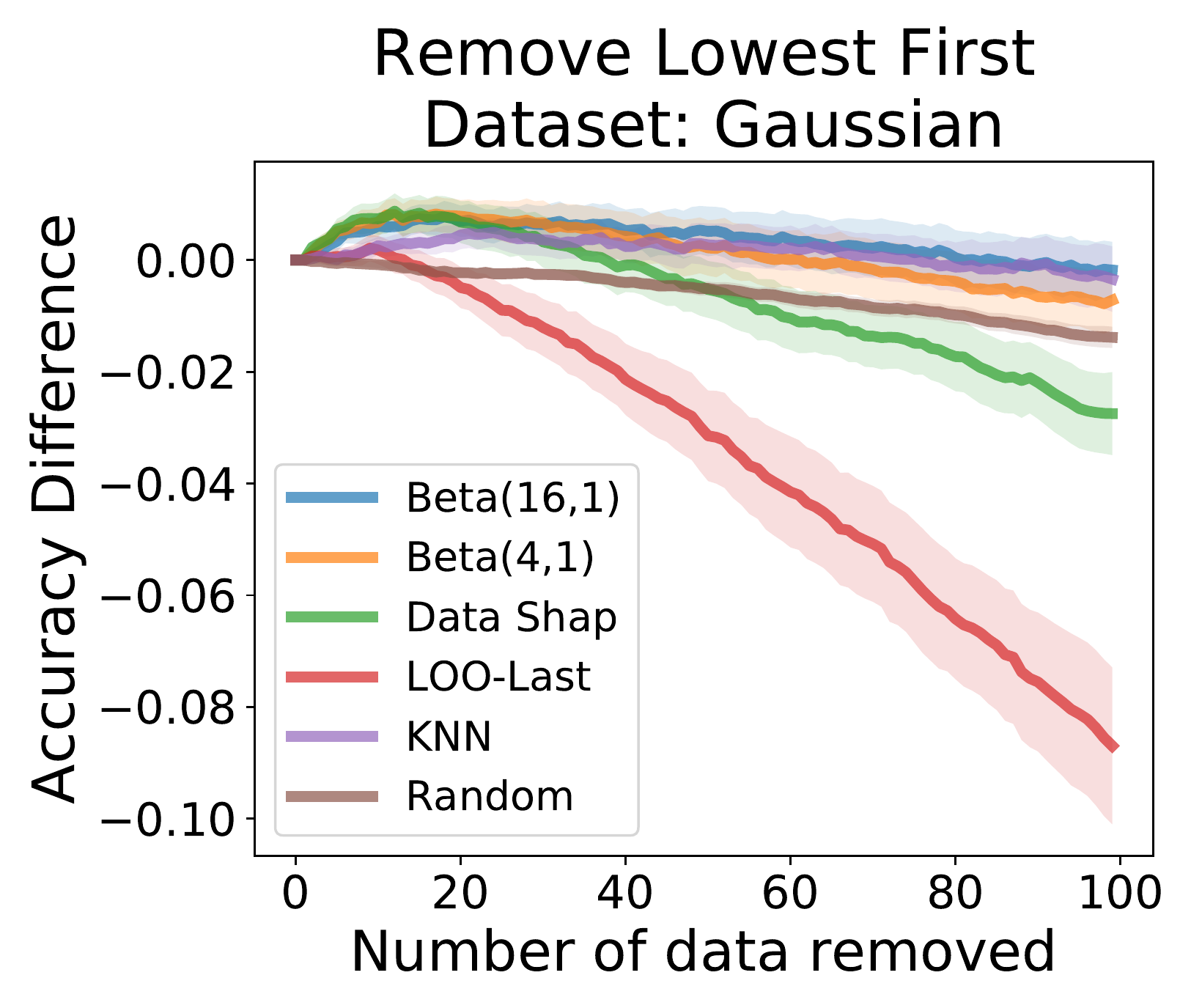}
    \includegraphics[width=0.245\textwidth]{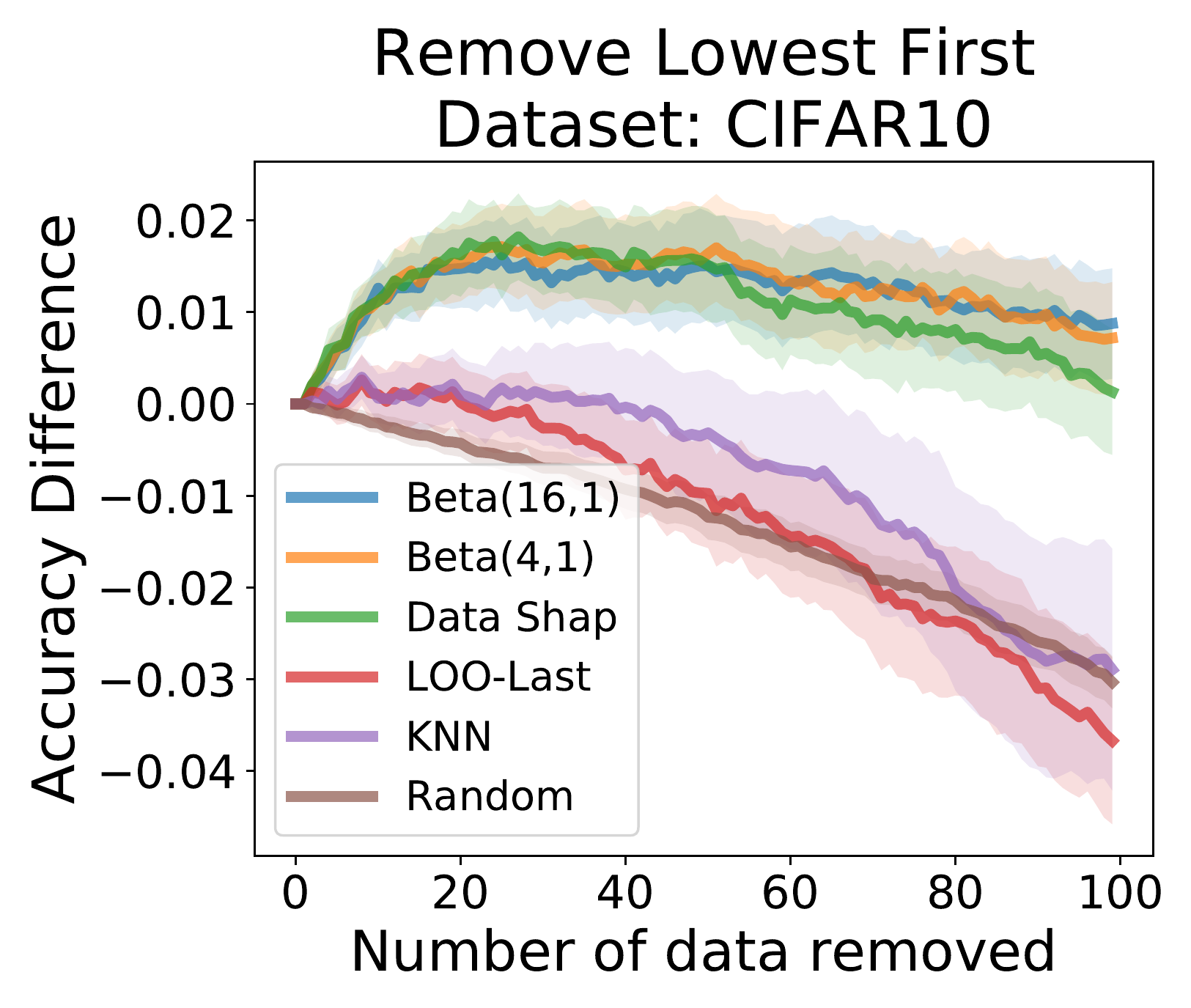}
    \includegraphics[width=0.245\textwidth]{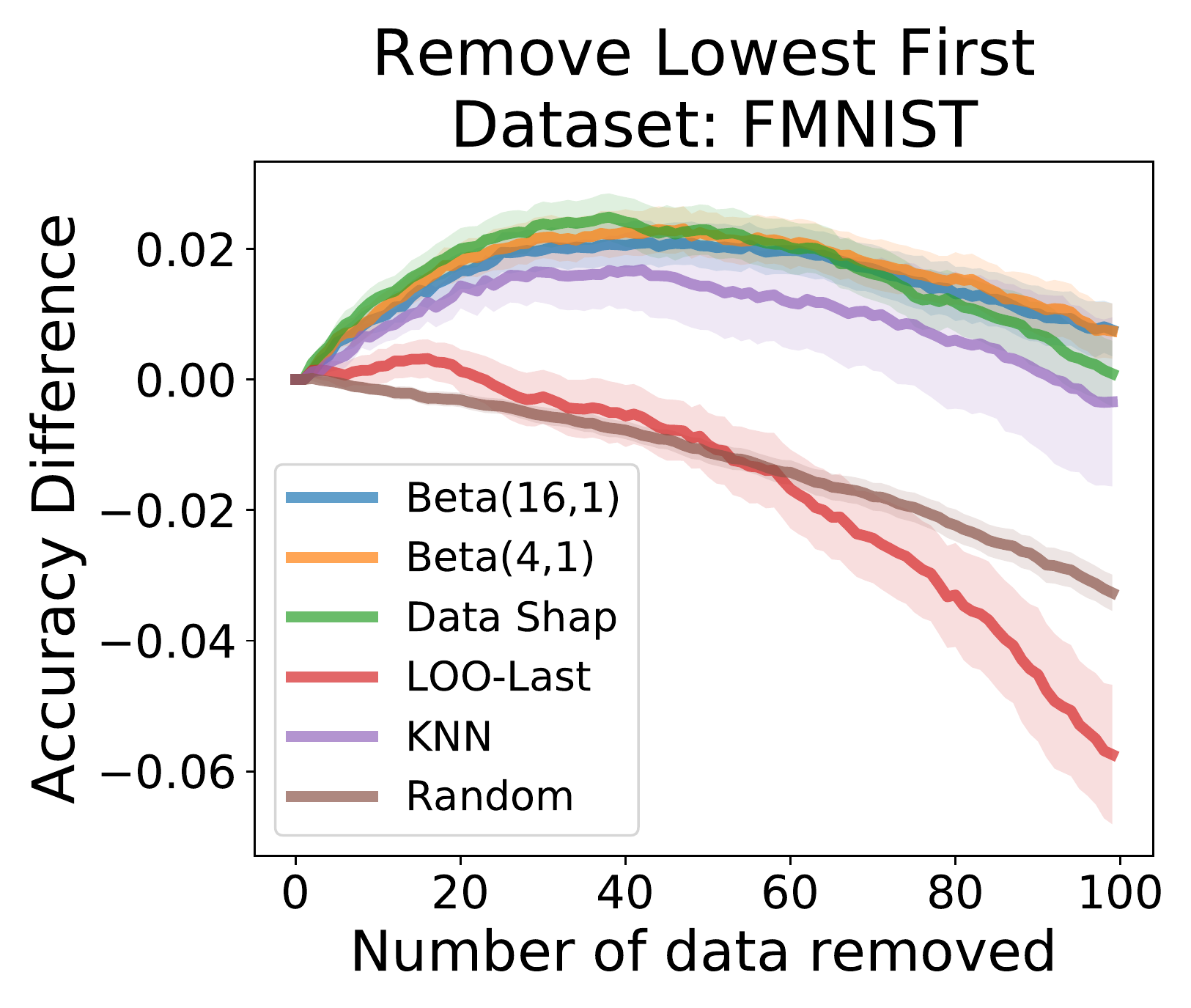}\\
    \includegraphics[width=0.245\textwidth]{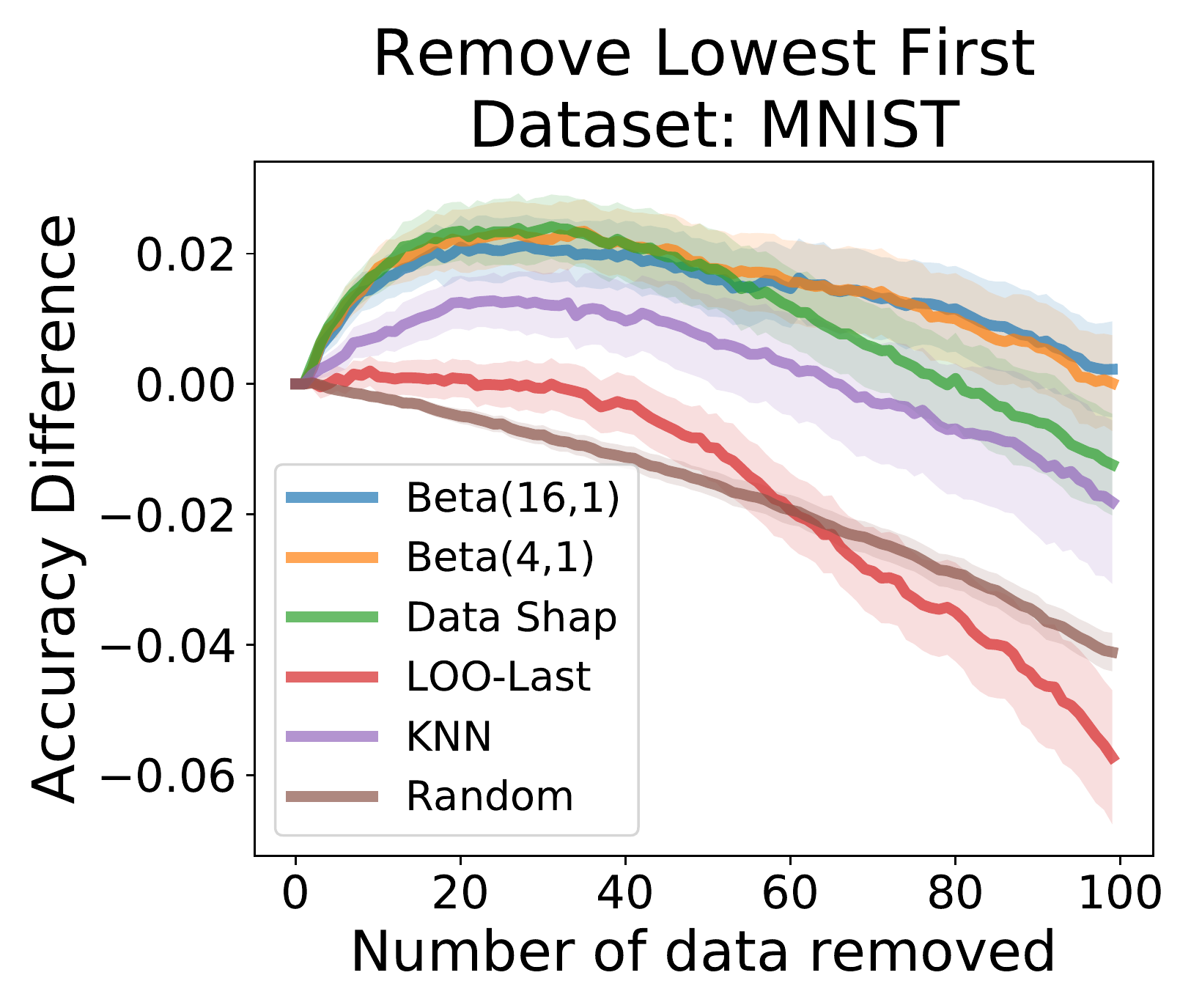}
    \includegraphics[width=0.245\textwidth]{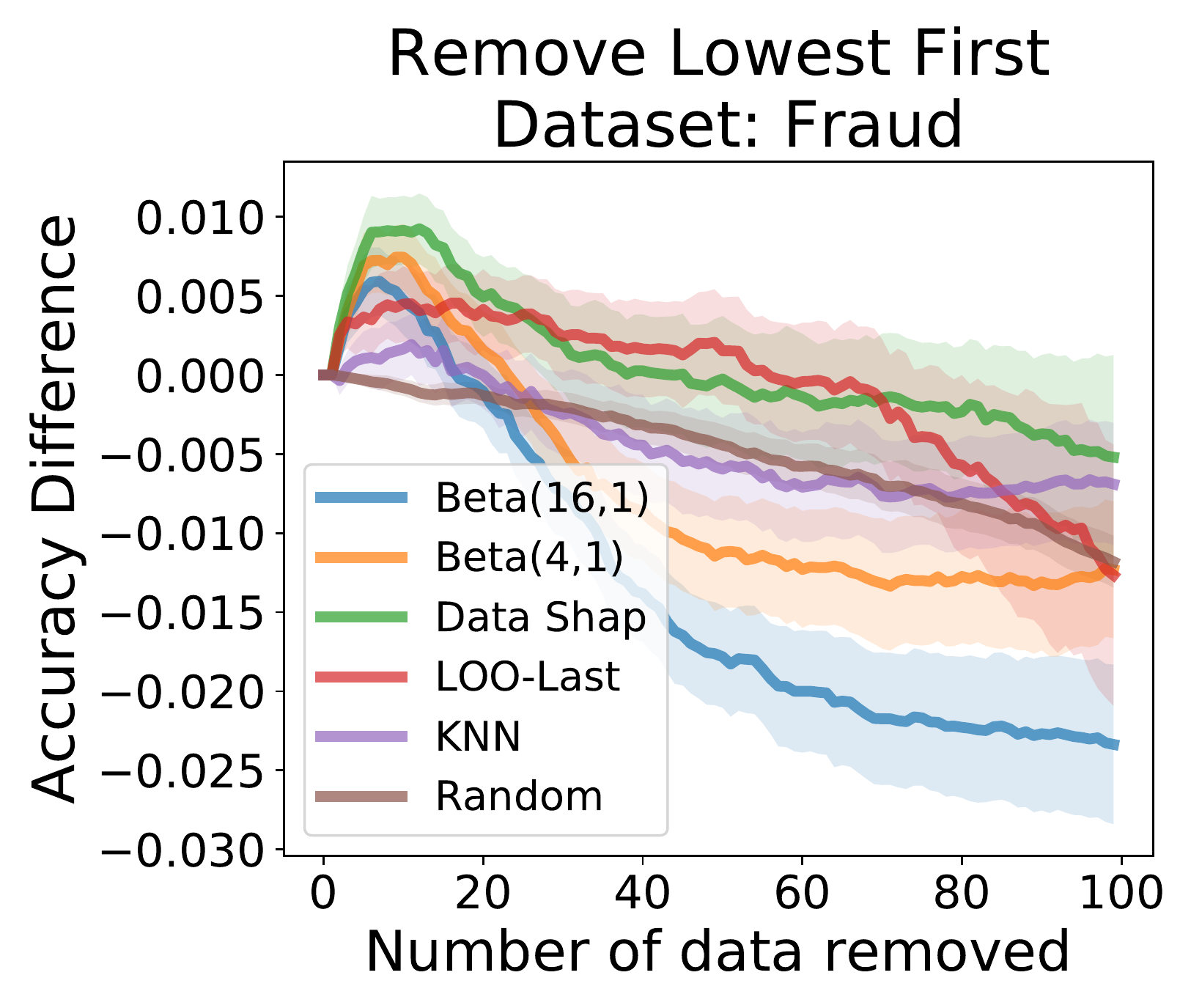}
    \includegraphics[width=0.245\textwidth]{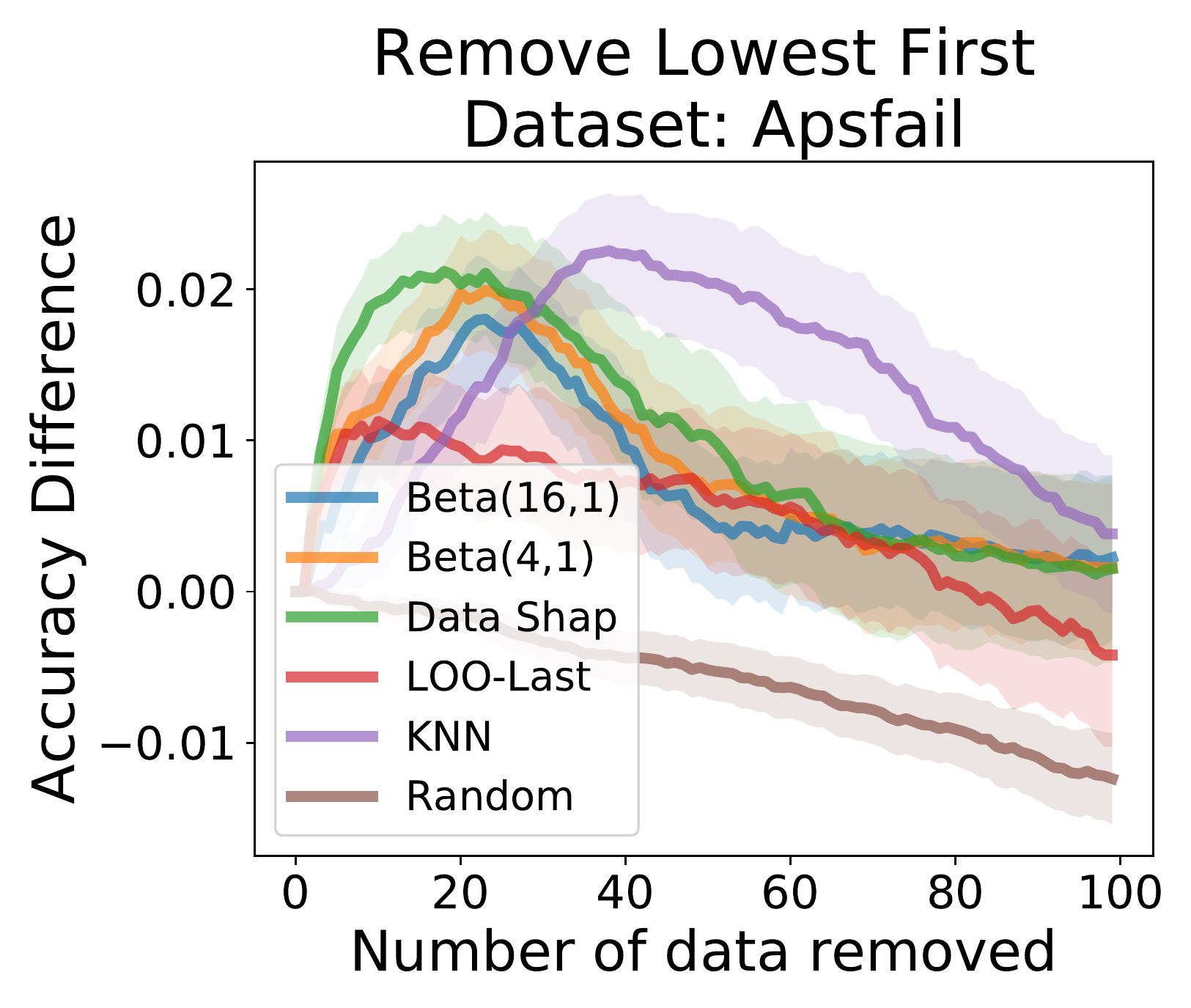}\\
    \includegraphics[width=0.245\textwidth]{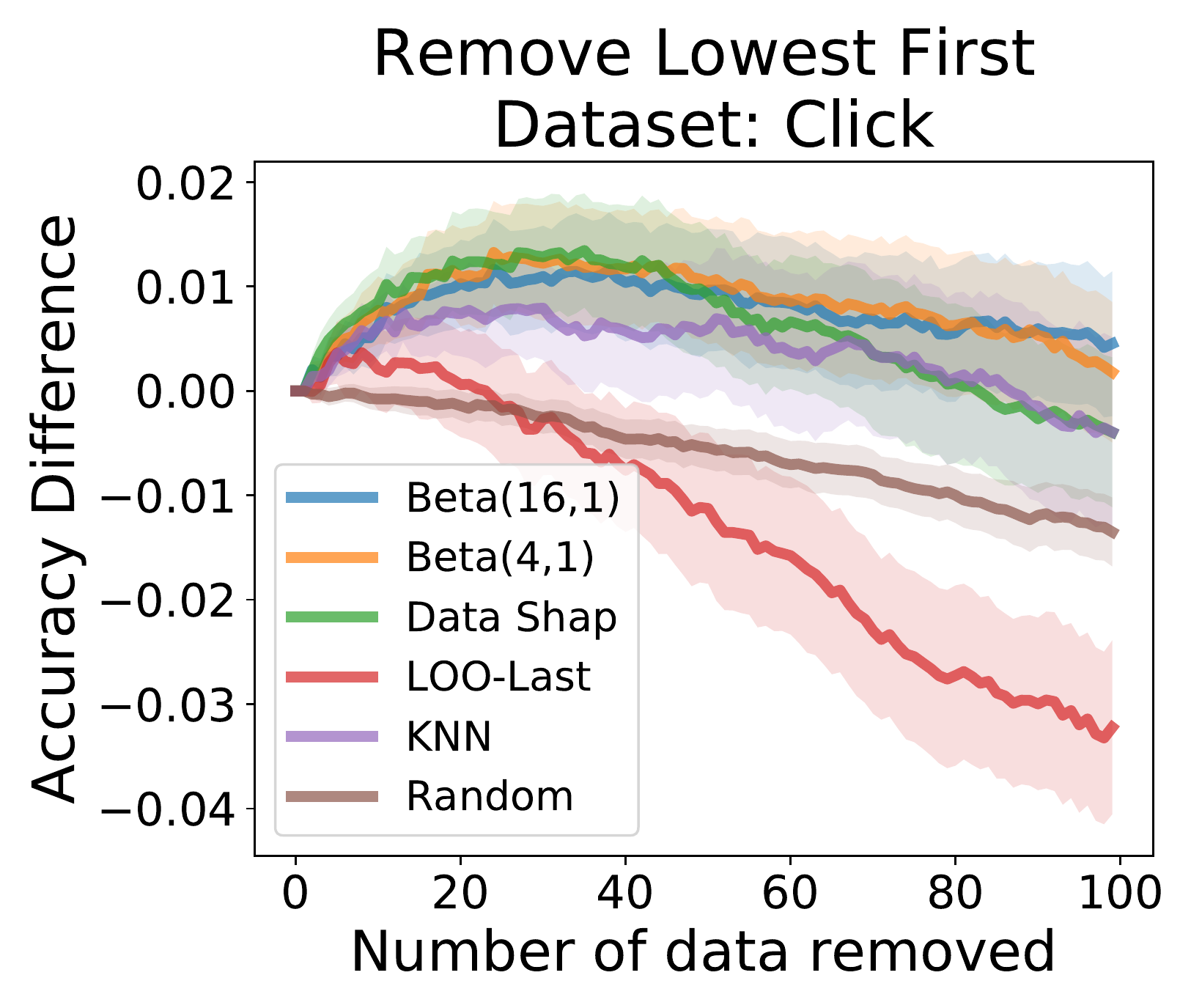}
    \includegraphics[width=0.245\textwidth]{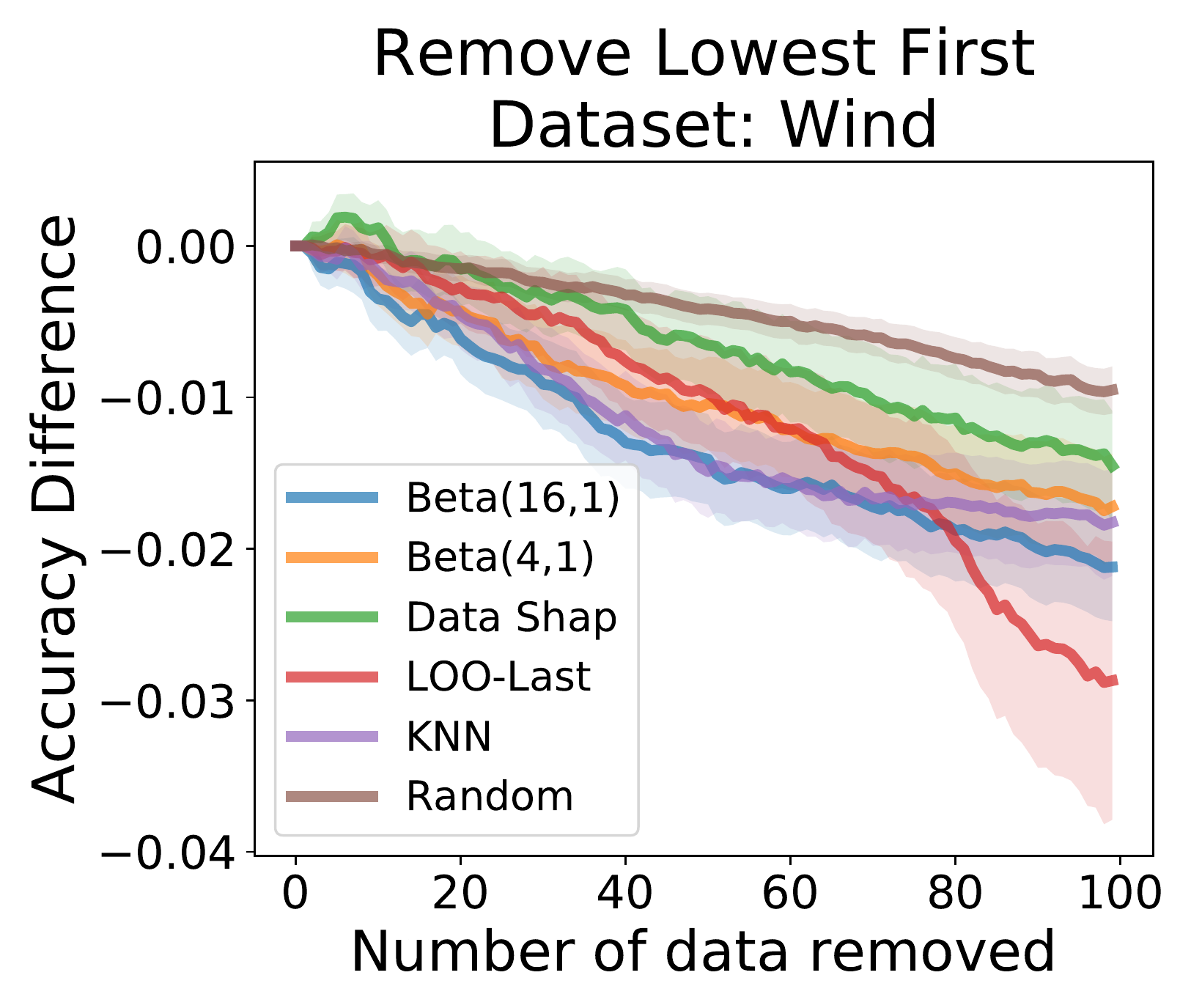}
    \includegraphics[width=0.245\textwidth]{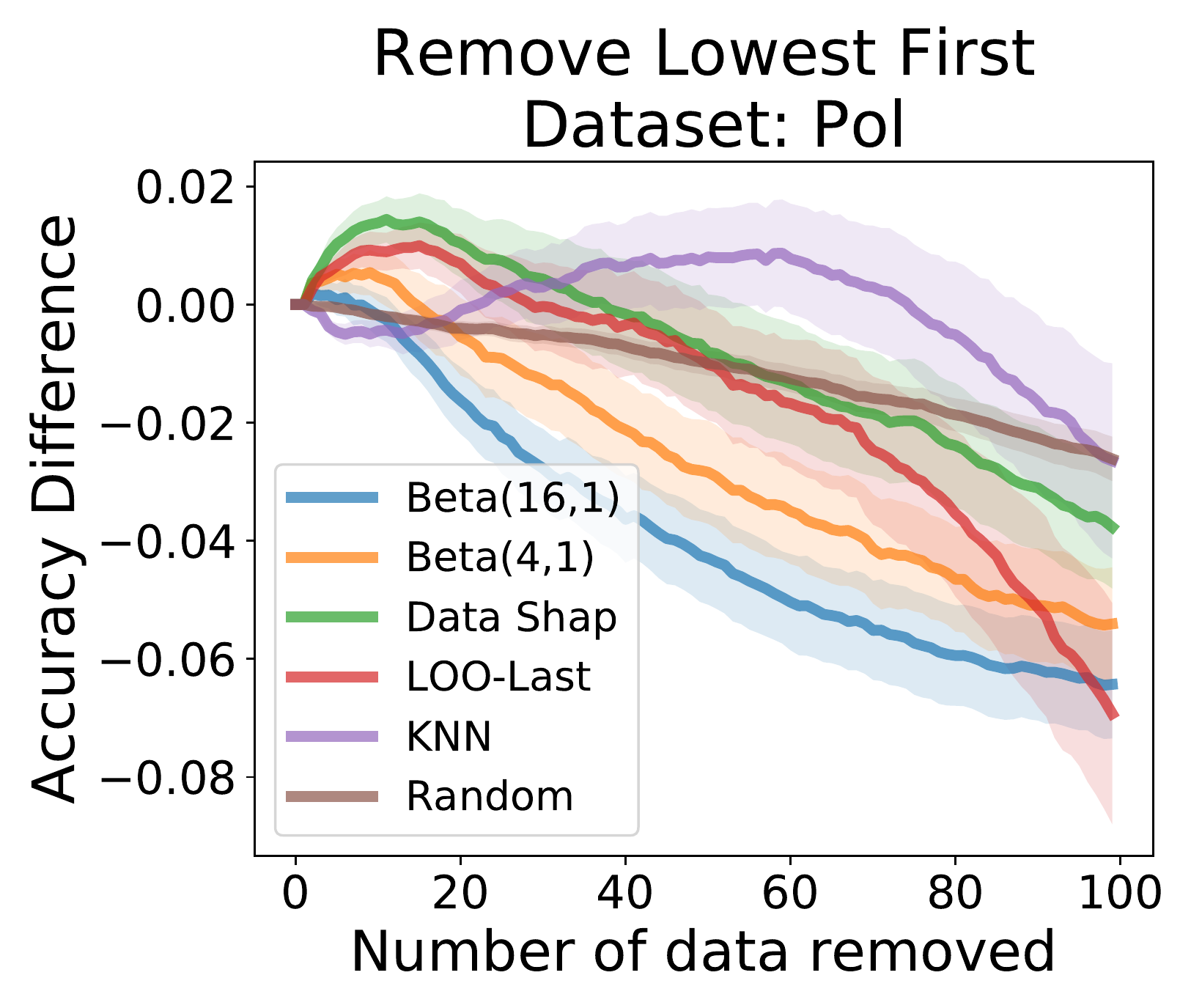}\\
    \includegraphics[width=0.245\textwidth]{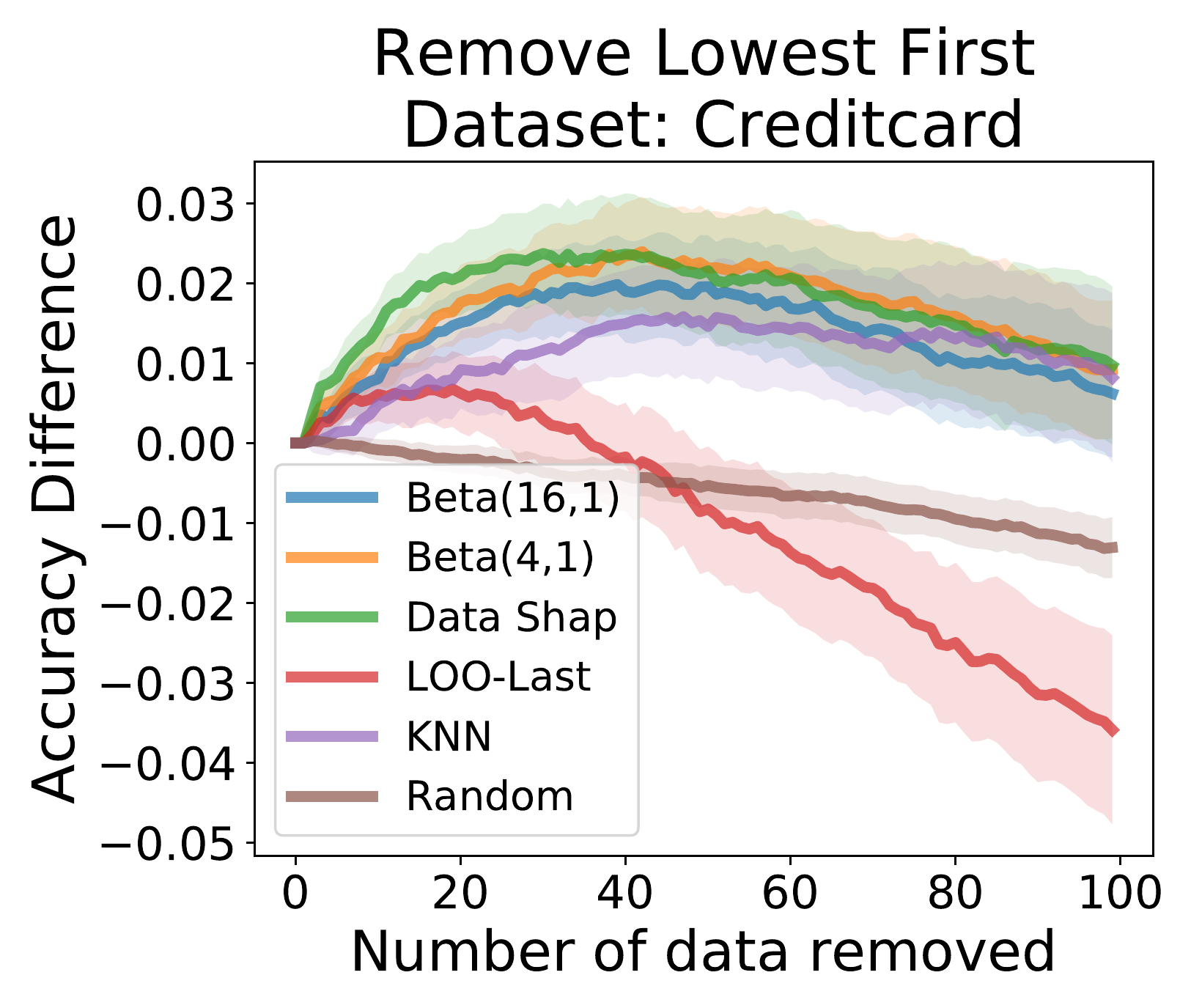}
    \includegraphics[width=0.245\textwidth]{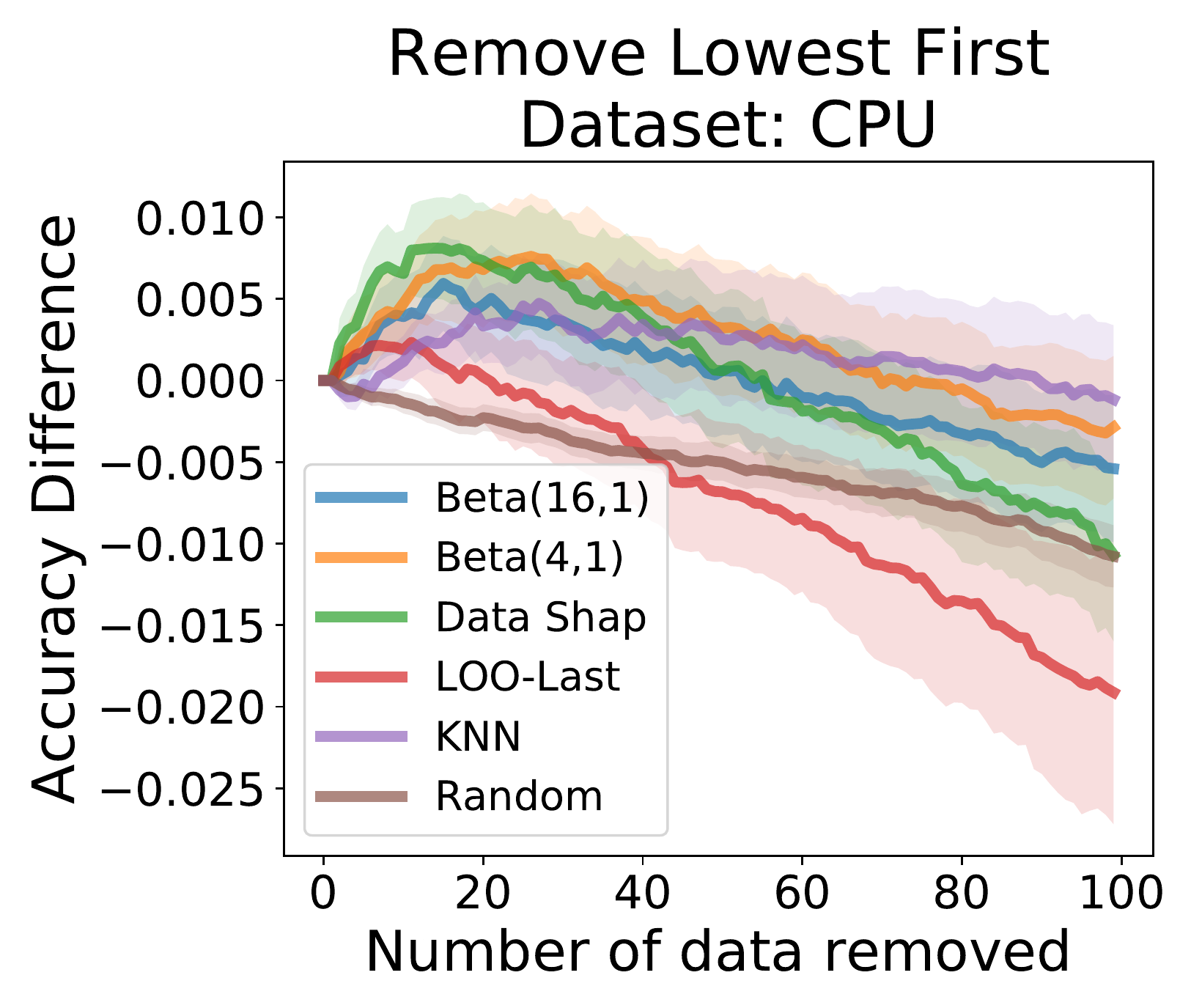}
    \includegraphics[width=0.245\textwidth]{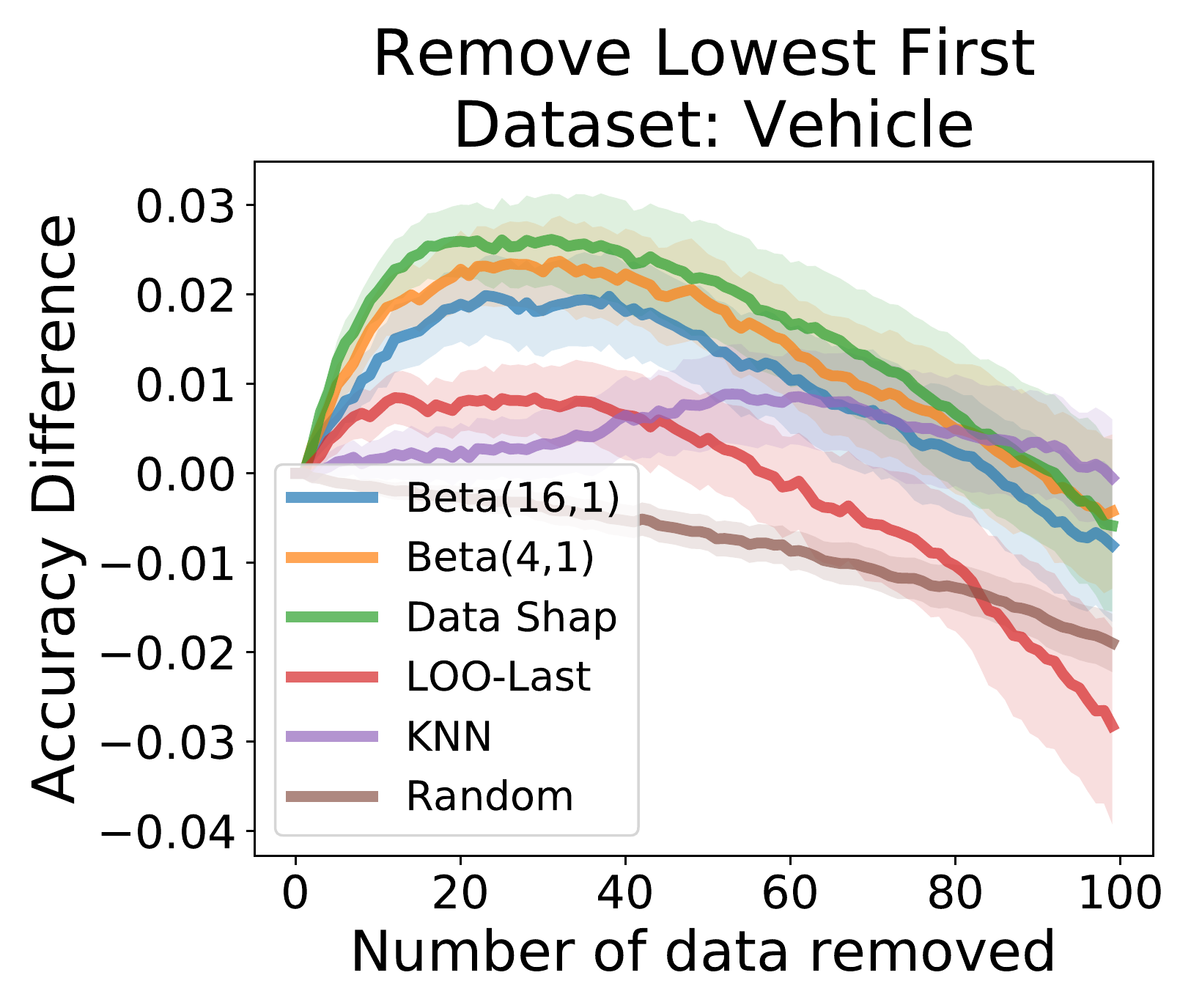}
    \includegraphics[width=0.245\textwidth]{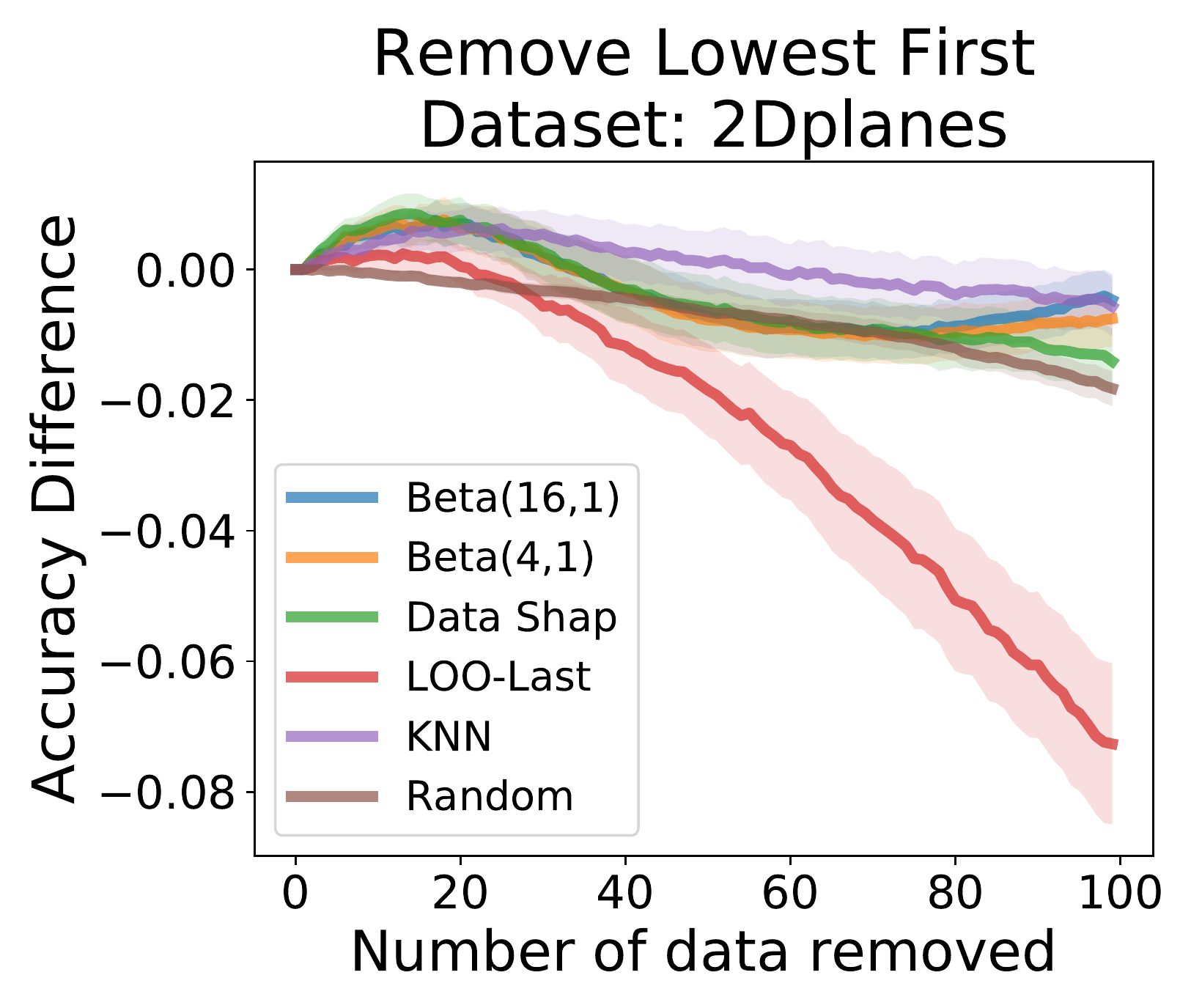}
    \caption{Accuracy change as a function of the number of data points removed on the thirteen datasets. We remove data points whose value is small first.}
    \label{fig:app_point_removal_experiment}
\end{figure*}

\subsection{Beta Shapley with a support vector machine model}
Throughout our experiments, we considered a logistic model to compute a utility. In this section, we demonstrate our results are robust against different models, in particular, a support vector machine model. In the following experiments, we use the same experiment setup but a model. Figure~\ref{fig:app_clean_noisy_marginal_contributions_svm} shows the marginal contribution for clean and noisy samples as a function of the cardinality. Similar to the case of a logistic regression model, the difference between clean and noisy groups is significantly big when the cardinality is small, but they overlap when the cardinality is big. This suggests the uniform weight used in data Shapley might not be optimal, but Beta Shapley can be more effective.
Figure~\ref{fig:app_heatmap_summary_count_svm} shows a summary of performance comparison on the fifteen datasets when a support vector machine is used.

\begin{figure*}[h]
    \centering
    \includegraphics[width=0.245\textwidth]{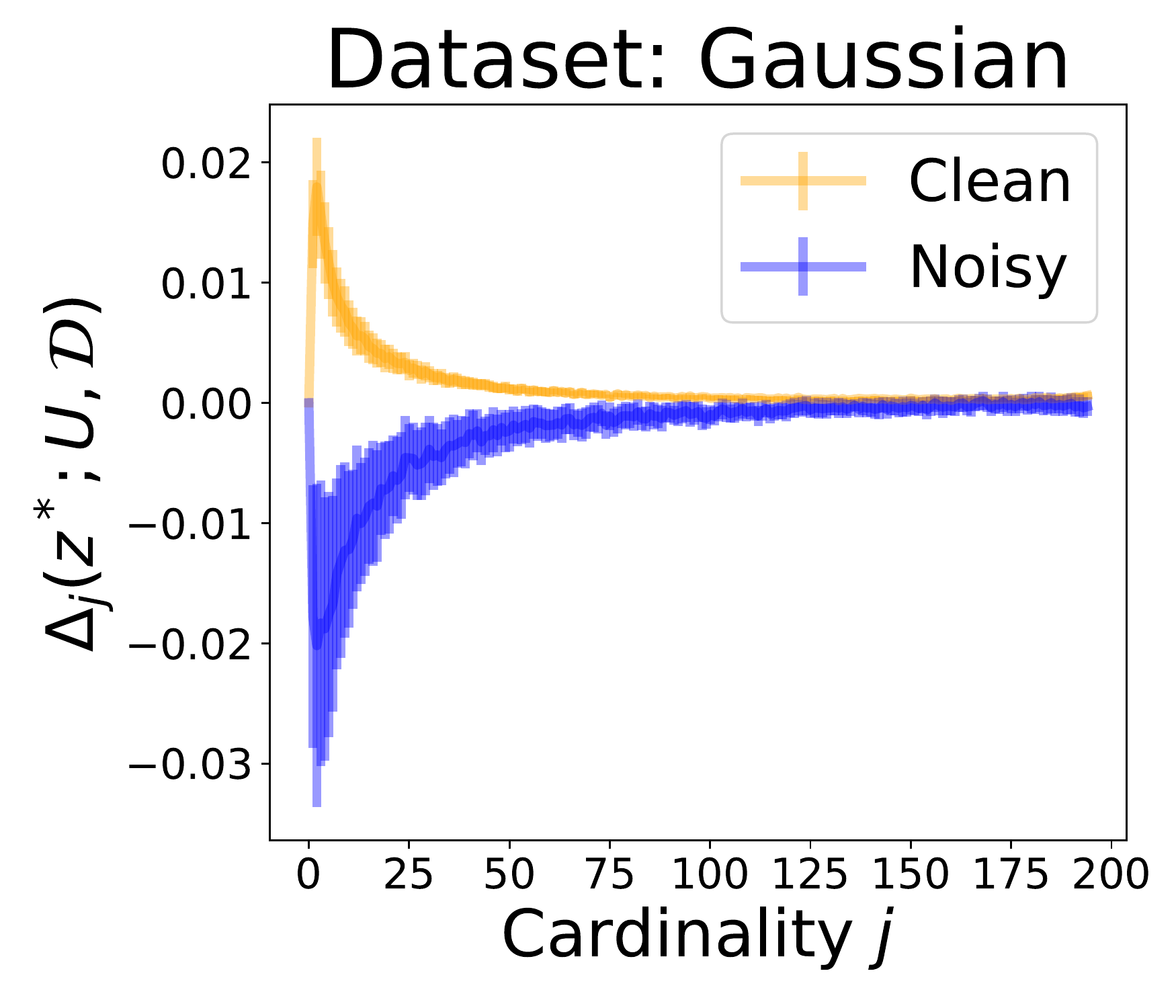}
    \includegraphics[width=0.245\textwidth]{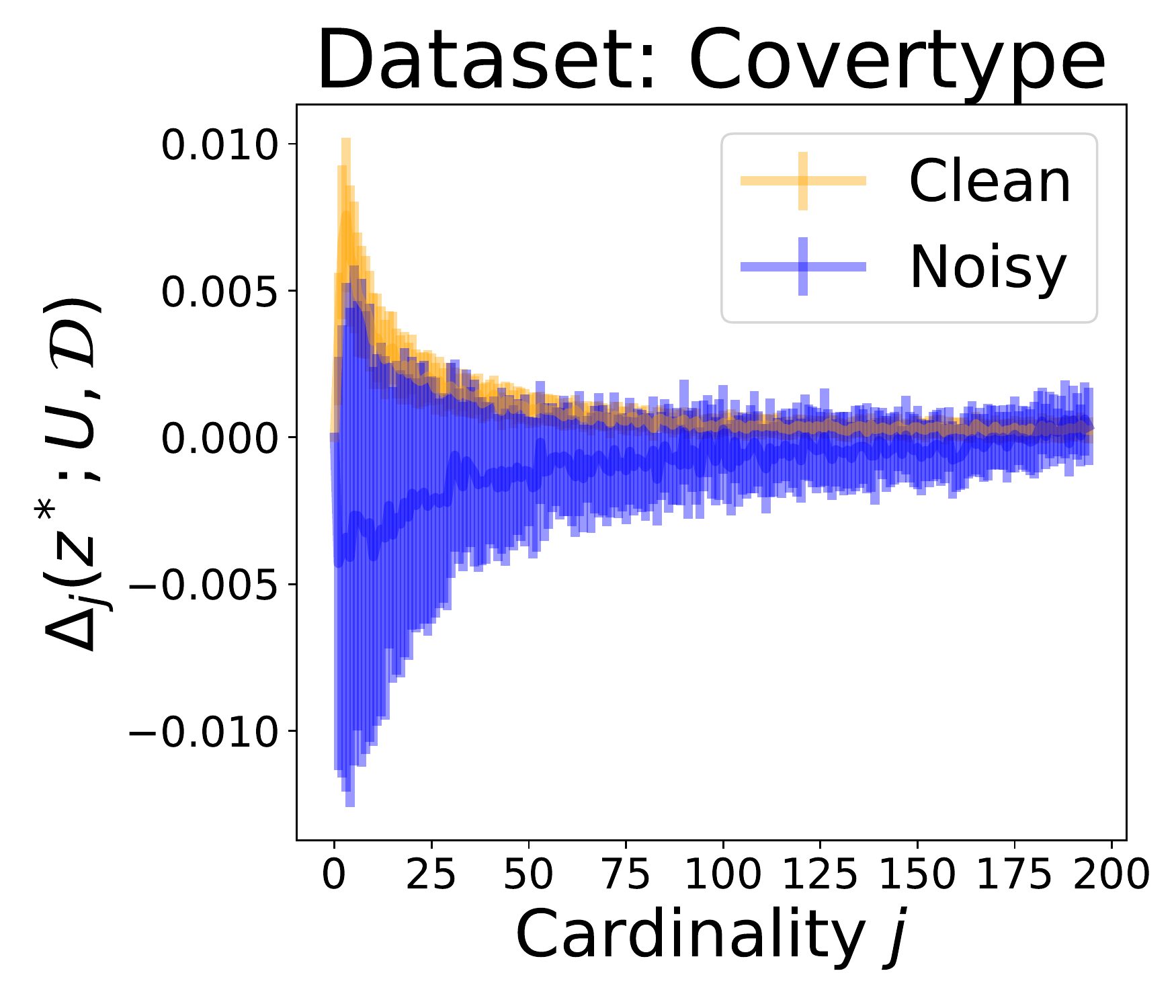}
    \includegraphics[width=0.245\textwidth]{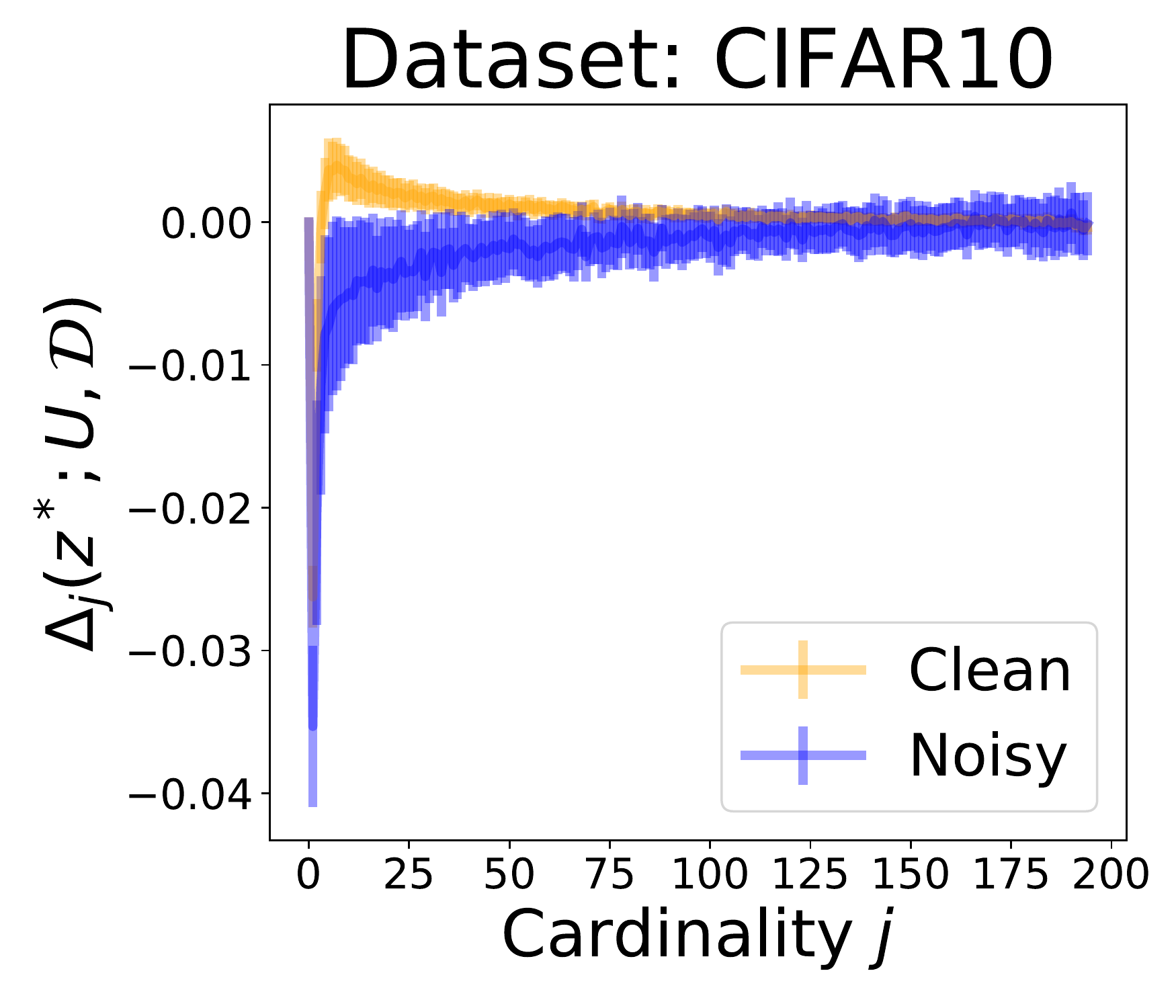}
    \includegraphics[width=0.245\textwidth]{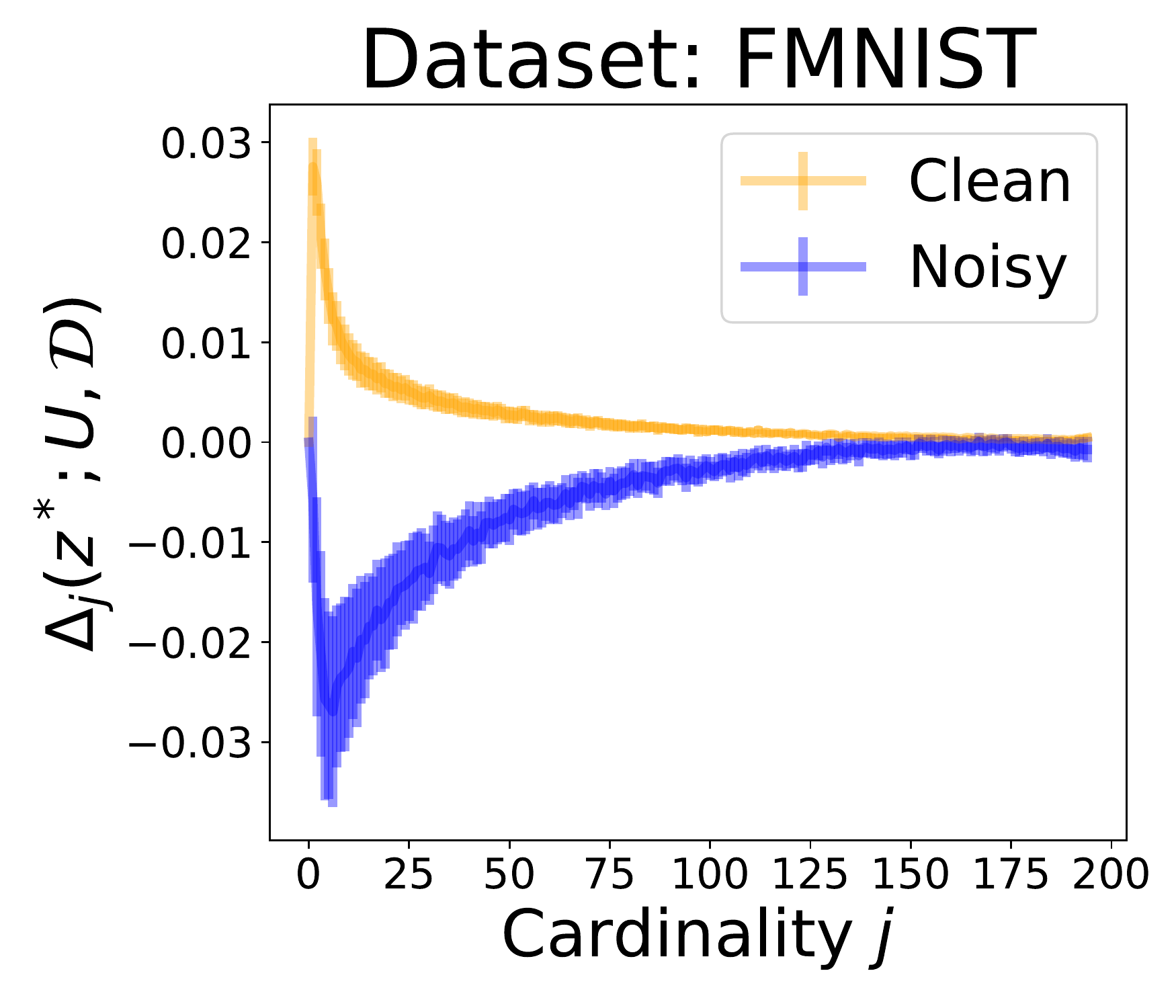}\\
    \includegraphics[width=0.245\textwidth]{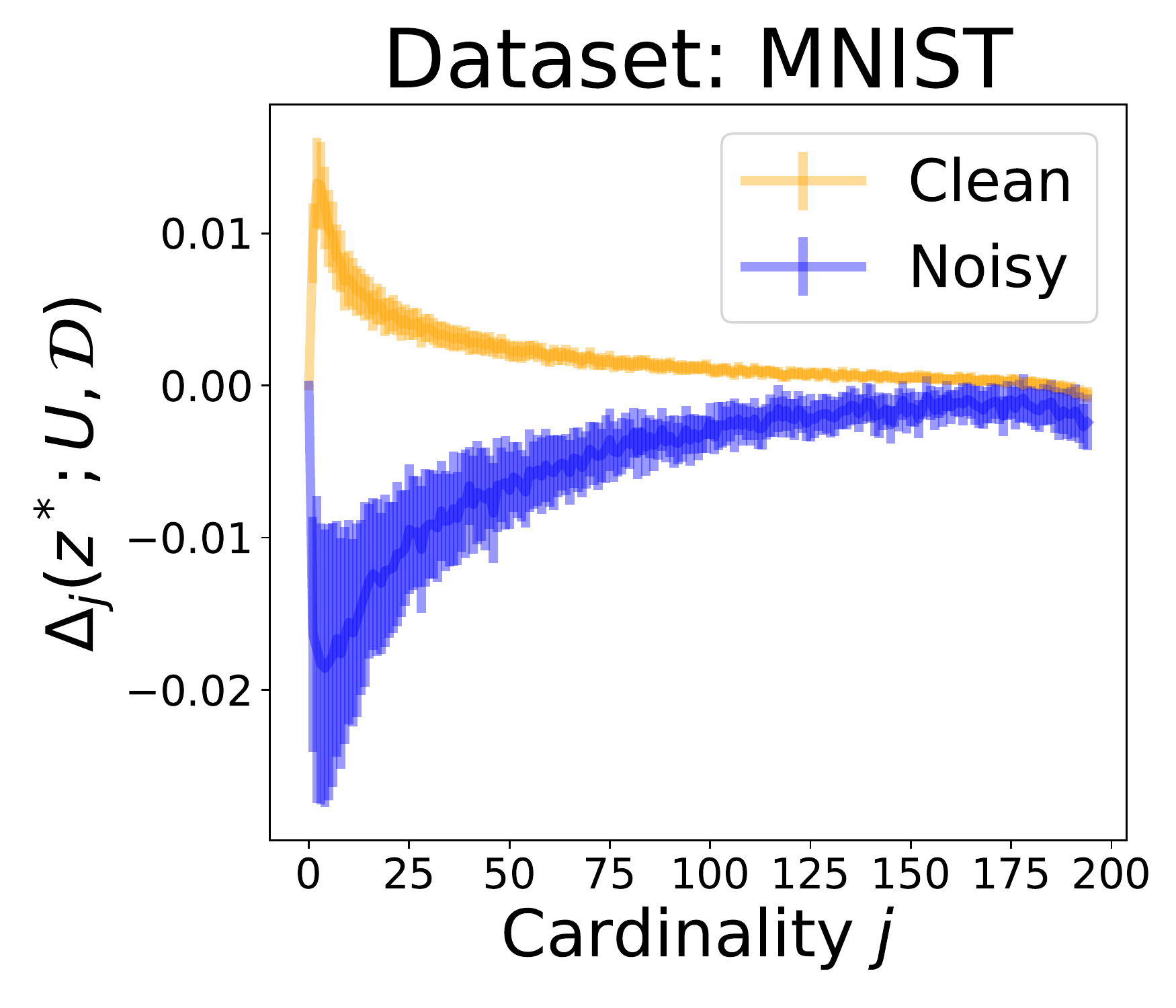}
    \includegraphics[width=0.245\textwidth]{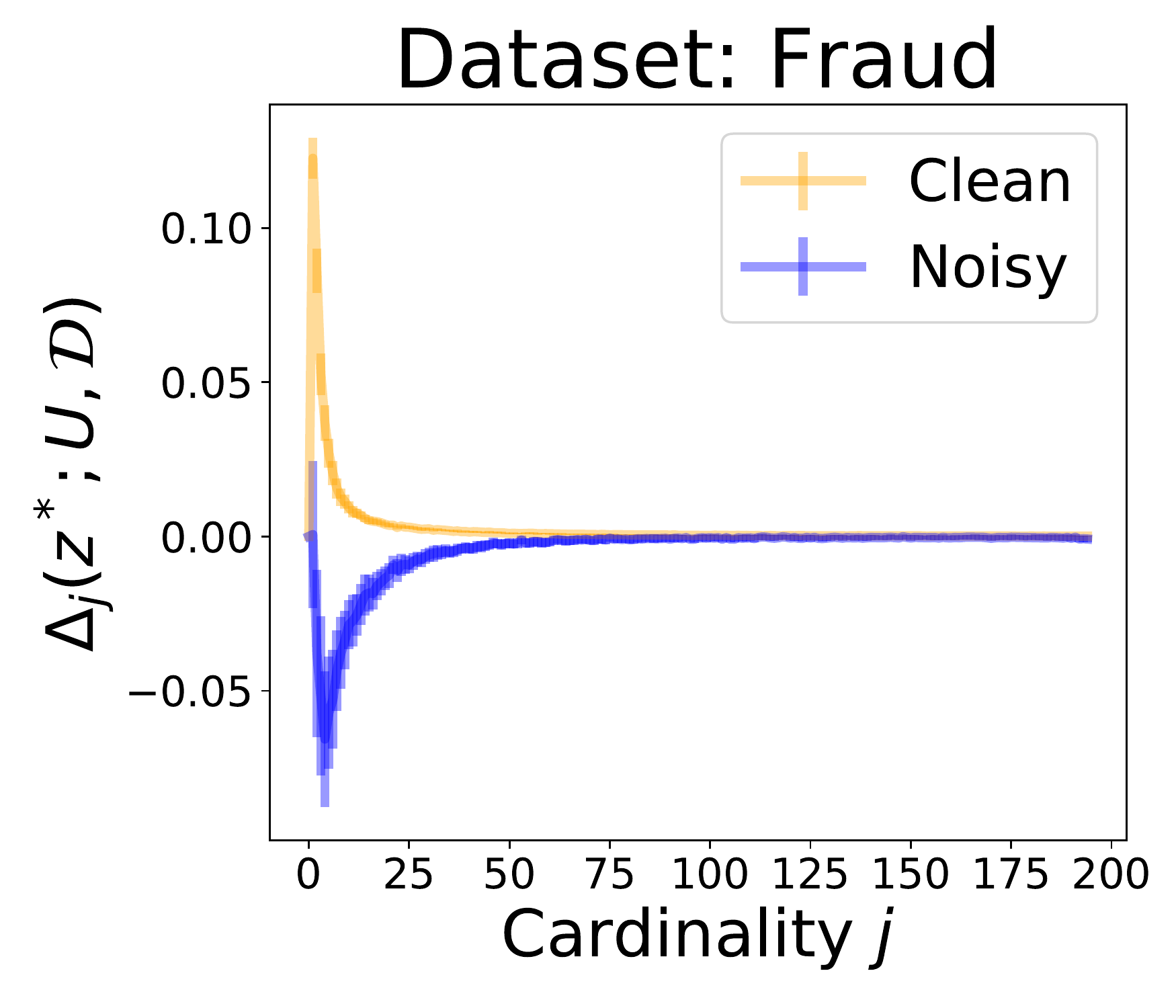}
    \includegraphics[width=0.245\textwidth]{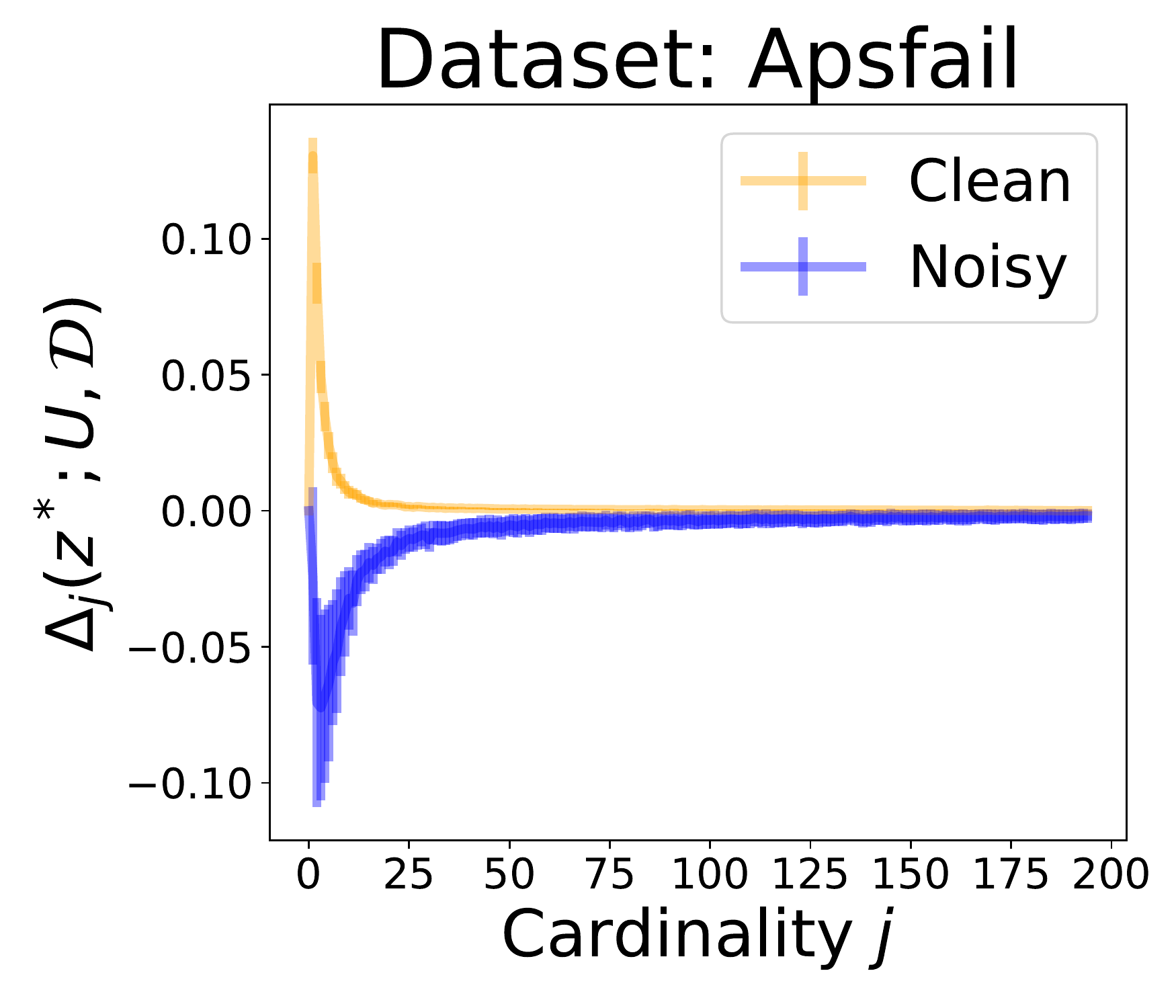}
    \includegraphics[width=0.245\textwidth]{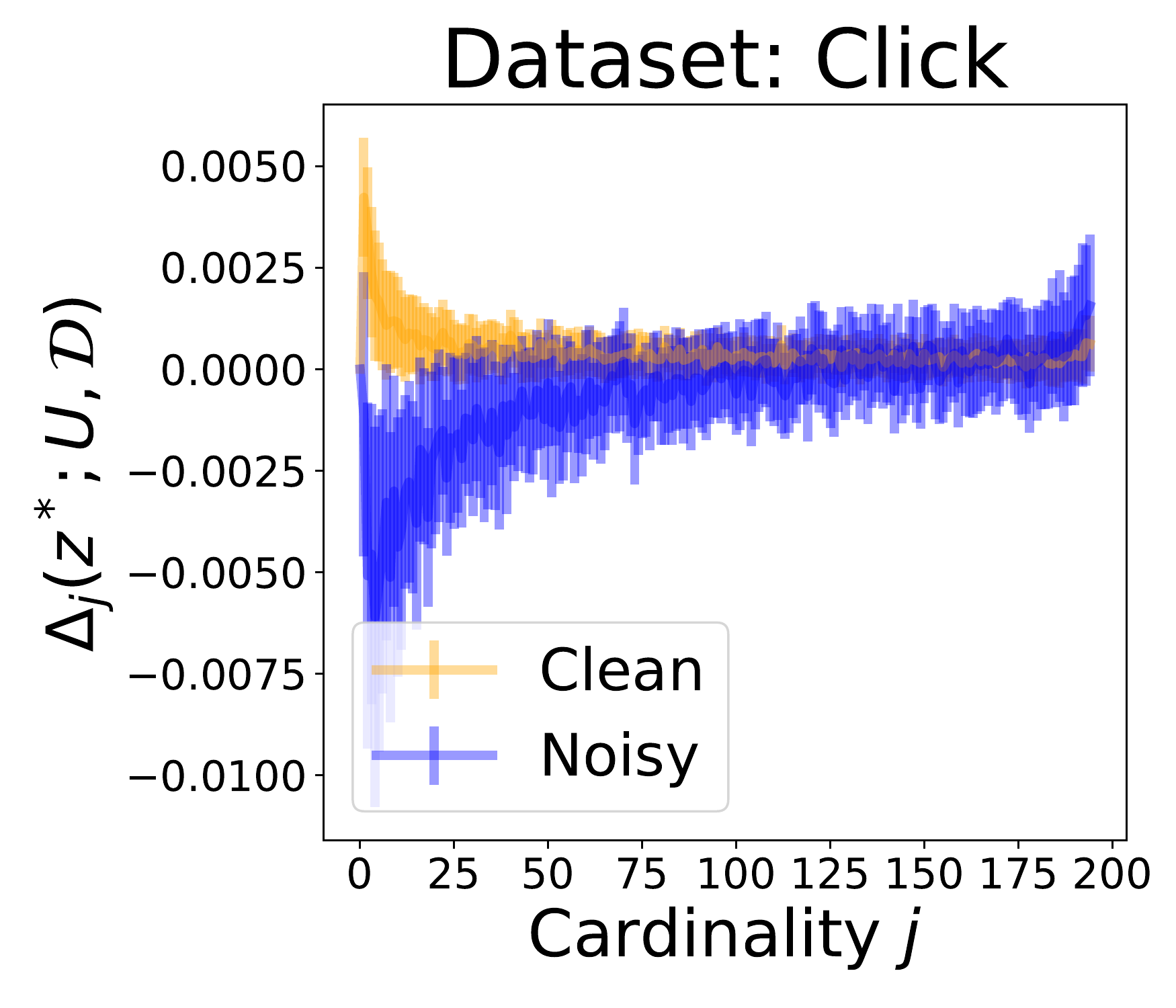}\\
    \includegraphics[width=0.245\textwidth]{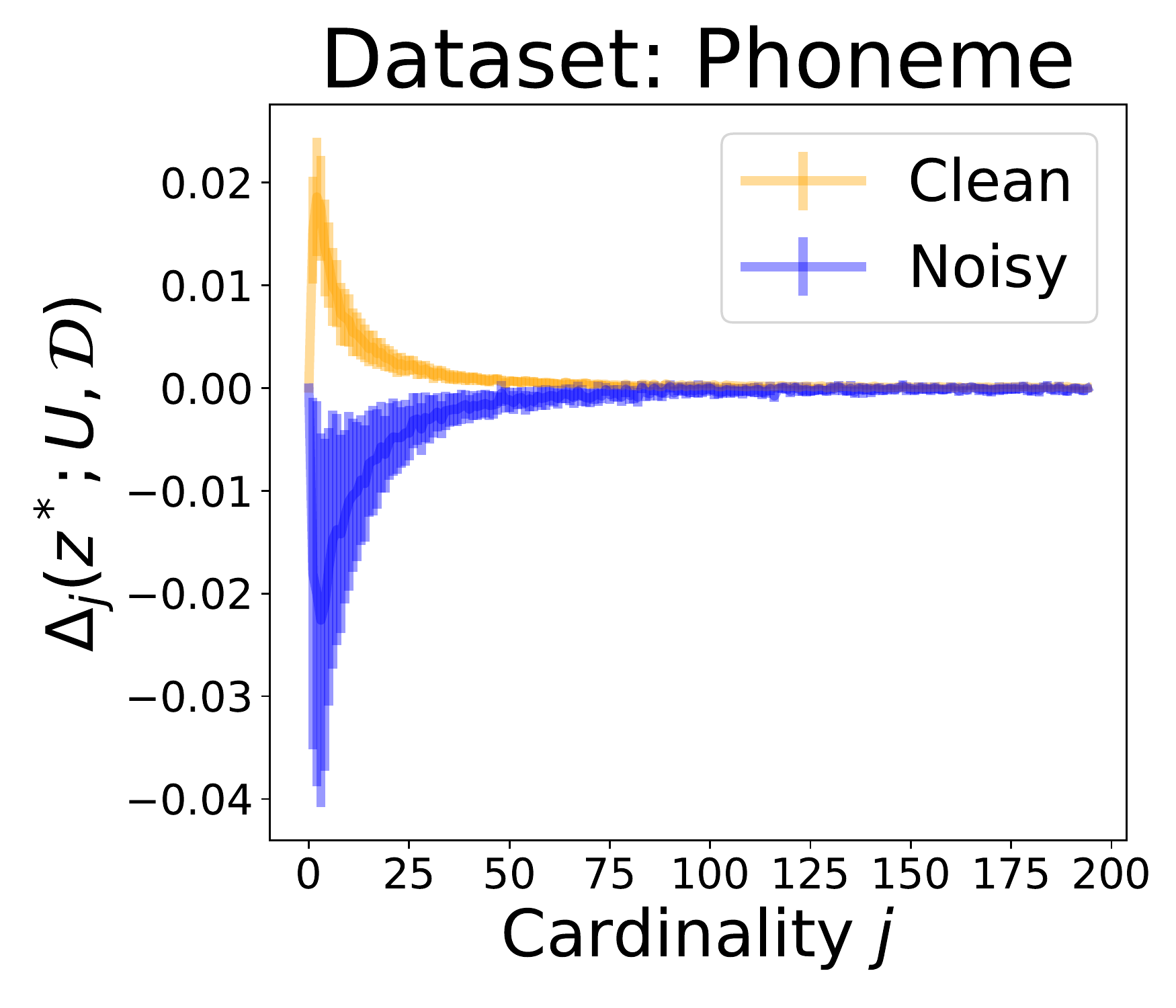}
    \includegraphics[width=0.245\textwidth]{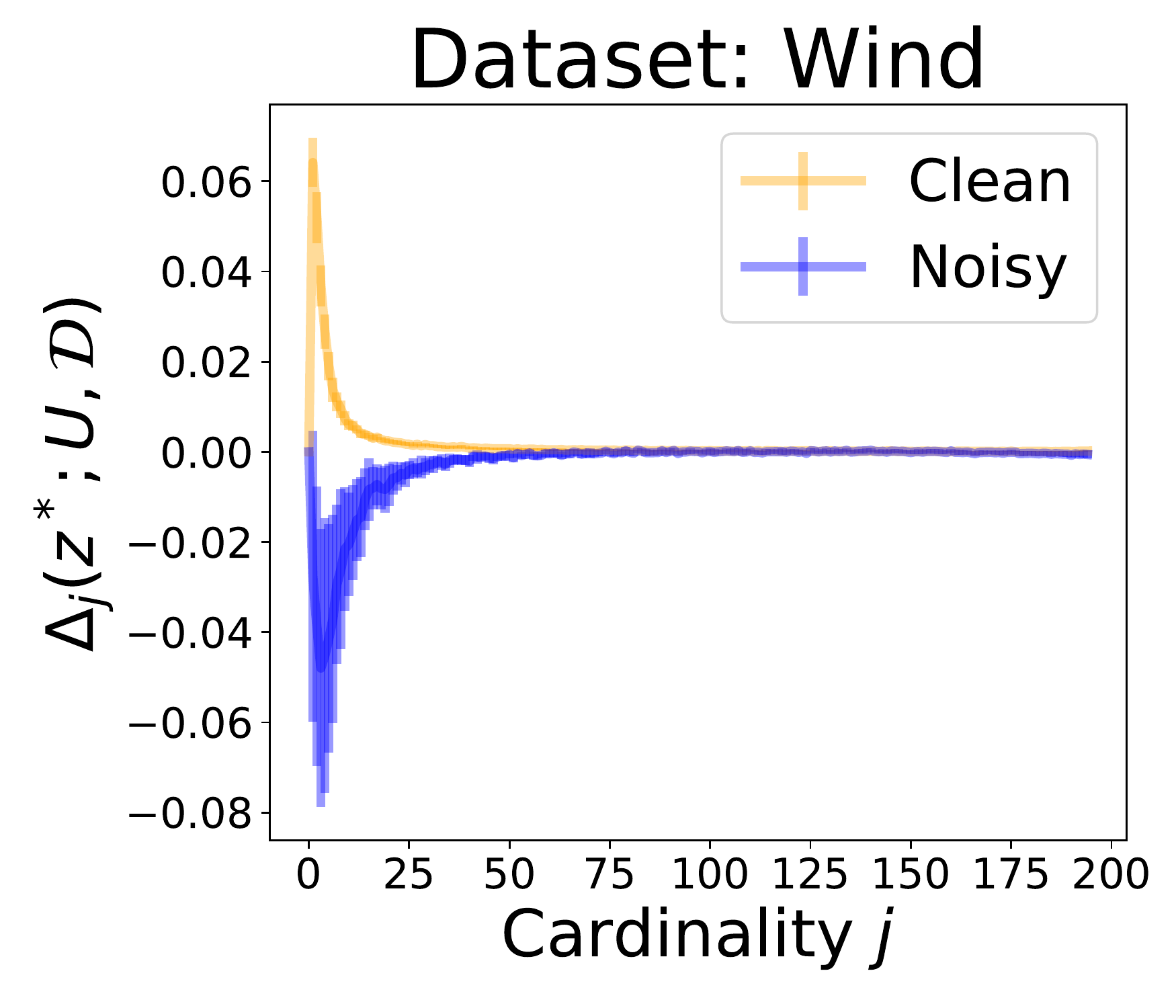}
    \includegraphics[width=0.245\textwidth]{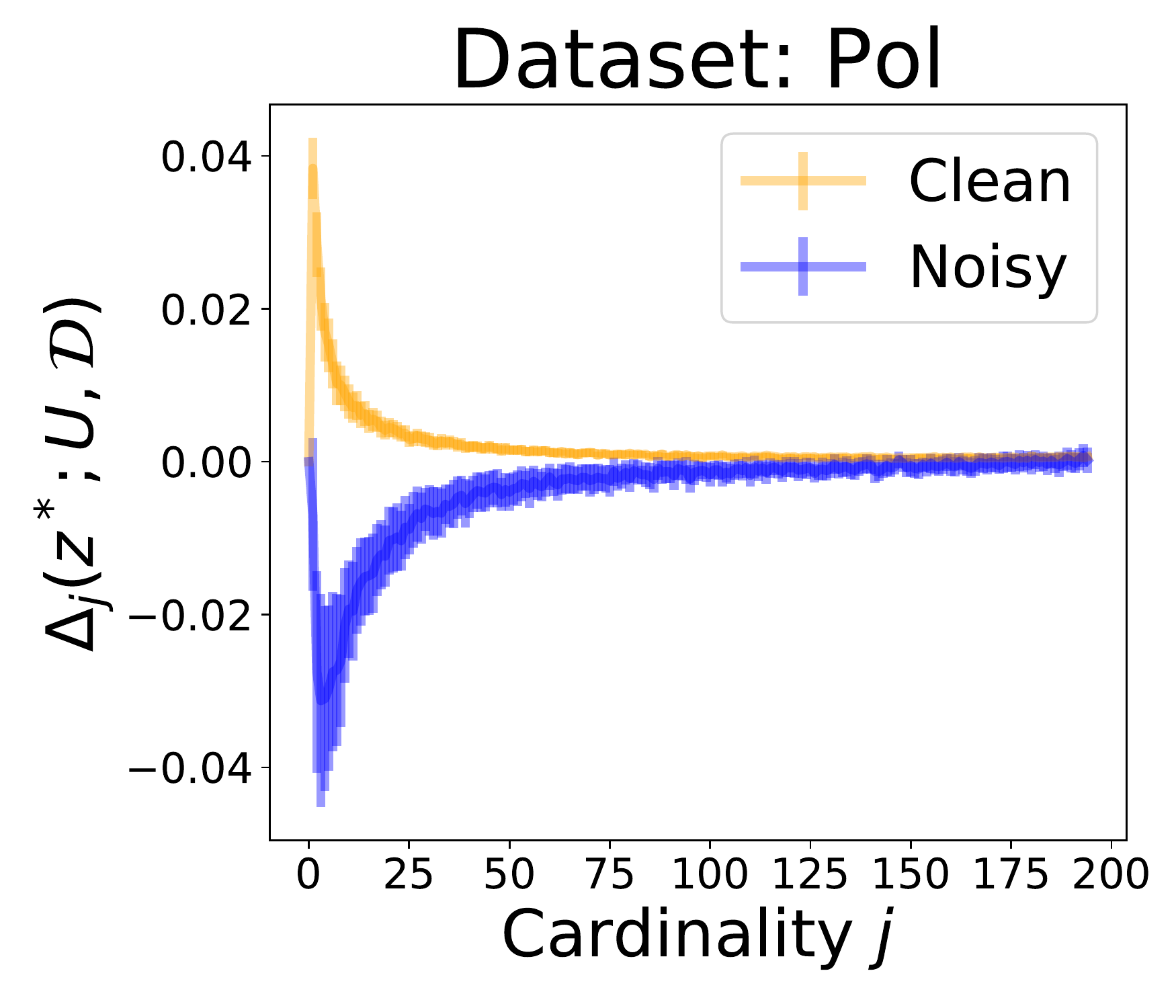}
    \includegraphics[width=0.245\textwidth]{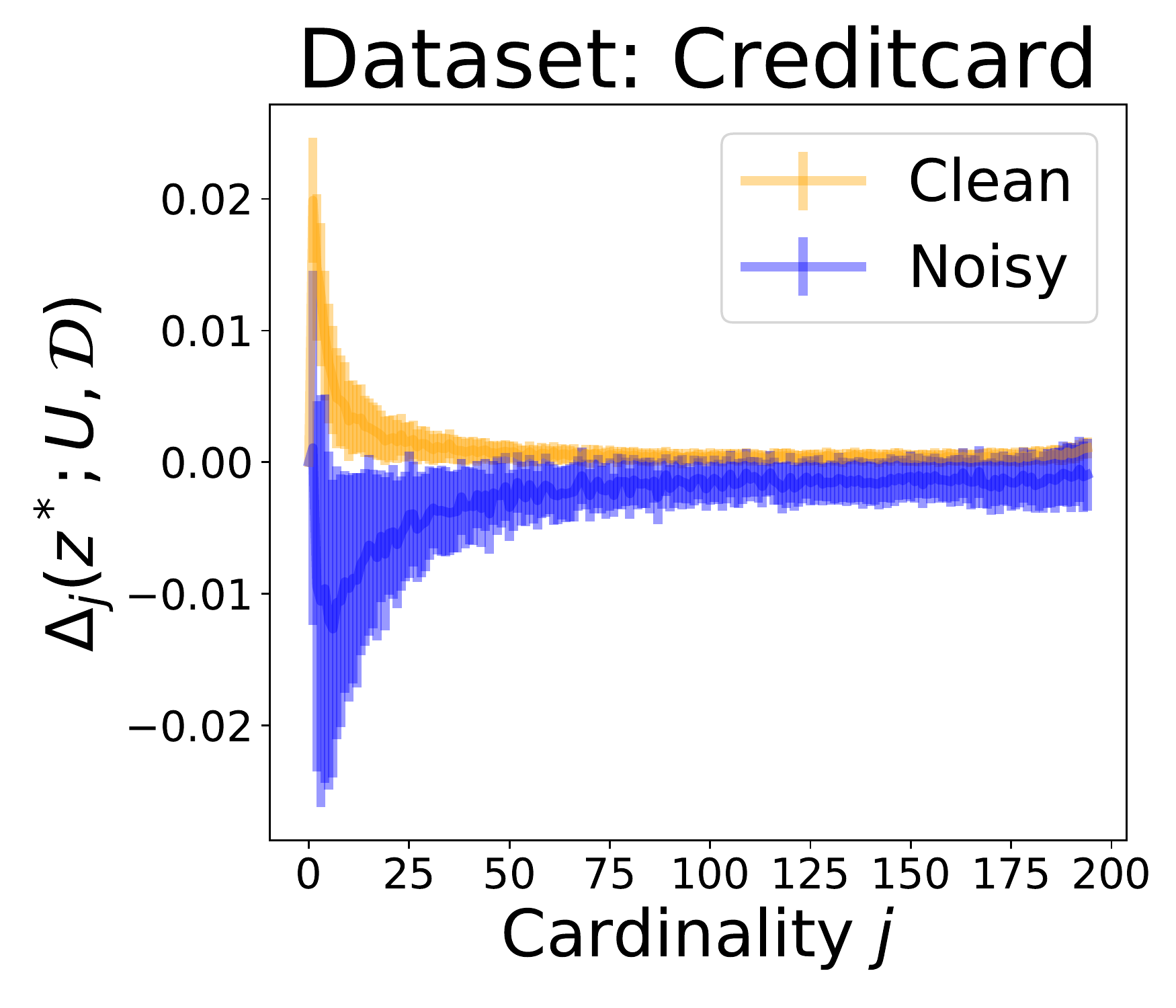}\\
    \includegraphics[width=0.245\textwidth]{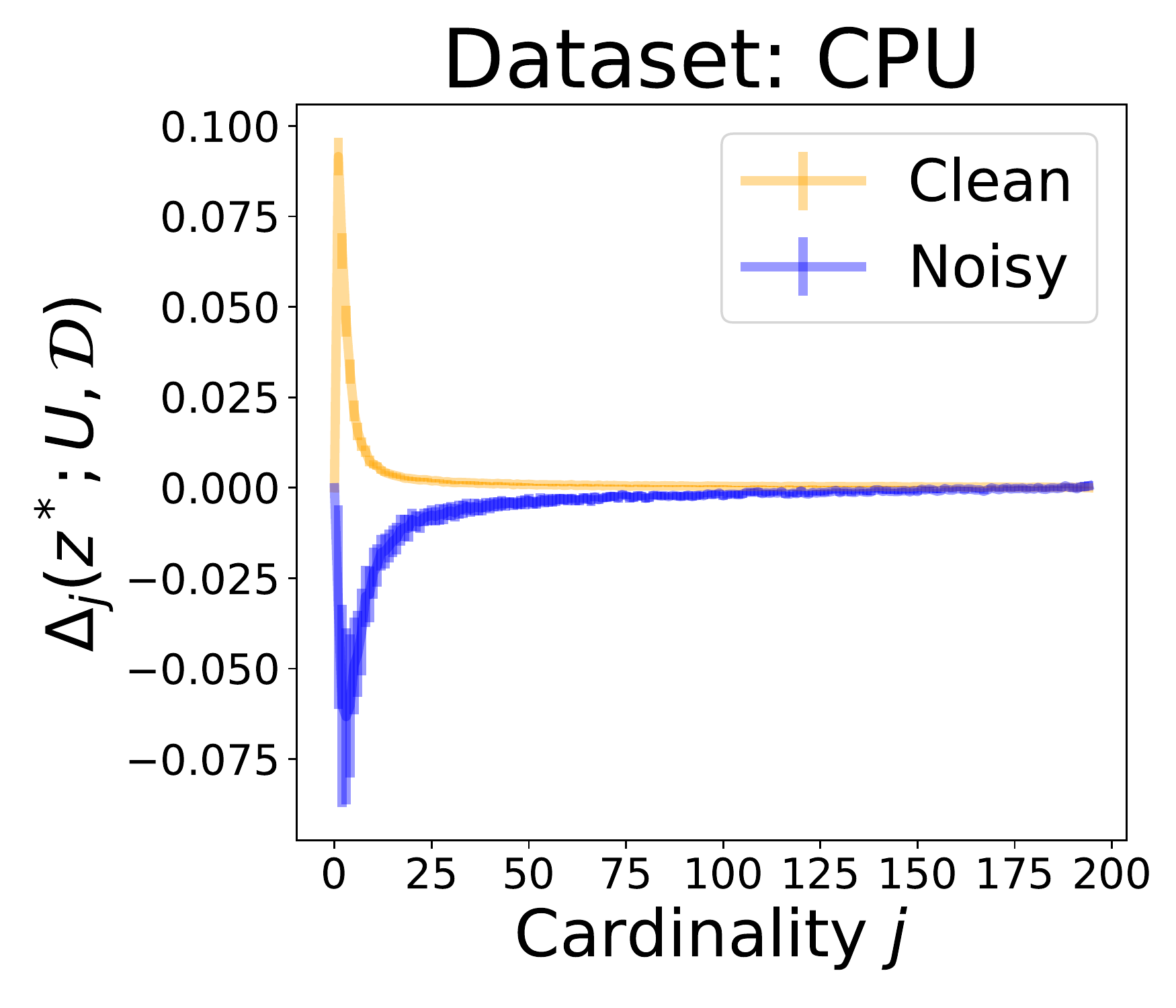}
    \includegraphics[width=0.245\textwidth]{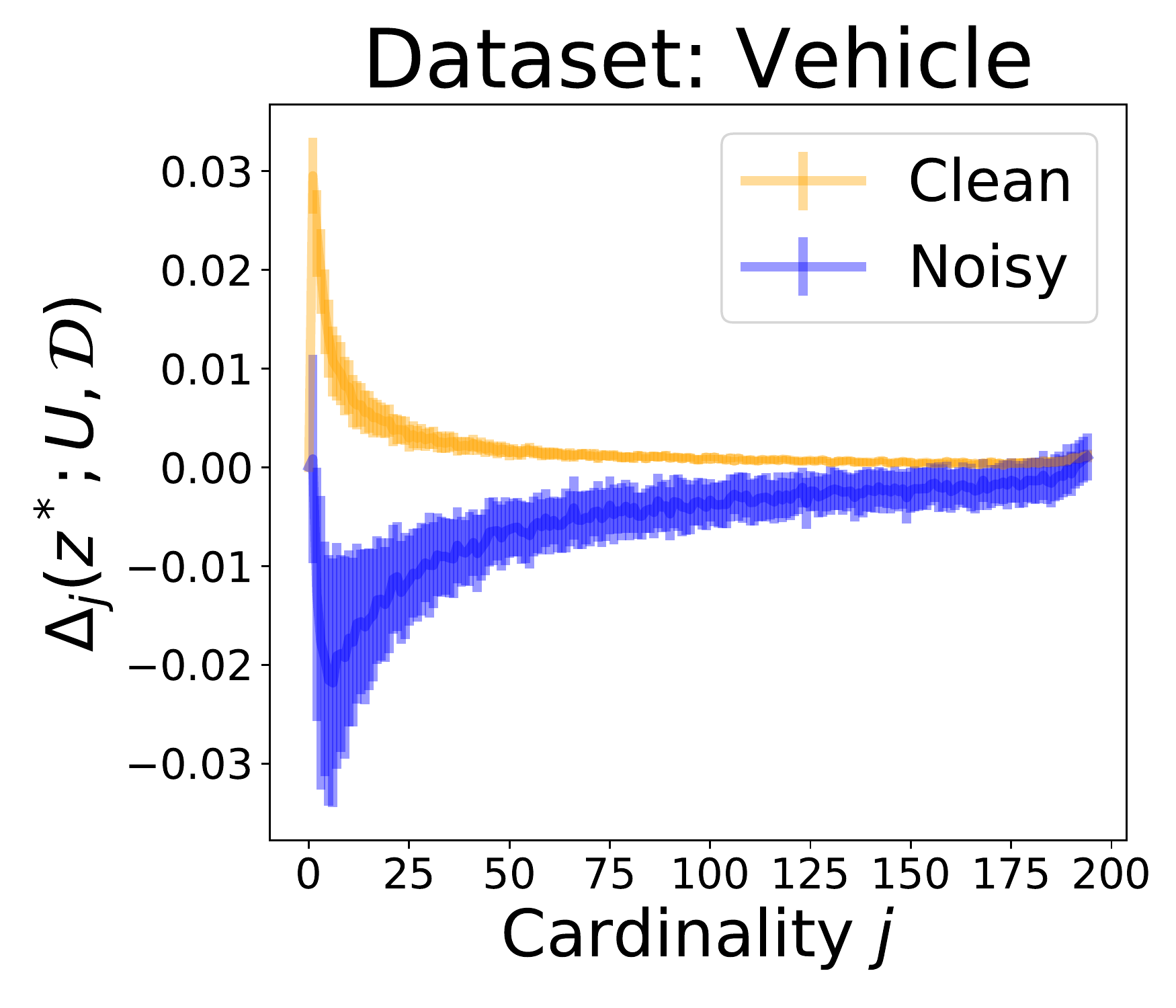}
    \includegraphics[width=0.245\textwidth]{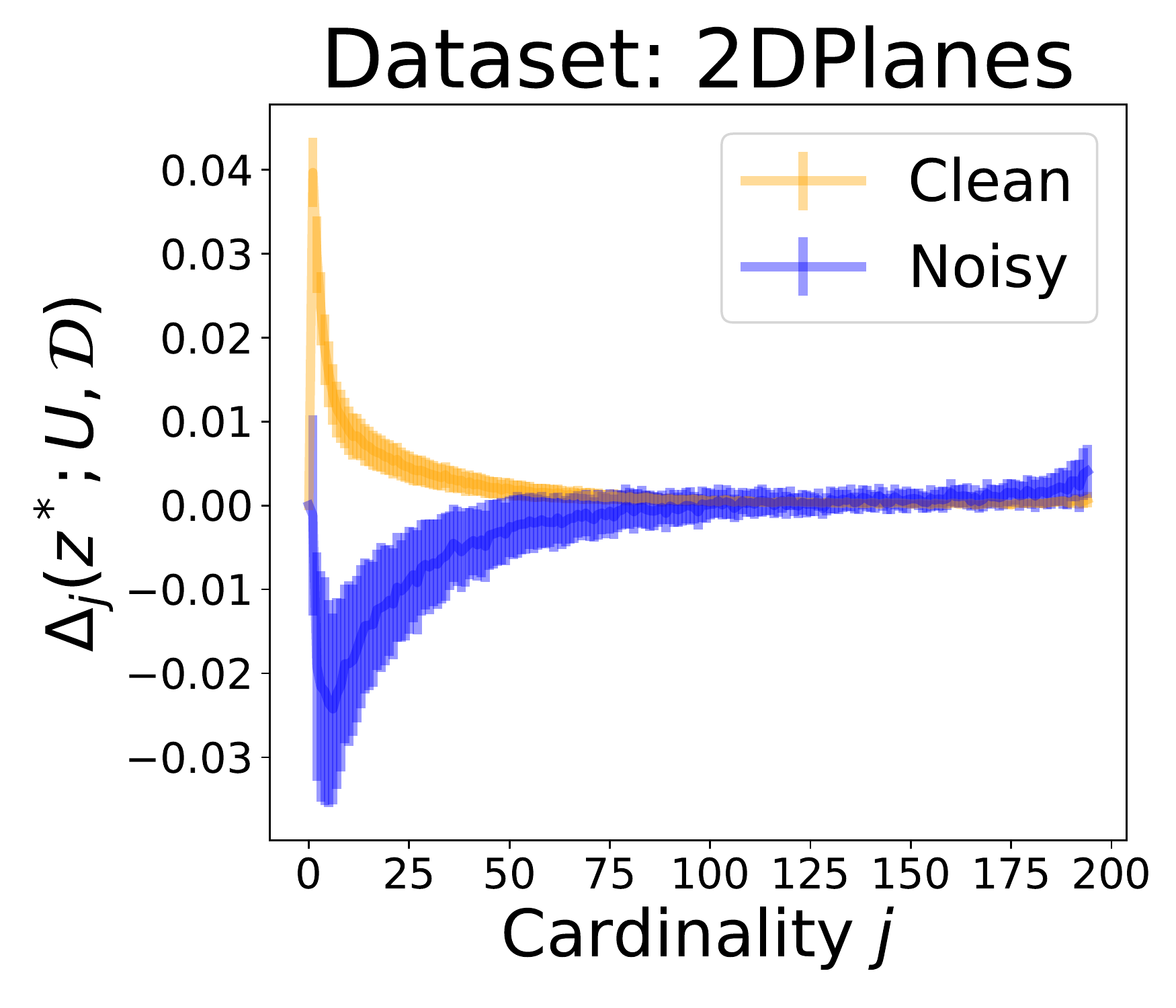}
    \caption{Illustrations of the marginal contributions $\Delta_j (z^*; U, \mathcal{D})$ as a function of the cardinality $j$ on the eleven datasets when a support vector machine is used. Each color indicates a noisy (blue) and a clean (yellow) data point. We denote a 99\% confidence band based on 50 independent runs. When the cardinality $j$ is large, it is hard to tell if point is noisy or not as they become similar or even reversed.}
    \label{fig:app_clean_noisy_marginal_contributions_svm}
\end{figure*}

\begin{figure}[t]
    \centering
    \hspace{-0.3in}
    \includegraphics[width=0.8\columnwidth]{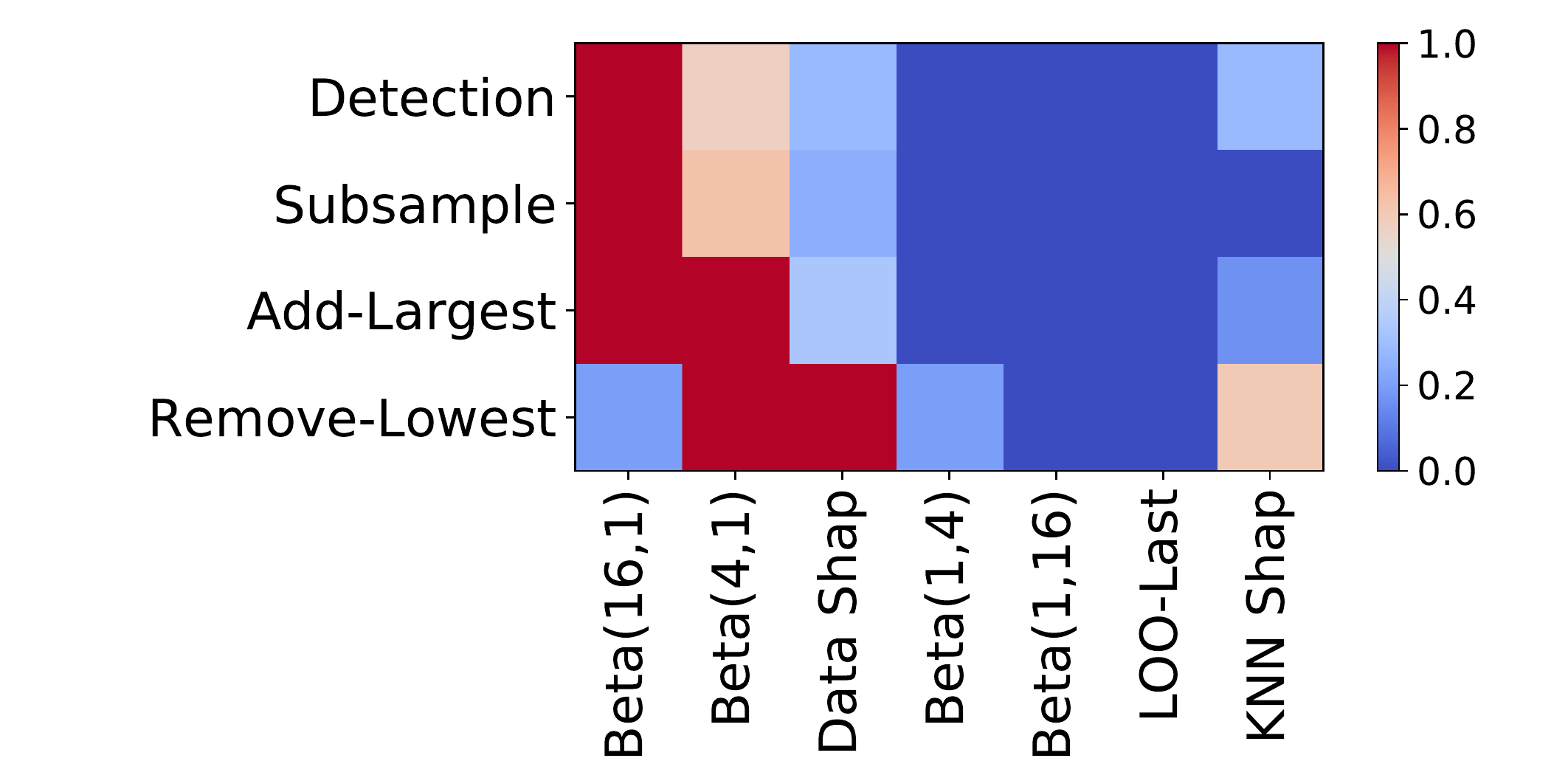}
    \caption{A summary of performance comparison on the fifteen datasets when a support vector machine is used. Each element of the heatmap represents a linearly scaled frequency for each task to be between 0 and 1. Better and worse methods are depicted in red and blue respectively.}
    \label{fig:app_heatmap_summary_count_svm}
    \vspace{-0.025in}
\end{figure}

\end{document}